\pdfoutput=1
\documentclass{informs3}
\renewcommand{\theARTICLETOP}{}
\OneAndAHalfSpacedXII

\usepackage{subcaption}
\theoremstyle{TH}

\let\argmin\relax
\let\SSclassh\h
\let\SSclassbe\be
\let\SSclassbi\bi
\let\h\relax
\let\be\relax
\let\bi\relax
\usepackage{custom_formats}
\let\h\SSclassh
\let\be\SSclassbe
\let\bi\SSclassbi
\usepackage{booktabs,multirow}
\usepackage{algorithm,algpseudocode}
\usepackage{enumitem,bbm,colortbl}
\usepackage{etoolbox}
\AtBeginEnvironment{tabular}{\TableSpaced}
\AtBeginEnvironment{algorithmic}{\SingleSpacedXI\small}
\usepackage{tikz}
\usetikzlibrary{shapes,arrows,shadows,positioning,decorations.pathreplacing,trees,calc,patterns}
\usepackage{rotating}
\usepackage{natbib}
\bibpunct[, ]{(}{)}{,}{a}{}{,}
\def\bibfont{\normalsize}

\usepackage{hyperref}
\hypersetup{
  hidelinks,
  pdftitle={Minority Reports: Balancing Cost and Accuracy in Ground Truth Data Annotation},
  pdfauthor={Hsuan Wei Liao, Christopher Klugmann, Daniel Kondermann, Rafid Mahmood}
}

\newcommand {\pleooone}{($p<0.001$)}

\newcommand {\bernoullidist}{\mathrm{Bernoulli}}
\newcommand {\binomialdist}{\mathrm{Binomial}}

\newcommand {\Perr}{P_{\mathrm{err}}}
\newcommand {\Perrcls}{P_{\mathrm{err}}^{\mathrm{cls}}}

\newcommand {\Ierr}{I_{\mathrm{err}}}
\newcommand {\Ierrcls}{I_{\mathrm{err}}^{\mathrm{cls}}}

\newcommand {\NOM}{\mathbf{RA}}
\newcommand {\NOMD}{\mathbf{RA-D}}
\newcommand {\CLS}{\mathbf{RPA}}
\newcommand {\CLSD}{\mathbf{RPA-D}}
\newcommand {\CLSB}{\mathbf{RPA-B}}

\EquationsNumberedThrough

\let\SSproof\proof
\let\SSendproof\endproof
\renewenvironment{proof}[1][Proof]{\SSproof{#1}}{\Halmos\SSendproof}

\makeatletter
\newcommand{\ECHowTheorems}{%
  \@for\SSecenv:=thm,defn,lem,cor,prop,propn,rem,example,claim,conj,assumption,ex\do{%
    \setcounter{\SSecenv}{0}%
    \expandafter\edef\csname the\SSecenv\endcsname{EC.\noexpand\arabic{\SSecenv}}%
    \expandafter\edef\csname theH\SSecenv\endcsname{EC.\noexpand\arabic{\SSecenv}}%
  }%
}
\makeatother

\MANUSCRIPTNO{}
\begin{document}
\RUNAUTHOR{Liao et al.}
\RUNTITLE{Minority Reports: Balancing Cost and Accuracy in Ground Truth Data Annotation}
\TITLE{Minority Reports: Balancing Cost and Accuracy in Ground Truth Data Annotation}
\ARTICLEAUTHORS{%
\AUTHOR{Hsuan Wei Liao\textsuperscript{1}, Christopher Klugmann\textsuperscript{2}, Daniel Kondermann\textsuperscript{2}, Rafid Mahmood\textsuperscript{1}}
\AFF{\textsuperscript{1}University of Ottawa \\
\textsuperscript{2}Wayve}
}
\ABSTRACT{
  Data annotation is an essential but costly aspect of developing machine learning-based technology, requiring multiple human workers to repeatedly label each data point in a data set. We explore the tradeoff between annotation accuracy and cost by identifying minority reports---worker responses disagreeing with the reference majority-vote label---that may indicate unnecessary redundancy, and can therefore be proactively pruned before they are executed. We develop a theoretical framework for choosing the number of repeated annotations required for a task and pruning policies at a target label quality. We estimate minority-report risk on two computer vision datasets re-annotated through a professional platform, then evaluate a rolling policy on the assignment stream to prune over 35\% of repeats while preserving more than 99.9\% of reference majority-vote labels.
}

\maketitle

\section{Introduction}

Machine learning (ML) systems are trained on massive datasets that pair input data, such as images or text, with outcome labels that the underlying model is expected to predict.
These labels are often produced by human annotators analyzing each input data point.
When labels are systematically noisy, the model is fit to a distorted target \citep{song2022learning, northcutt2021pervasive}.
Consequently, producing reliable labels at scale is an operational bottleneck under finite budget and turnaround constraints \citep{dimensional2019, braun2019eurocity}.

ML developers obtain labeled data through annotation operations that are either in-house or outsourced to third-party platforms. These platforms decompose raw data into annotation tasks, assign these tasks to workers, and aggregate the resulting responses into final labels.
However, variation in worker skills and task difficulty can lead to noisy annotations. Platforms can influence label quality through several levers
before work reaches the annotation queue (e.g., training, incentives, and task design) and while work is being completed (e.g., assignment, redundancy, and aggregation algorithms) \citep{karger2014budget, ogawa2020gamification, gonzalez2024fairness}.
Unfortunately, execution-stage controls are difficult to manage due to high-throughput task completion and volatile worker pools, meaning the most common practice for ensuring label quality is to collect repeated responses for each task and aggregate them through majority vote or other variant rules \citep{karger2014budget, khetan2016achieving, Liao_2021_CVPR}. Although repetition improves reliability, it also consumes capacity, lengthens turnaround, and raises labor costs.

In this work, we focus on mitigating the redundancy in repeat task allocation. This lever exists in every annotation pipeline, regardless of the data modality, the composition of the worker pool, or the platform's assignment and aggregation rules.
To measure redundancy in repeats, we evaluate the prevalence of \emph{minority reports}, i.e., individual responses that disagree with the final majority-vote label.
Although disagreeing responses may still contain useful information, we study their redundancy relative to the labels produced by the counterfactual majority-vote.
If a platform can identify planned repeats that are likely to end up in the eventual minority, it can prune those repeats from the assignment pool before workers perform them. To this end, we develop a policy where we estimate the minority report risk for any planned task assignment and cancel high-risk repeats before service. This just-in-time policy can work within existing allocation policies by determining whether planned assignments proceed without changing the underlying rules.

We first establish a theoretical framework for choosing: (i) the number of repeats to take, and (ii) a pruning policy, to ensure that an annotation task meets a maximum label-error tolerance. We characterize the trade-off between reliability versus cost from additional repeats and pruning. We prove a sufficient condition under which a pruning policy improves on repeat-only allocation.

We validate our framework by implementing a policy that sequentially updates risk estimates as user and task responses become available. We use two public computer vision datasets re-annotated through a professional data annotation platform under their standard customer operating conditions. The resulting assignment logs contain repeated worker responses, timestamps, and reference majority-vote labels for 196,446 and 16,000 tasks, respectively \citep{braun2019eurocity, alibeigi2023zenseact}.
We find that minority report risk is driven primarily by image ambiguity, worker heterogeneity, and within-shift fatigue. Moreover, pre-execution pruning yields favorable cost-quality frontiers. With regular recalibration, the policy can reduce annotation costs by 35\% while preserving more than 99.9\% of reference majority-vote labels.

Our contributions are:
\begin{enumerate}
\item We introduce the minority report as an operational unit of annotation redundancy. Using minority reports, we formalize two optimization problems for determining the minimum redundancy needed to ensure that the probability of an erroneous majority vote annotation is below a desired threshold.

\item Using a platform-scale re-annotation study, we analyze factors affecting erroneous annotations in the labeling of computer vision data to show that image ambiguity, worker skill variability, and continuous activity predict disagreement with the majority vote.

\item We develop a prescriptive framework to forecast in real-time whether an annotator will disagree with the future crowd decision and determine whether to remove task assignments. In simulations, our policy prunes 35\% of repeats while preserving an $F_1$ score of 0.997 with respect to the reference majority-vote labels.
\end{enumerate}

\textbf{Notation.} For any probability $p$, we write $\bar p:=1-p$, and apply the same convention to probability-valued functions. We use $\Ind(A)$ for the indicator of an event $A$ and $\field{N}=\{0,1,2,\ldots\}$ for the nonnegative integers. We omit function arguments when their dependence is clear from context.

\begin{figure}[t!]
    \centering
    \includegraphics[width=0.7\linewidth]{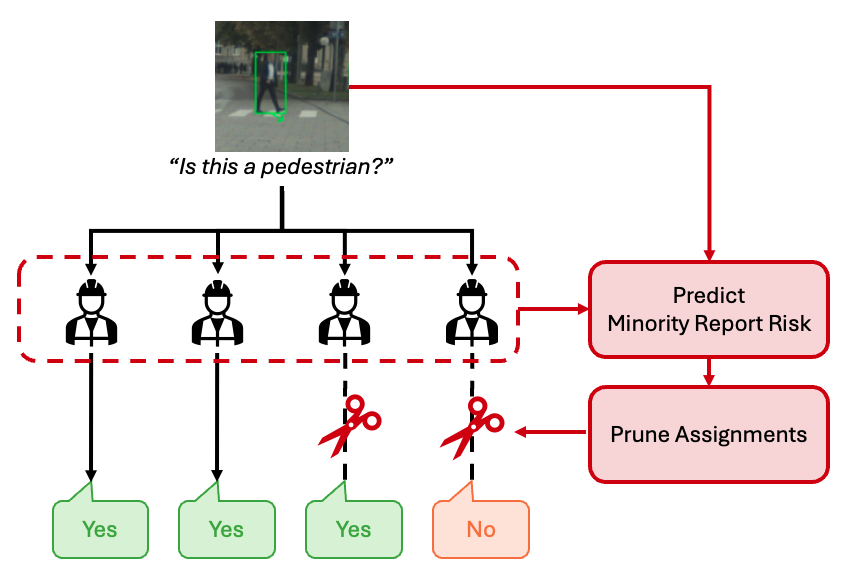}
    \caption{
        In our framework, an annotation task is first routed to multiple workers for repeated labeling. We estimate the minority-report risk of each worker-task assignment before execution and prune assignments whose risk exceeds the chosen threshold. The remaining responses are collected and aggregated.
    }
    \label{fig:teaser}
\end{figure}

\section{Related literature}
\label{sec:related_lit}

Training ML models requires large labeled datasets, but even well-known benchmark datasets contain nontrivial errors on up to 10\% of data samples \citep{deng2009imagenet, northcutt2021pervasive}.
The ML literature proposes methods to bypass these errors by learning from noisy labels \citep{song2022learning}, transferring labels across annotation protocols \citep{liao2024transferring}, and using predictive uncertainty to verify labels \citep{durasov2026uncertainty}.
Moreover, the label aggregation literature combines multiple workers' responses using majority-vote aggregation, expectation-maximization estimators \citep{dawid1979maximum,raykar2010learning}, optimization-based aggregation \citep{karger2014budget, khetan2016achieving, Wang2017}, Bayesian aggregation \citep{yin2021learning, palley2023boosting}, and end-to-end annotation-learning approaches \citep{Rodrigues2018, Tanno2019, Goh2023}.
In contrast, we explore which already-assigned tasks should still be executed, and which likely redundant assignments should be pruned before workers complete them.

Worker reliability and adversarial behavior pose challenges for crowdsourced labeling. Annotation tasks are often low paying, so some workers become disengaged or submit low-effort answers \citep{staats2012specialization, karger2014budget}. In related work, \citet{jagabathula2017identifying} identify adversarial workers by tracking repeated disagreements and filtering those workers after tasks are completed. In contrast, we intervene at an earlier level. Our model estimates the probability that a specific worker-task assignment will end in a minority report, then prunes unfavorable assignments before service. When a worker is consistently performing at low skill, the policy may still remove nearly all of that worker's tasks through assignment-level decisions rather than a hard worker ban.

Dataset annotation pipelines involve large operational processes where platforms assign a high volume of tasks for workers (see \citet{bouajaja2017survey} for a review). The workforce-allocation literature studies assignment decisions across assembly systems \citep{vila2014branch}, crowdsourcing \citep{Wang2017, Tian2022, Manshadi2022, Fatehi2022}, finance \citep{stratman2024decision}, and healthcare \citep{lanzarone2014robust,adelman2024dynamic}.
\citet{Wang2017} adaptively choose which object to assign to an arriving worker using estimated object and worker quality, seeking a target label quality with fewer labels. We study cancellation of repeats already assigned to workers, holding the assignment stream fixed. Our policy can therefore complement dynamic allocation or simple randomized matching \citep{karger2014budget}.

Our paper also relates to the literature on online labor platforms and gig-work design \citep{Kittur2013, horton2010labor, Horton2011}. Data annotation platforms often feature frictions such as unstable working conditions \citep{berg2018digital,gonzalez2024fairness}, tight deadlines \citep{Yin2018}, and complex contract design \citep{Wen2016,kaynar2023estimating}.
In gamified annotation, \citet{farajollahzadeh2026raising} find that higher AI-generated minimum performance requirements on earlier tasks increase subsequent discretionary performance and contributor retention. In contrast, we study how a platform can control queue load and quality when workers differ in skill, tasks differ in difficulty, and worker activity varies over the shift.

Finally, our work relates to the broader empirical literature on worker behavior and productivity. Task switching, interruptions, specialization, and prior experience shape worker performance in healthcare \citep{kc2012accumulating, jin2020cost,gurvich2020collaboration}, manufacturing \citep{cai2018recover}, finance \citep{staats2012specialization}, and software development \citep{narayanan2009matter}. Our empirical analysis shows that the same forces affect high-volume data annotation work.

\section{How many times should we repeat an annotation?}

In this section, we first describe the data annotation workflow. We then formalize two problems: (i) how many repeats should be allocated to annotate a given data sample; and (ii) given an estimator that can forecast whether a task assignment will lead to a poor outcome, how can we simultaneously allocate repeats and prune specific task assignments. We leave all proofs to the Electronic Companion.

\subsection{Preliminaries on annotation operations}

A data annotation process begins by defining the ML task and the characteristics of the data necessary to train an ML model. This ontology establishes clear guidelines for labeling and ensuring consistency across annotators. For example, to train an ML model that can detect different types of obstacles observed from an onboard car camera, we may define the characteristics of pedestrians (i.e., humans with both feet on the ground) to distinguish them from motorcyclists (i.e., humans with both feet on a two-wheeled vehicle).
A series of multiple choice questions (e.g., “Is the object in the image a pedestrian? Answer yes/no/can’t solve”) is constructed based on these guidelines. We refer to an annotation task as the tuple of a data sample and a single corresponding question.

Tasks are organized and distributed to a pool of annotators through an online system. Depending on the platform, workers may be in-house employees or crowdsourced on demand.
Annotators undergo training and calibration with trial tasks to familiarize themselves with the guidelines.
The assignment of a task to an annotator is referred to as a repeat, and is determined randomly.
Each task is repeated by multiple workers, and the platform aggregates the repeated responses into a final label via majority vote. %
We interchange the terms ‘response’—which refers to the multiple choice answer to the task—and ‘label’—which refers to the object in an image crop as determined by one or more annotation tasks. There may be repeats where the annotator submitted a response that differed from the majority vote. We refer to such instances as minority reports, i.e., disagreements with the reference majority vote.

We address two platform decisions in this paper. First, the platform must choose how many repeats to allocate for each annotation task.  Second, assuming access to a model that can estimate the risk of a minority report for each planned repeat, the platform can choose to prune high-risk repeats before they are executed.

\subsection{Optimization models of repeat allocation and pruning}
\label{sec:ra_rpa_model}

We now formalize the workflow above into a model of allocating repeats and ex ante pruning assignments. We consider a single task under fixed classifier operating characteristics to isolate the trade-off between the number and reliability of votes; this can be further applied to multiple tasks or to sequential settings where information becomes available over time.
Consider a single binary annotation task. There exists a pool of workers characterized by features $X \sim P_X$. We randomly draw $n$ workers i.i.d. from the pool to answer the annotation question. We assume that $n$ is an odd number, i.e., $n = 2m + 1$ for $m \in \field{N}$.
For each annotator $j \leq n$, let $Y_j \in \{0, 1\}$ be a binary indicator that the worker's response does not match the ground truth, meaning the task assignment produced a minority report.

We define the task-level minority report propensity as $\phi := \EX_X[ \Pr(Y = 1 \;|\; X )]$, the marginal probability that a randomly drawn worker for this task produces a minority report.
When analyzing a single annotation task, $\phi$ is a fixed value, but our analysis naturally scales to multiple tasks by treating $\phi$ as a task-specific random variable.
Conditioned on the task and the worker pool, the minority report indicator for annotator $j$ follows a Bernoulli distribution $Y_j \sim \mathrm{Bernoulli}(\phi)$. Finally, let $S_n := \sum_{j=1}^n Y_j \sim \mathrm{Binomial}(n, \phi)$ be the total number of minority reports produced for the task.
Then, the aggregated label is incorrect if the majority of the annotators disagree with the unobserved ground truth, occurring with probability $\Perr(n, \phi) := \Pr( S_n \geq \lceil n/2 \rceil)$.

Our nominal problem, Repeat Allocation ($\NOM$), is to determine the minimum number of repeats to allocate to a specific annotation task in order to ensure that the downstream label is correct. Let $\epsilon \in (0,1)$ be a maximum tolerance on the probability of an incorrect aggregated label. Then, our problem is
\begin{align*}
  \NOM: \qquad \min_{n, m \in \field{N}} \quad & n \\
  \st \quad & \Perr(n, \phi) \leq \epsilon , \quad n = 2m + 1
\end{align*}
Let $n^\dagger(\epsilon, \phi)$ be the minimizer of problem $\NOM$.

Annotators who generate inaccurate labels can contribute unnecessary redundancy. To identify these annotators, suppose we have access to a binary classifier $f_\theta(X)$ that predicts for each task assignment, whether it will cause a minority report $f_\theta(X_j) = \hat{Y}_{j} \in \{0, 1\}$.
In practice, this classifier can be a statistical estimator, machine learning model, na\"ive decision rules, or even a human manager.
Without loss of generality, we assume $\theta \in \field{R}$ to be a scalar parameter governing the aggressiveness of the classifier (e.g., decision threshold of a logistic regression model).
We characterize the classifier by its out-of-sample True Positive Rate (TPR) $q_T(\theta) := \Pr( \hat{Y}_{j} = 1 \;|\; Y_{j} = 1)$ and False Positive Rate $q_F(\theta) := \Pr(\hat{Y}_{j} = 1\;|\; Y_{j} = 0)$. %

Finally, we define our task assignment pruning policy. Given a task with $n$ repeats, let $\hat{Y}_j$ indicate whether the classifier predicts the individual repeat is a minority report. If for repeat $j$, we predict $\hat{Y}_j = 1$, then we prune the task assignment from the queue before it is completed. Consequently, the prune rate, i.e., the fraction of annotation repeats that are pruned before assignment, is $r(\theta, \phi) := \Pr(\hat{Y} = 1) = \phi q_T(\theta) + (1-\phi) q_F(\theta)$. We focus on operating points where $r(\theta, \phi) < 1$.
Let $Z_n := \sum_{j=1}^n \Ind(\hat{Y}_j = 0)$ be the number of repeats remaining after pruning. Note that $Z_n \sim \mathrm{Binomial}(n, 1 - r(\theta, \phi))$ and let
\begin{align*}
  \pi(\theta, \phi) := \Pr(Y=1\;|\;\hat{Y}=0) = \frac{\phi (1-{q}_T(\theta))}{\phi (1-q_T(\theta)) + (1-\phi) (1-q_F(\theta))}
\end{align*}
be the probability of an un-pruned task being a minority report. We omit the dependency of $r$ and $\pi$ on $\phi$ and $\theta$ when it is obvious.
Finally, the probability of an incorrect aggregation after pruning is
\begin{align*}
  \Perrcls(n,\theta,\phi) := \EX_{Z_n}[\Perr(Z_n, \pi(\theta, \phi))] =\sum_{z=0}^{n} {n \choose z} (1-r(\theta,\phi))^{z} r(\theta,\phi)^{n-z}
  \Pr\!\left(B_z \ge \left\lceil \frac{z}{2}\right\rceil\right),
\end{align*}
where $B_z \sim \mathrm{Binomial}(z,\pi(\theta,\phi))$. We count ties for even $z$ (and the degenerate case $z=0$) as errors.

We now define the problem of jointly choosing the repeat count $n$ and the pruning policy $\theta$. As in problem $\NOM$, the platform seeks to satisfy a target bound on the probability of an incorrect aggregated label, which is now dependent on the number of executed repeats. This Repeat-Prune Allocation ($\CLS$) problem is
\begin{align*}
  \CLS: \qquad  \min_{n, m \in \field{N}, \theta\in\field{R}} \quad & n(1 - r(\theta, \phi)) \\
  \st \quad & \Perrcls(n,\theta,\phi) \leq \epsilon , \quad n = 2m + 1
\end{align*}
Let $n^*(\epsilon, \theta^*, \phi)$ and $\theta^*$ denote the minimizers of problem $\CLS$.

Note that because $1 - r(\theta,\phi) \geq 0$ for all $\theta,\phi$, problem $\CLS$ is equivalent to the following bi-level form:
\begin{align*}
  \CLSB:\quad \min_{\theta \in \field{R}} \quad & n^*(\epsilon, \theta, \phi) (1 - r(\theta, \phi)) \\
  \st \quad & n^*(\epsilon, \theta, \phi) = \argmin_{n =2m+1, m\in\field{N}} \left\{ n ( 1 - r(\theta, \phi)) \;|\; \Perrcls(n, \theta, \phi) \leq \epsilon \right\}
  \end{align*}
We solve $\CLSB$ in practice via grid searching. For each candidate threshold $\theta$ on a finite search grid, we find the minimum odd $n$ satisfying $\Perrcls(n,\theta,\phi) \leq \epsilon$, then evaluate the cost $n(1-r(\theta,\phi))$. We select the pair $(n^*,\theta^*)$ minimizing this cost across all $\theta$ candidates.

If the classifier family contains a no-pruning point $\theta$ such that $q_T(\theta)=q_F(\theta)=0$, then $r(\theta,\phi)=0$, $\pi(\theta,\phi)=\phi$, and $\Perrcls(n,\theta,\phi)=\Perr(n,\phi)$. Therefore, $\CLSB$ implies a weak bound $\min_{\theta} n^*(\epsilon,\theta,\phi)(1-r(\theta,\phi)) \le n^\dagger(\epsilon,\phi)$.
Consequently, a pruning threshold is valuable when the lower number of executed repeats outweighs any increase in repeats needed to satisfy the same error target. The theoretical analysis below characterizes when that tradeoff yields a strict cost improvement over repeat-only allocation.

\subsection{Theoretical analysis of cost improvements from pruning}
\label{sec:theoretical_analysis}

We now determine when $\CLS$ provides a strict cost improvement over $\NOM$.
First, we characterize when additional repeats reduce aggregation error, both without pruning and after pruning.
Second, we interpret how classifier quality and pruning aggressiveness affect the error-cost tradeoff.
Finally, we derive a sufficient condition under which pruning plus optimization yields strictly lower annotation cost than optimization without pruning.

We first identify when increasing repeats lowers majority-vote error in annotation. Intuitively, if the minority report propensity $\phi$ is not too large, additional repeats make a wrong majority less likely.
\begin{lem}\label{lem:perr_decreasing}
  $\Perr(n+2, \phi) \leq \Perr(n, \phi)$ if and only if $\phi \leq (\lfloor n/2 \rfloor + 1)/(n+1)$.
\end{lem}

When $n$ is odd, the upper bound reduces to $1/2$ and when $n$ is even, the upper bound converges to $1/2$ from above.
For the rest of this paper, we assume that workers have a better-than-random chance of correct annotation. Under this assumption, Lemma \ref{lem:perr_decreasing} will always hold.
\begin{assumption}\label{ass:phi_half}
  The minority report propensity satisfies $\phi \leq 1/2$.
\end{assumption}

We now analyze the predictive pruning setting to understand how threshold choice affects both retained minority report propensity and executed workload. We first assume mild regularity conditions on the classifier's operating characteristics.
\begin{assumption}\label{ass:clf}
  ~
  \begin{enumerate}
    \item The classifier's estimates are i.i.d. over all annotation repeats.
    \item $q_T(\theta)$ and $q_F(\theta)$ are monotonically non-decreasing in $\theta$.
    \item $q_T(\theta) \geq q_F(\theta)$ for all $\theta$.
  \end{enumerate}
\end{assumption}
The second statement is rooted in a ``no free lunch'' observation that as the aggressiveness of any pruning threshold changes,
$q_T(\theta)$ and $q_F(\theta)$ should move in the same direction. Finally, the condition $q_T(\theta) \geq q_F(\theta)$ excludes classifiers that are worse than random guessing.
Under Assumption \ref{ass:clf}, pruning improves retained-label quality by lowering the effective minority report probability among surviving repeats.
\begin{lem}\label{lem:pi_le_phi}
  Under Assumption \ref{ass:clf}, $\pi(\theta,\phi) \leq \phi$ for every admissible $\theta$ and $\phi\in[0,1]$.
\end{lem}

Pruning also reduces the number of retained votes. This creates a tradeoff between better retained-vote quality and lower retained-vote quantity.
As in Lemma \ref{lem:perr_decreasing}, we show that the label error probability decreases under a sufficiency condition of controlled pruning.

\begin{thm}\label{thm:monotone_n}
  Under Assumptions \ref{ass:phi_half} and \ref{ass:clf}, if $r(\theta, \phi) \leq (1 - 2 \phi \bar{q}_T(\theta))/3$, then $\Perrcls(n+2, \theta, \phi) \leq \Perrcls(n, \theta, \phi)$.
\end{thm}

Intuitively from Theorem \ref{thm:monotone_n}, if pruning removes too many would-be majority votes, the variance from a smaller retained sample can dominate the gain from cleaner retained votes.
Moreover, the sufficient bound in Theorem \ref{thm:monotone_n} is conservative, as we show in the following example.

\begin{figure}[t!]
\centering
\includegraphics[width=0.32\textwidth]{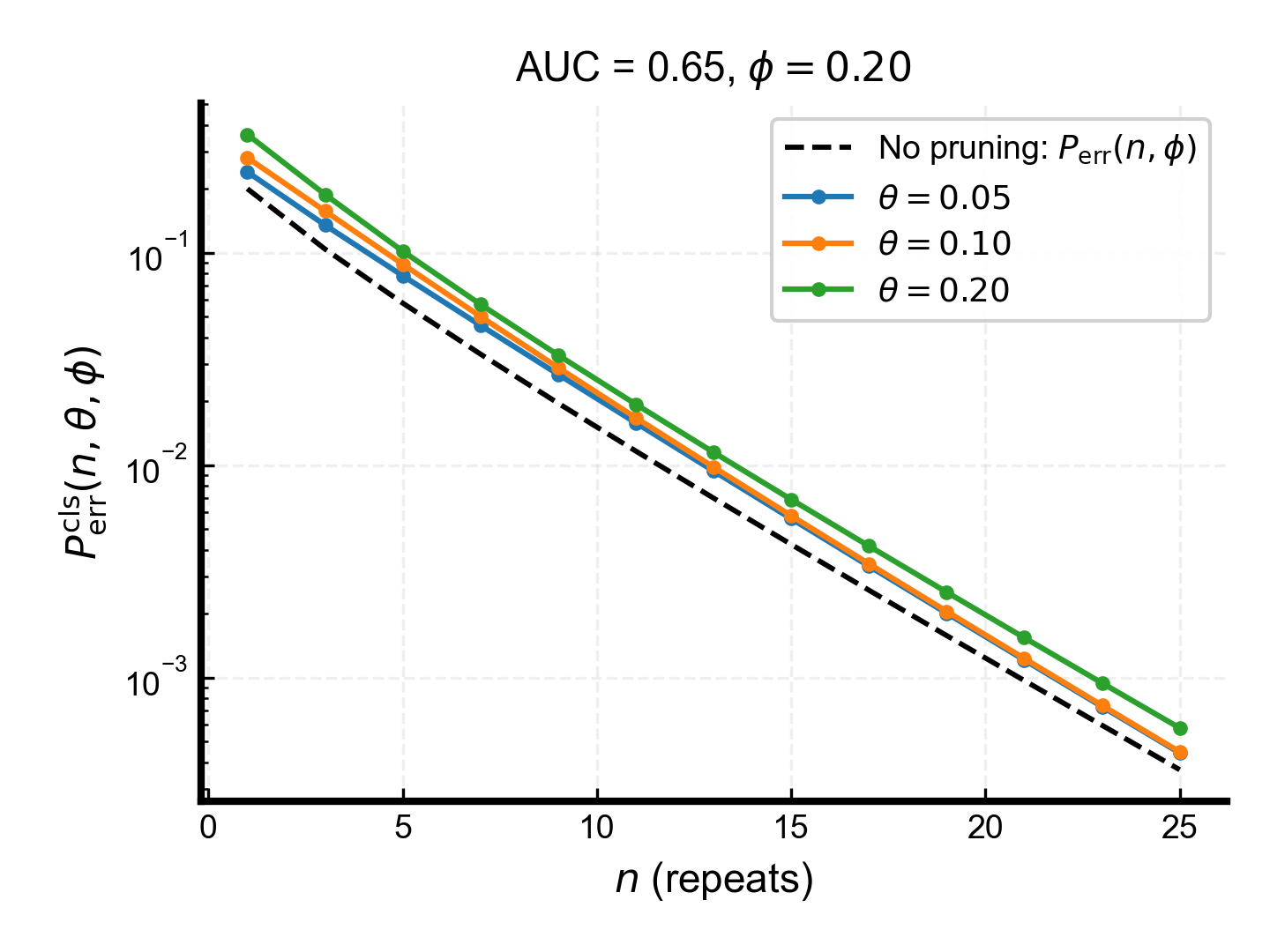}
\includegraphics[width=0.32\textwidth]{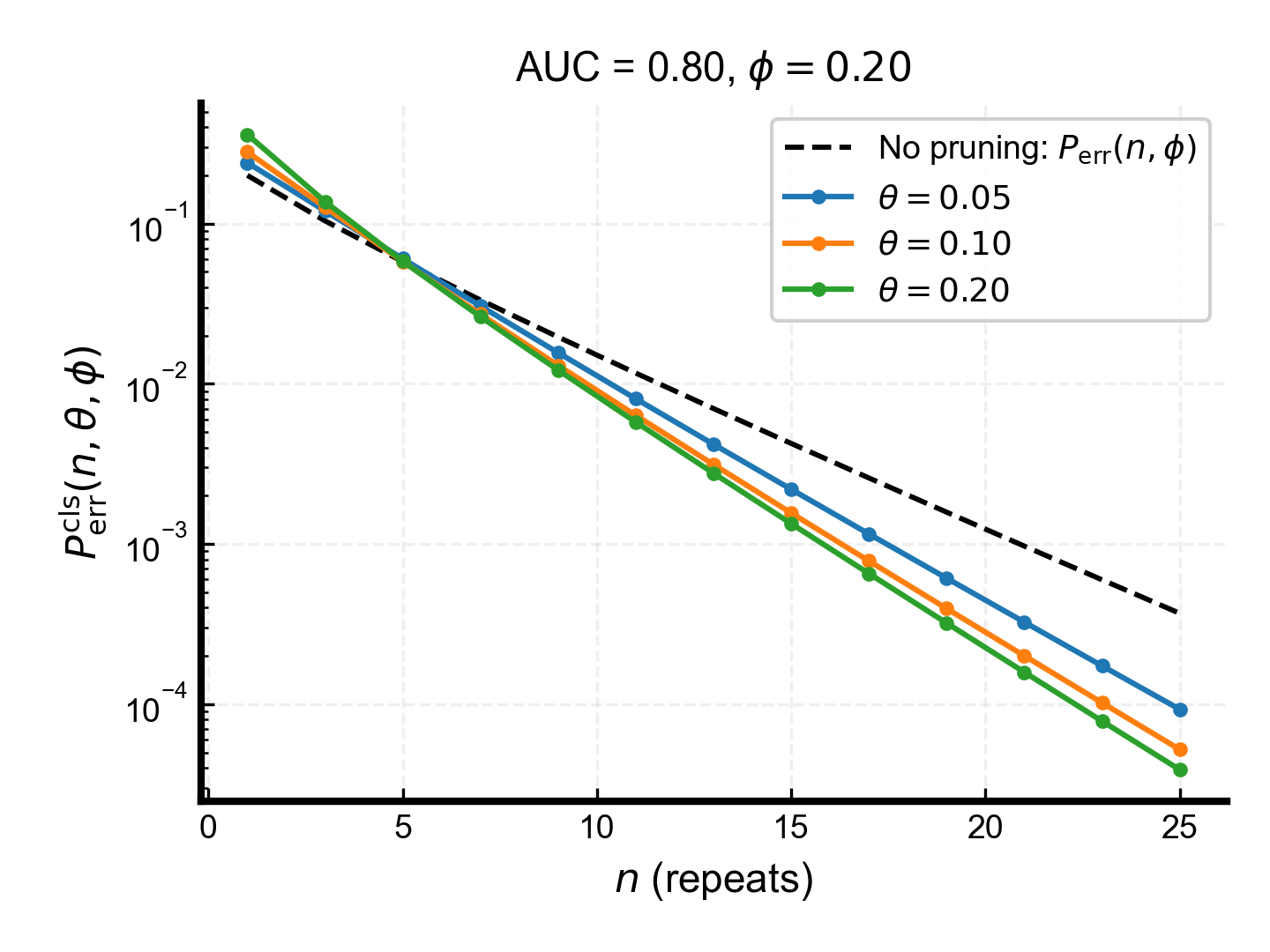}
\includegraphics[width=0.32\textwidth]{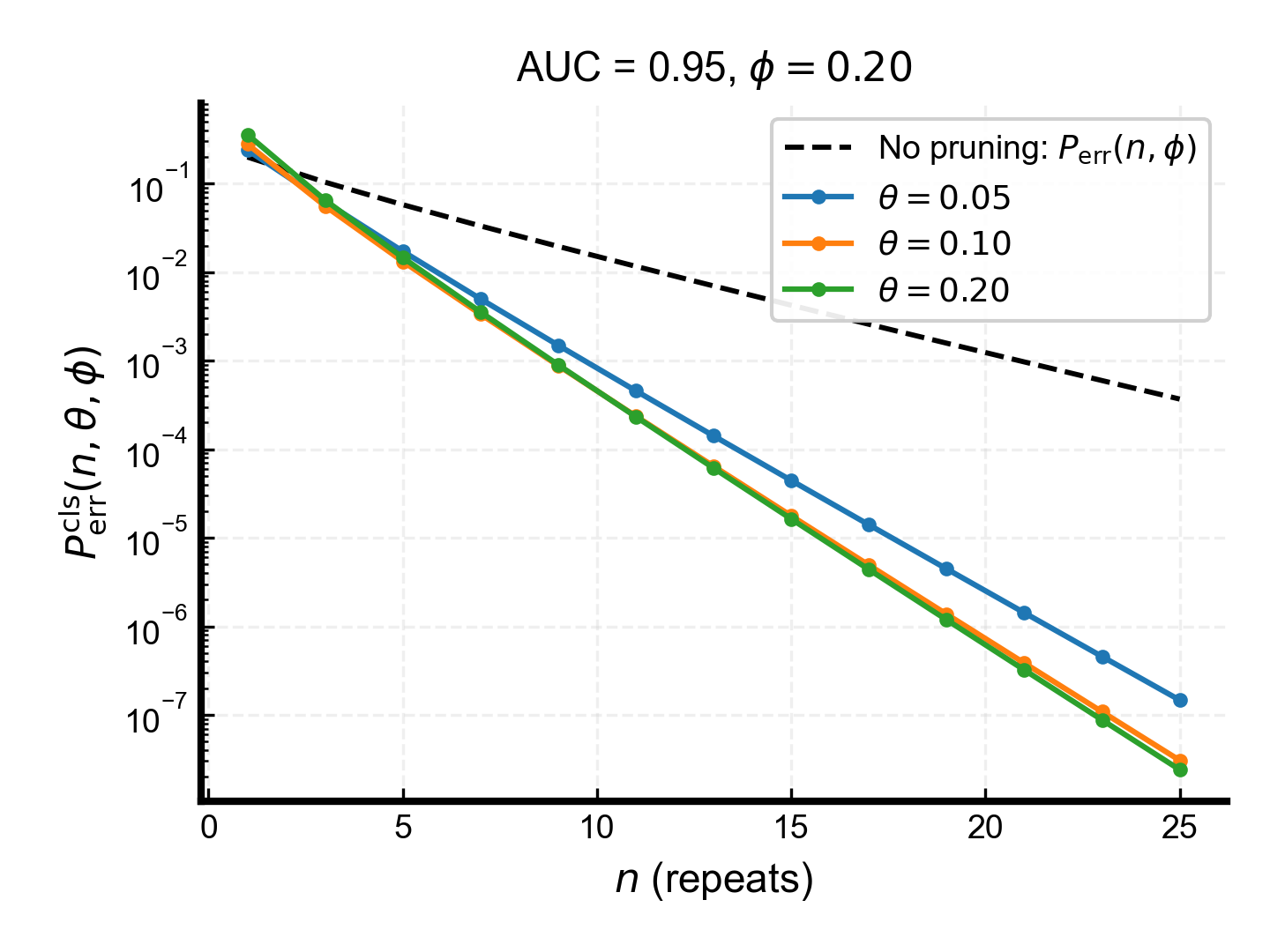}
\caption{Post-pruning error $P_{\mathrm{err}}^{\mathrm{cls}}(n,\theta,\phi)$ versus odd repeat counts under three binormal ROC models for $\phi = 0.20$. The dashed black curve is the no-pruning benchmark $P_{\mathrm{err}}(n,\phi)$.}
\label{fig:theory_perr_vs_n}
\end{figure}

\begin{example}
We visualize the effects of Lemma \ref{lem:perr_decreasing} and Theorem \ref{thm:monotone_n} using three stylized pruning classifiers generated by binormal ROC models with AUC values $0.65$, $0.80$, and $0.95$, representing weak, moderate, and strong pruning quality following the procedures of \citet{metz1999proper}. Let the latent classifier score $S \in \field{R}$ satisfy $S\mid Y=0\sim\mathcal N(0,1)$ and $S\mid Y=1\sim\mathcal N(d,1)$, where $Y=1$ indicates a true minority report. We prune assignments if $S>c$ for a raw-score cutoff $c$.
For each target AUC, we set $d=\sqrt{2}\,\Phi^{-1}(\mathrm{AUC})$.
We parameterize the ROC curve by the false-positive rate $\theta=1-\Phi(c)$, so that $q_F(\theta)=\theta$ and $q_T(\theta)=\Phi\!\left(\Phi^{-1}(\theta)+d\right)$. %

Figure \ref{fig:theory_perr_vs_n} fixes $\phi=0.20$, varies odd repeat counts $n\in\{1,3,\ldots,25\}$, and compares post-pruning error against the no-pruning benchmark at the same $n$.
First, for each threshold, post-pruning error decreases with $n$, matching the monotonicity implication of Theorem~\ref{thm:monotone_n}.
Second, stronger classifiers shift the post-pruning curves downward, implying that better pruning quality can hit the same error target with fewer repeats.

Figure~\ref{fig:theory_prune_rate} plots $\Perrcls(n, \theta, \phi)$ against realized prune rate $r(\theta,\phi)$ as $\theta$ varies.
Initially, pruning lowers error by reducing the minority report likelihood of retained votes, but aggressive pruning can remove too many votes.
Moreover, higher AUC can sustain wider low-error regions, while weaker classifiers require more conservative thresholds.
\end{example}

\begin{figure}[t!]
\centering
\includegraphics[width=0.32\textwidth]{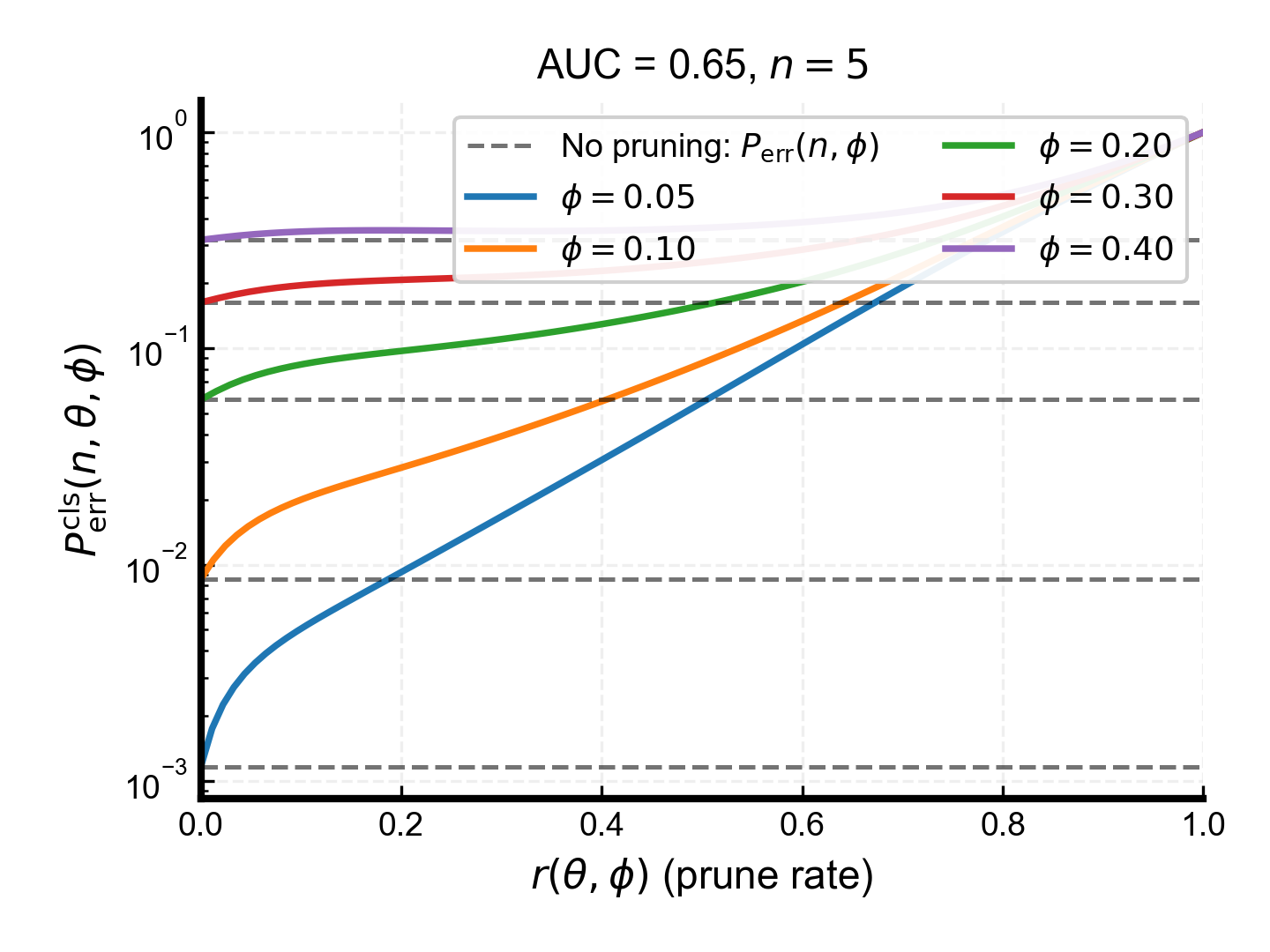}
\includegraphics[width=0.32\textwidth]{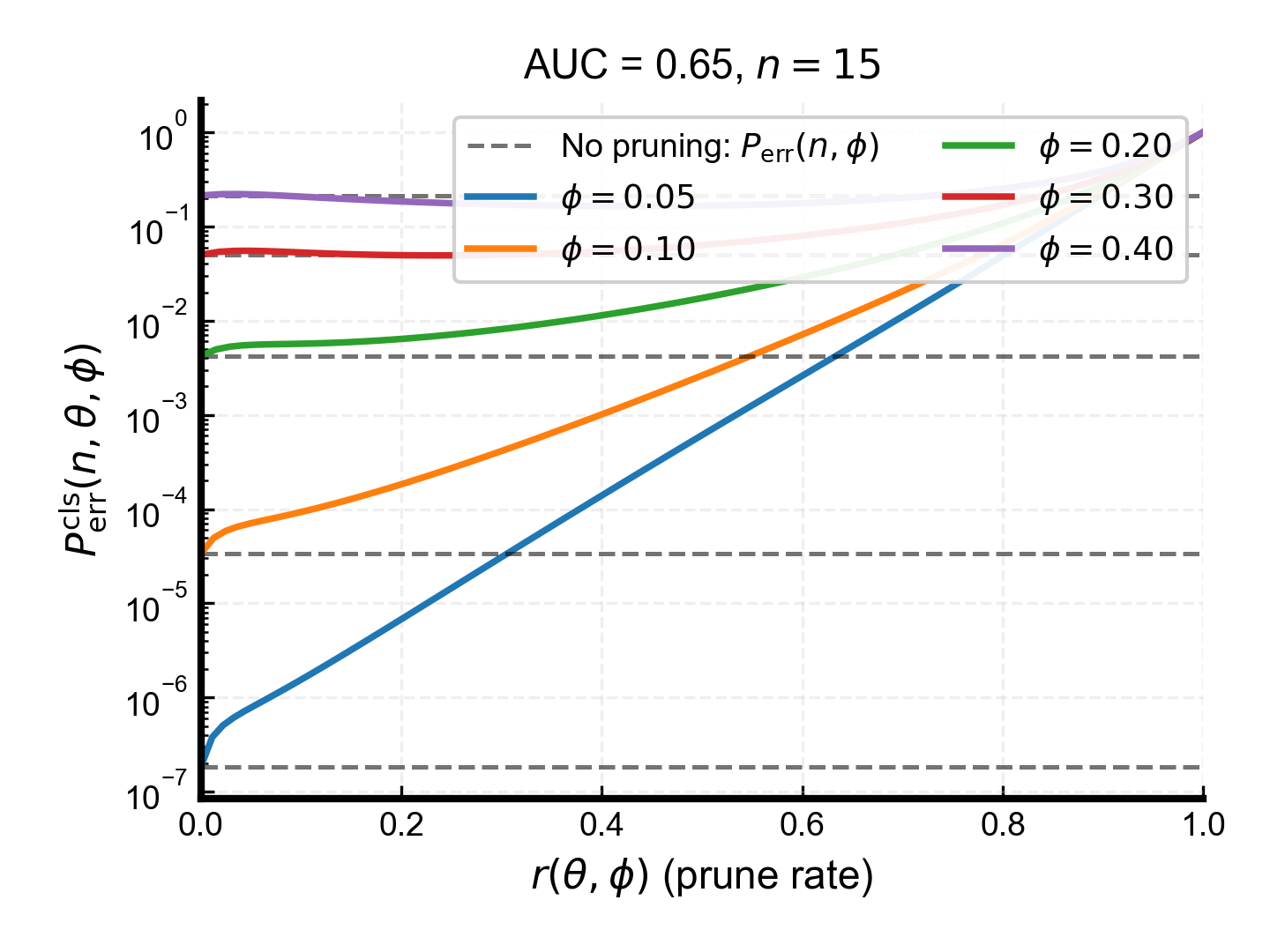}
\includegraphics[width=0.32\textwidth]{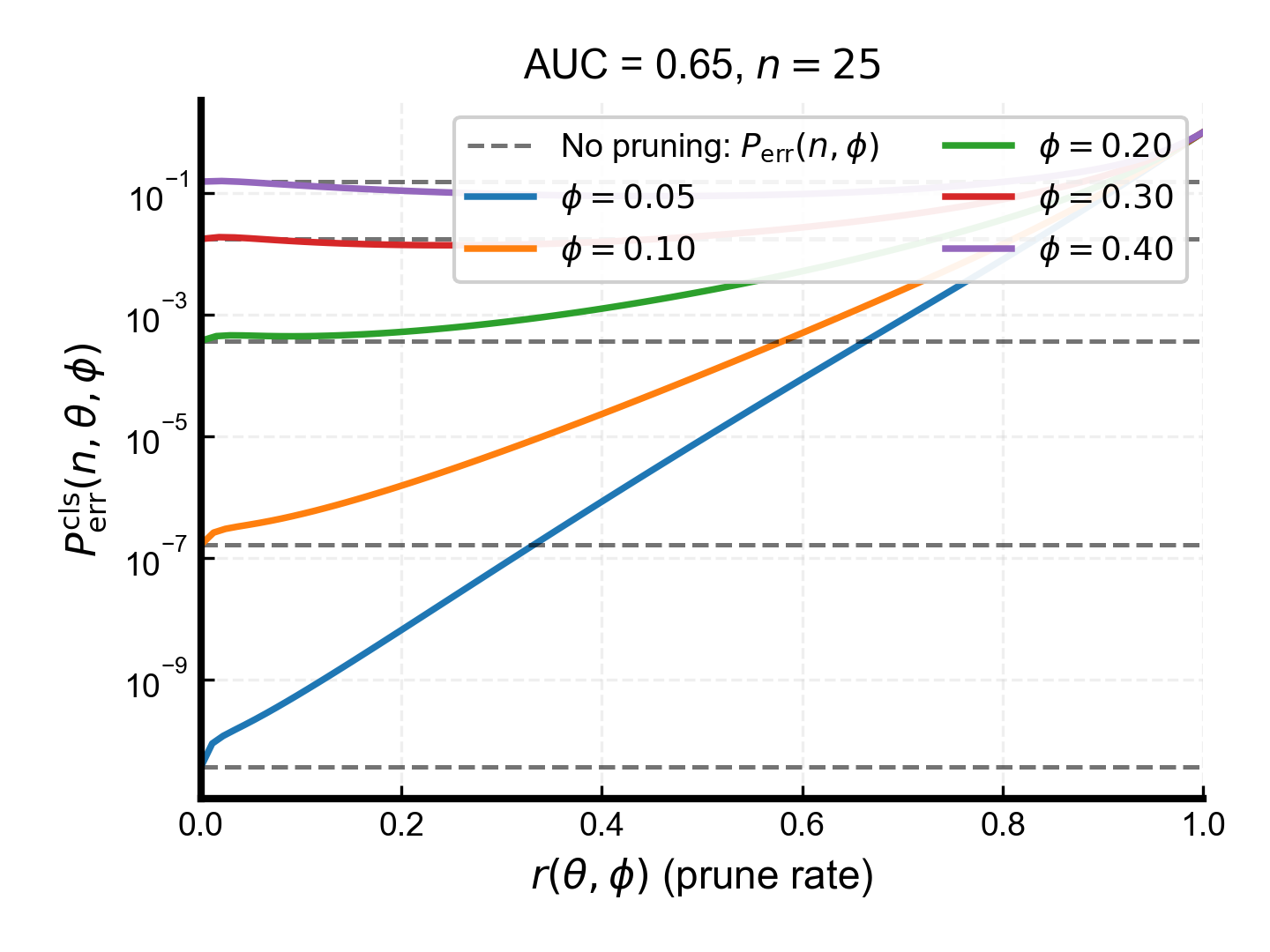}
\includegraphics[width=0.32\textwidth]{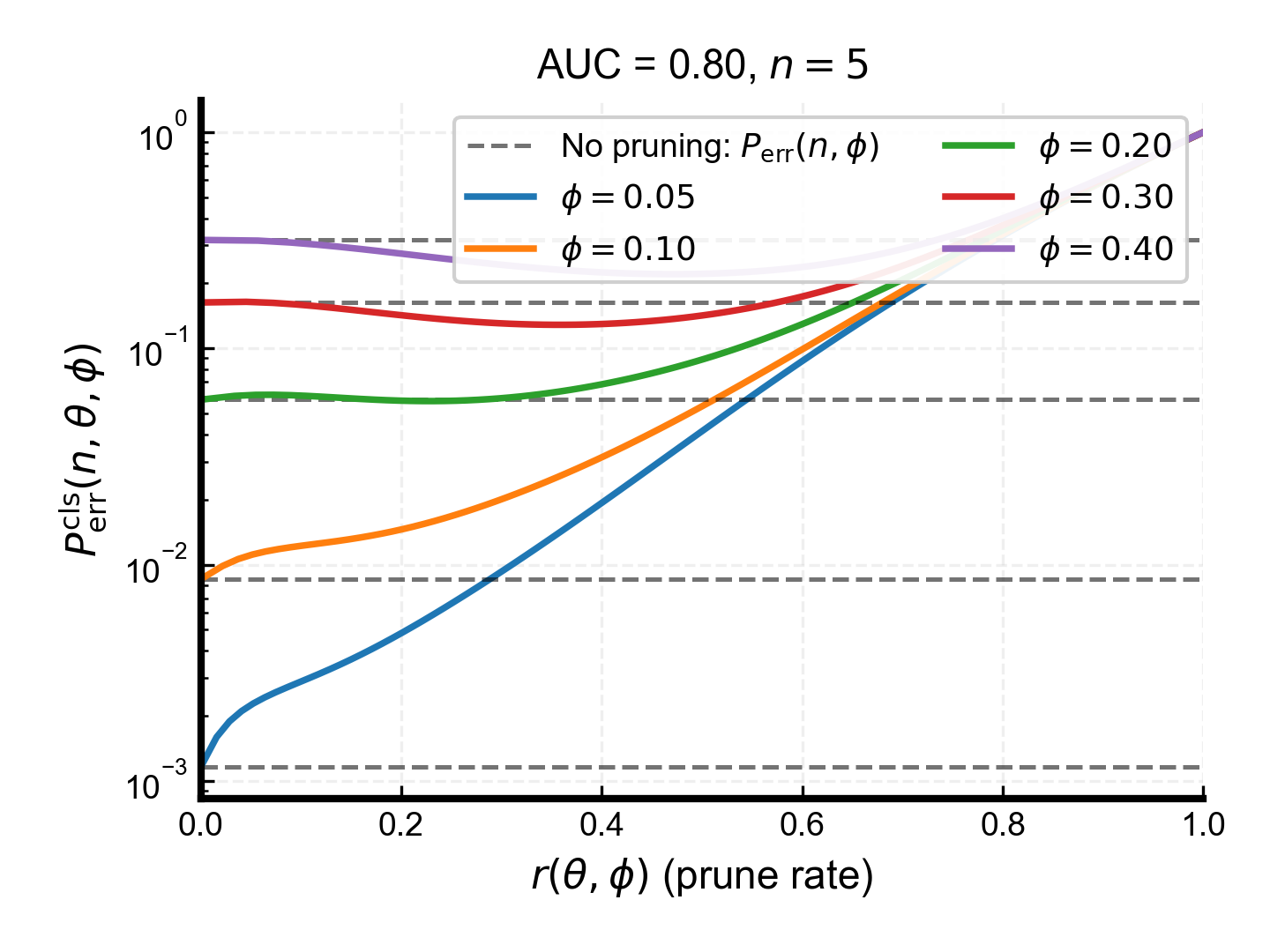}
\includegraphics[width=0.32\textwidth]{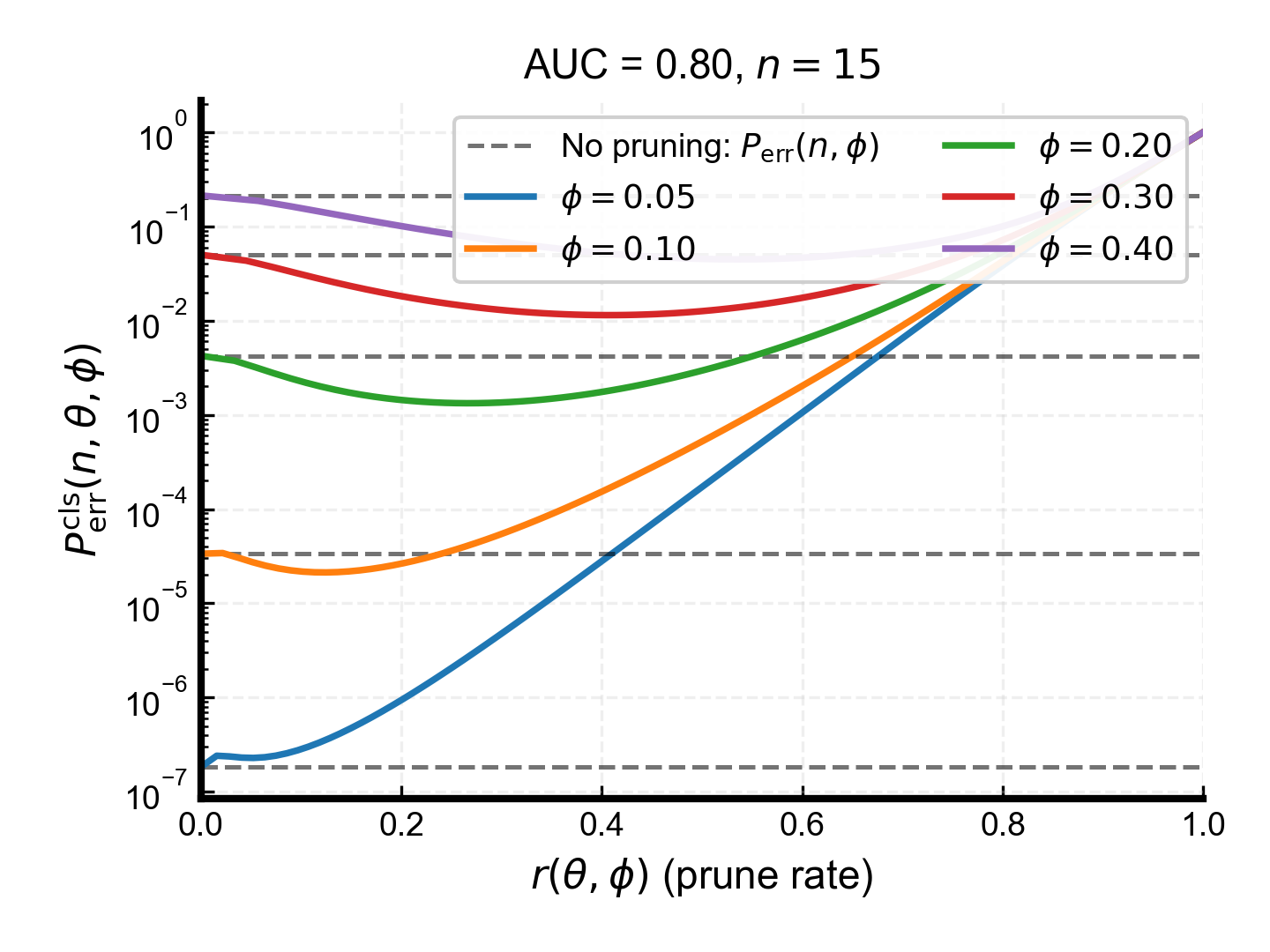}
\includegraphics[width=0.32\textwidth]{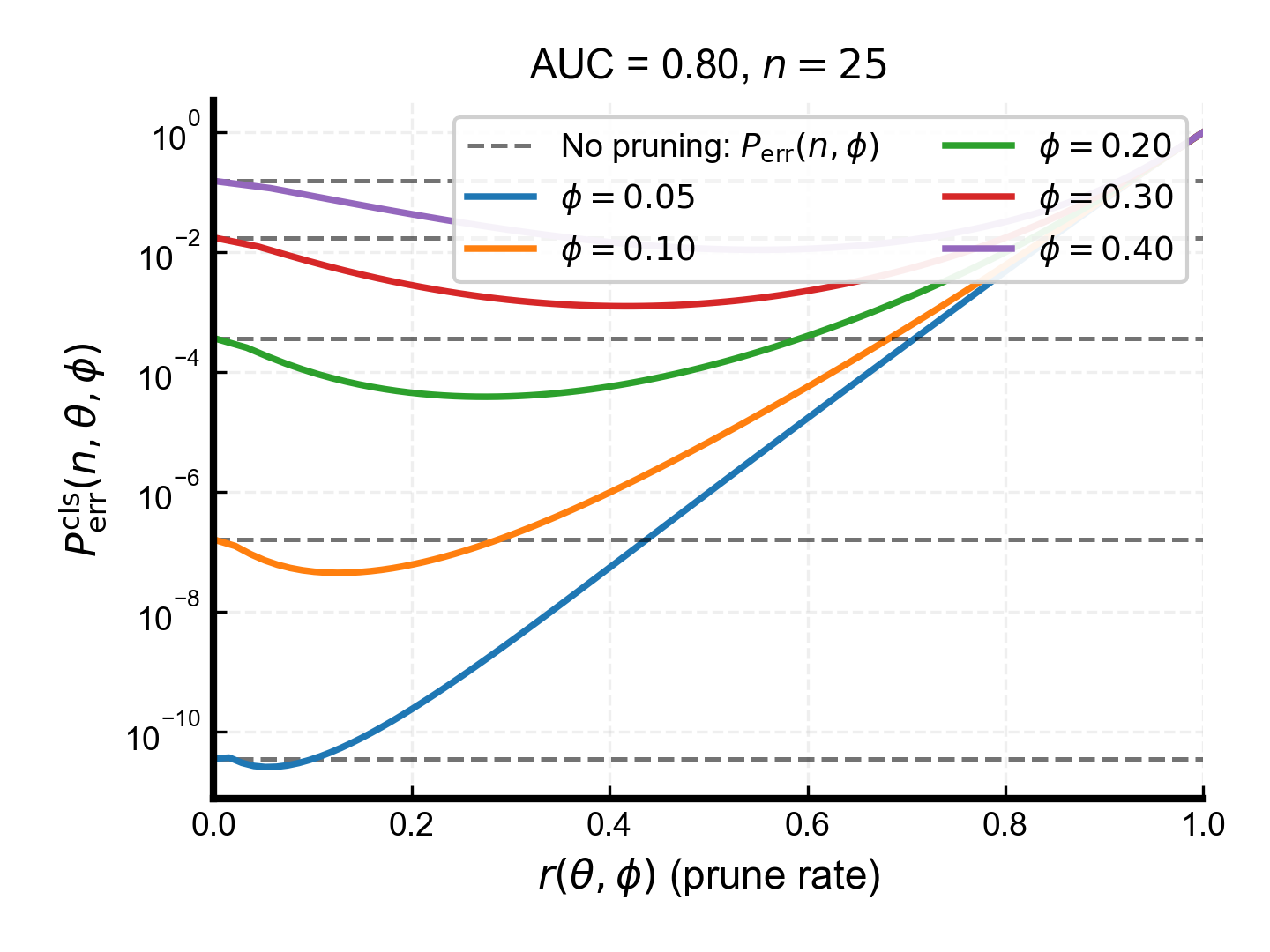}
\includegraphics[width=0.32\textwidth]{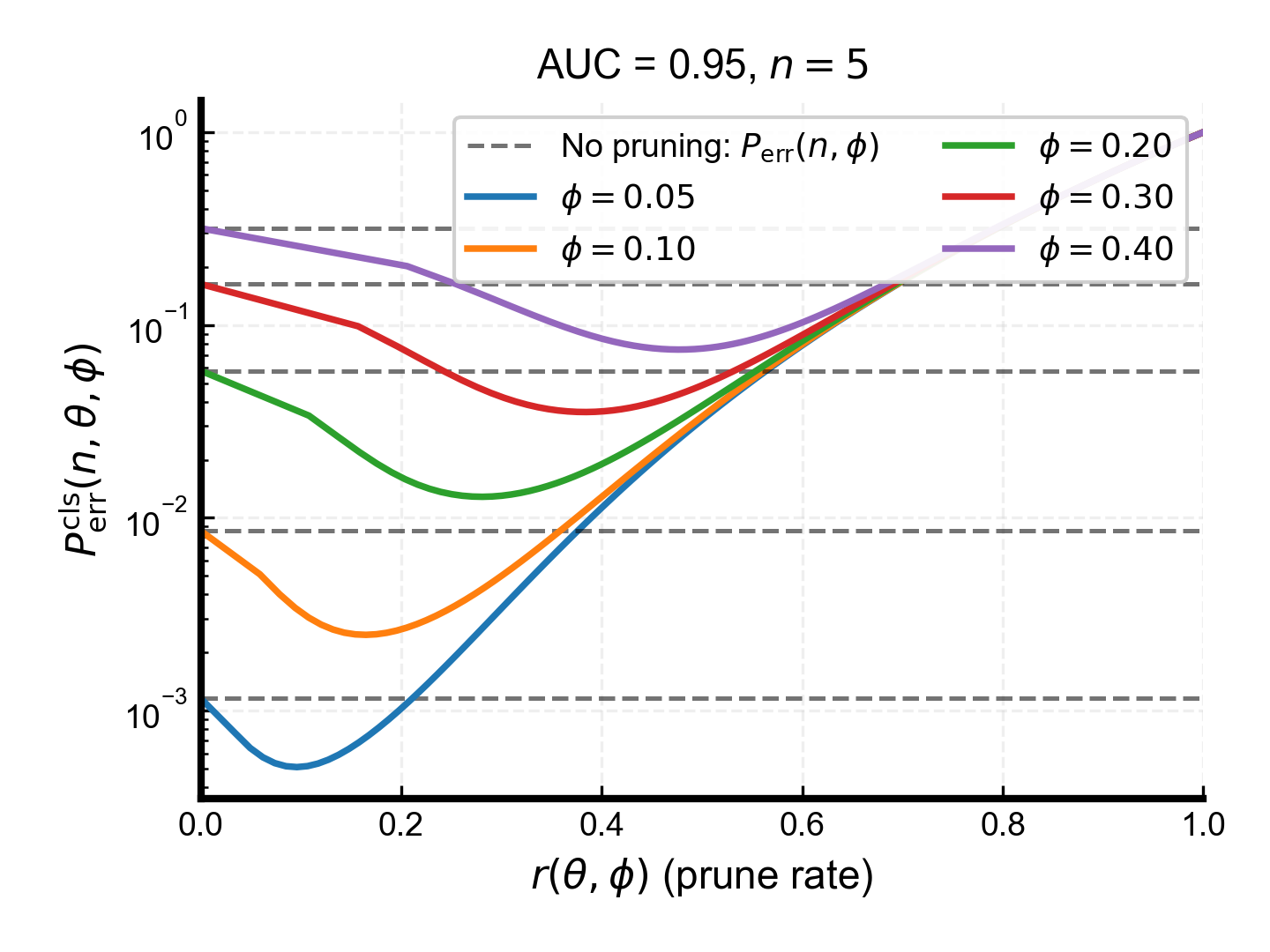}
\includegraphics[width=0.32\textwidth]{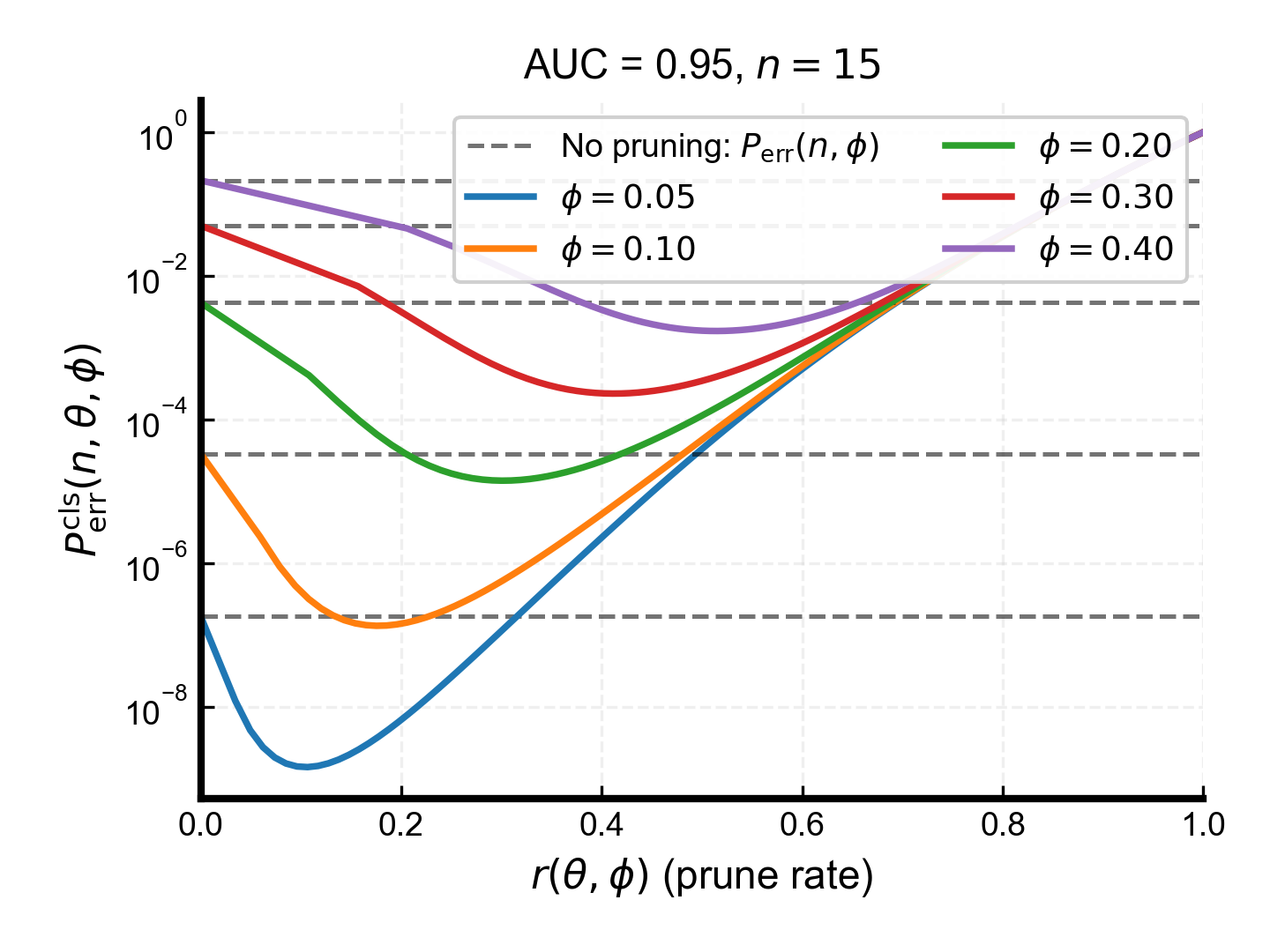}
\includegraphics[width=0.32\textwidth]{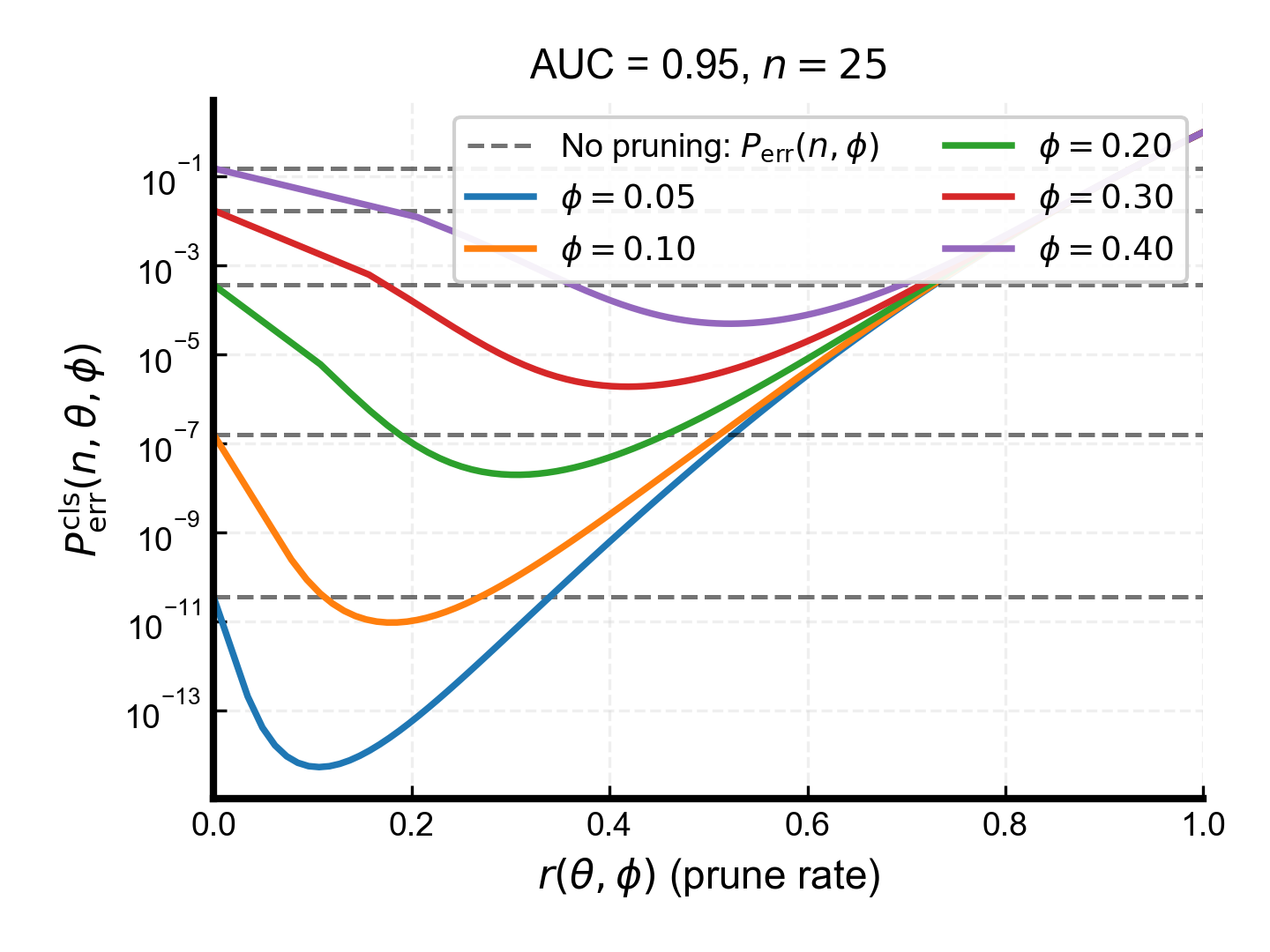}
\caption{Post-pruning error $P_{\mathrm{err}}^{\mathrm{cls}}(n,\theta,\phi)$ versus realized prune rate $r(\theta,\phi)$, where each curve sweeps $\theta$. Columns fix repeat count $n$ while rows use binormal ROC models with varying AUC. The dashed horizontal lines are no-pruning baselines $P_{\mathrm{err}}(n,\phi)$.}
\label{fig:theory_prune_rate}
\end{figure}

We now ask when the optimized pruning policy in $\CLS$ beats the optimized no-pruning policy in $\NOM$ at a fixed quality target $\epsilon$.
To compare the two problems, we define two information-rate quantities $\Ierr(\phi):=D_{KL}(1/2\|\phi)$ and $\Ierrcls(\theta,\phi):=- \log\!\left(r(\theta,\phi) + \bar{r}(\theta,\phi)e^{-\Ierr(\pi(\theta,\phi))}\right)$ where $D_{KL}(p\|q) := p \log(p/q) + (1-p)\log((1-p)/(1-q))$ is the KL divergence between two Bernoulli distributions with parameters $p$ and $q$.
We expand on $\Ierr(\phi)$ and $\Ierrcls(\theta,\phi)$ in the electronic companion (EC).
The quantity $\Ierr(\phi)$ is the per-repeat information rate under no pruning, where larger $\Ierr(\phi)$ means majority-vote error decays faster with additional repeats.
The quantity $\Ierrcls(\theta,\phi)$ is the effective information rate per allocated repeat under pruning, showing that with probability $r(\theta,\phi)$ a repeat is pruned and contributes no vote, while with probability $\bar r(\theta,\phi)$ a retained vote is drawn from a cleaner pool with minority report propensity $\pi(\theta,\phi)$.
Here, pruning helps when quality gains from lowering $\pi(\theta,\phi)$ dominate quantity losses from increasing $r(\theta,\phi)$.

We now state the main sufficient condition that yields strict finite-cost improvement for $\CLS$. The proof requires a series of auxiliary results that characterize the asymptotic behavior of the optimal costs $n^\dagger(\epsilon,\phi)$ and $n^*(\epsilon, \theta, \phi)$, which we leave in the EC.

\begin{thm}\label{thm:clf_asymptotically_better}
  For any fixed $\phi \in (0, 1/2)$, if there exists $\theta$ such that $(1 - r(\theta, \phi)) \Ierr(\phi) < \Ierrcls(\theta,\phi)$ then there exists $\epsilon_0\in(0,1)$ such that for all $\epsilon < \epsilon_0$, problem $\CLS$ attains a strictly lower expected cost than problem $\NOM$, i.e., $n^*(\epsilon, \theta^*, \phi)(1-r(\theta^*,\phi)) < n^\dagger(\epsilon,\phi)$.
\end{thm}

Theorem \ref{thm:clf_asymptotically_better} shows that annotation cost under $\CLS$ can be strictly lower than under $\NOM$.
This holds under a condition that the information rate under pruning exceeds the effective information of a policy with the same fraction of annotations. When the right-hand side of the inequality dominates, there exists a high-accuracy regime in which pruning can achieve the same target error with a strictly lower expected number of paid annotations. This condition holds over broad ranges of task difficulty and pruning thresholds, even for weak classifiers. The below example illustrates this result.

\begin{example}
We visualize the sufficient condition in the same binormal ROC models as the previous example. Figure~\ref{fig:condition_for_RPA} plots the region of $\theta$ versus $\phi$ that satisfy the sufficiency condition.
For larger $\phi$, nominal allocation needs more repeats and pruning has more opportunity to reduce cost by filtering likely minority reports before spending on execution.
Moreover, the valid region where $\CLS$ excels expands as classifier quality improves, since higher AUC translates into cleaner retained votes at comparable prune rates.

Figure~\ref{fig:nom_over_cls} numerically compares the relative cost ratio between the two strategies
\begin{align*}
  \frac{\NOM}{\CLS} := \frac{n^\dagger(\epsilon, \phi)}{n^*(\epsilon, \theta^*, \phi)(1-r(\theta^*,\phi))}
\end{align*}
across a grid of $\epsilon,\phi$ for the same three binormal ROC models.
Ratios above $1$ indicate that $\CLS$ is cheaper than $\NOM$, while values near $1$ indicate limited incremental benefit from pruning.
\end{example}

\begin{figure}[t!]
\centering
\includegraphics[width=0.32\textwidth]{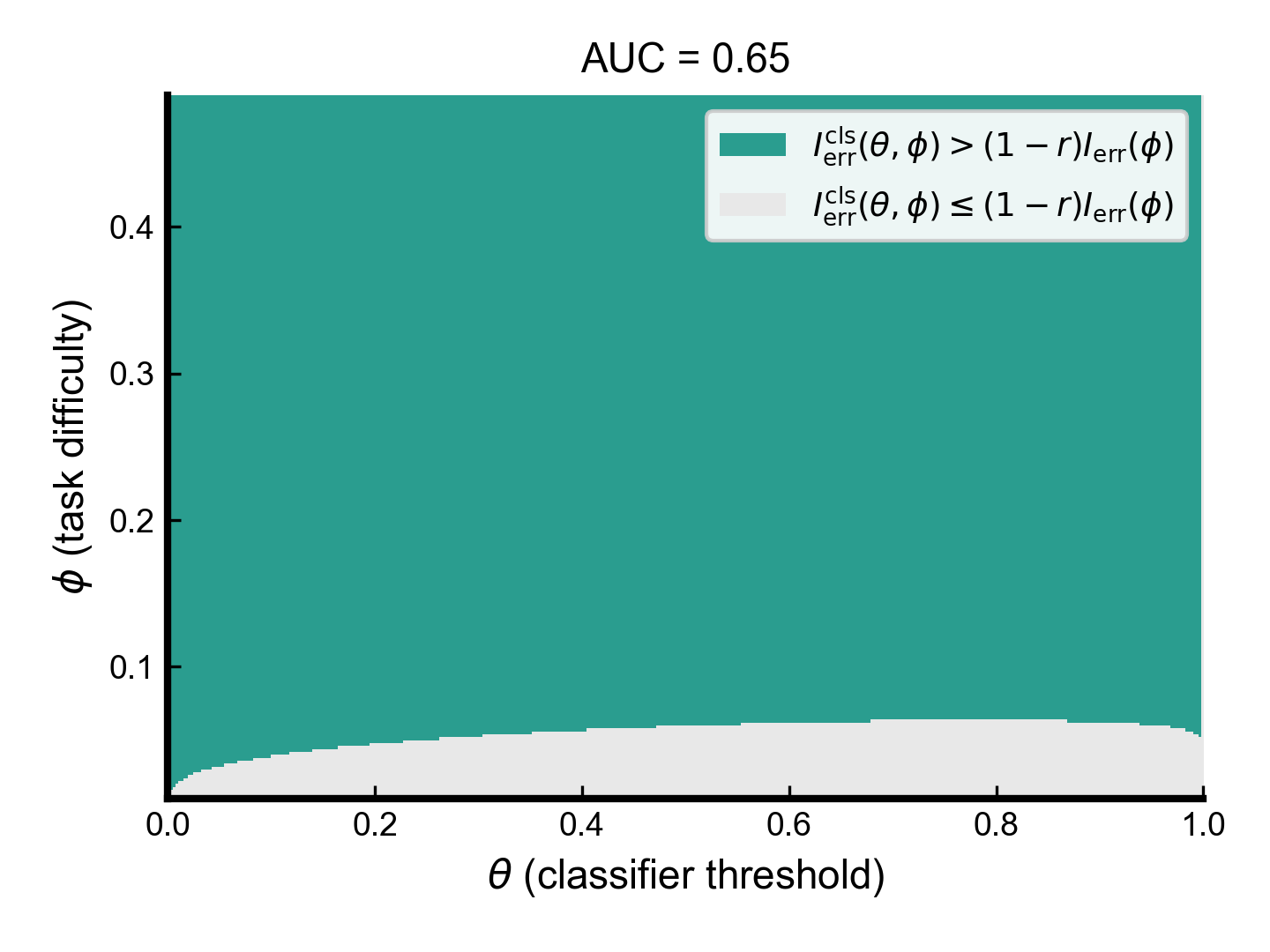}
\includegraphics[width=0.32\textwidth]{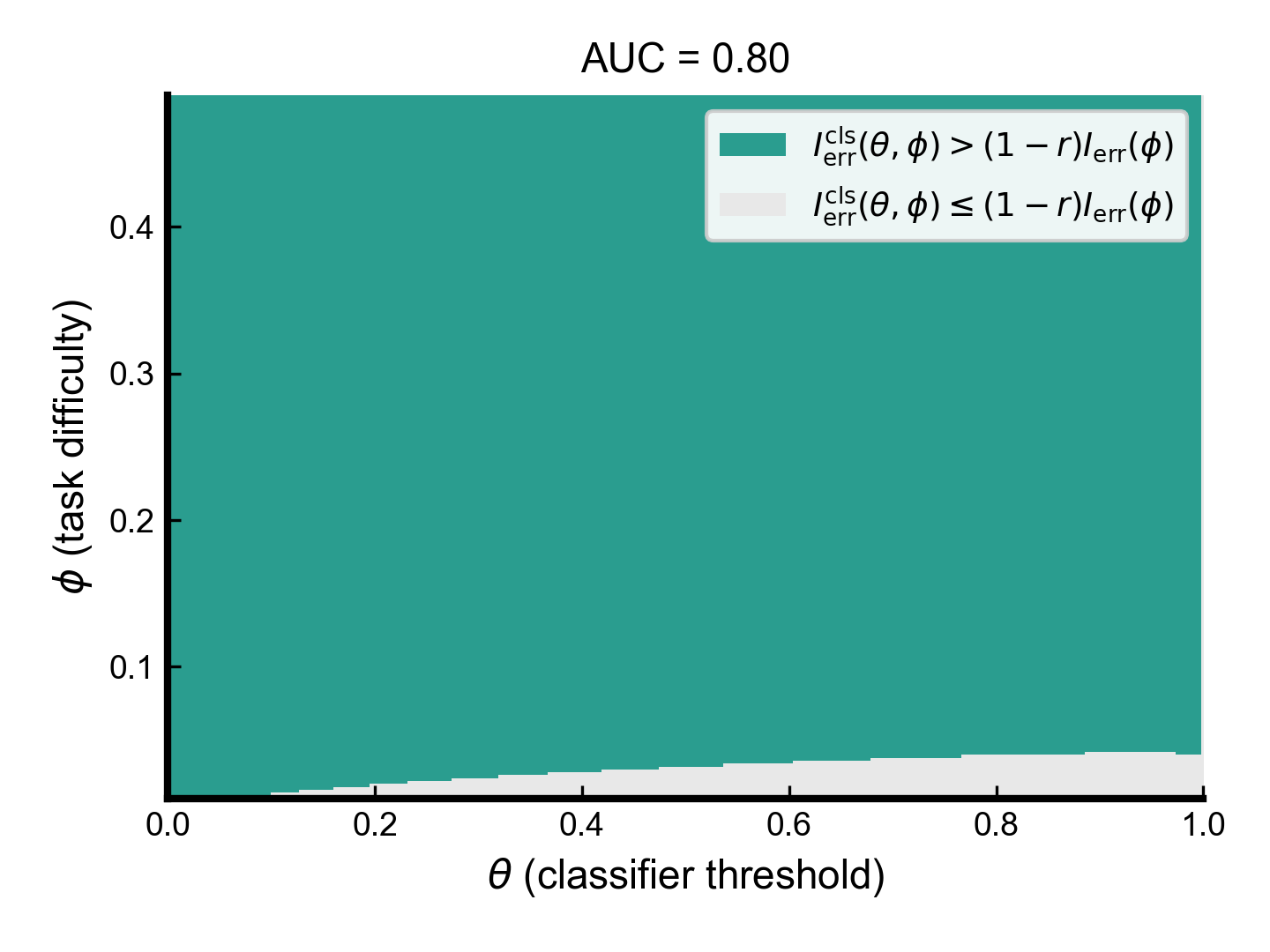}
\includegraphics[width=0.32\textwidth]{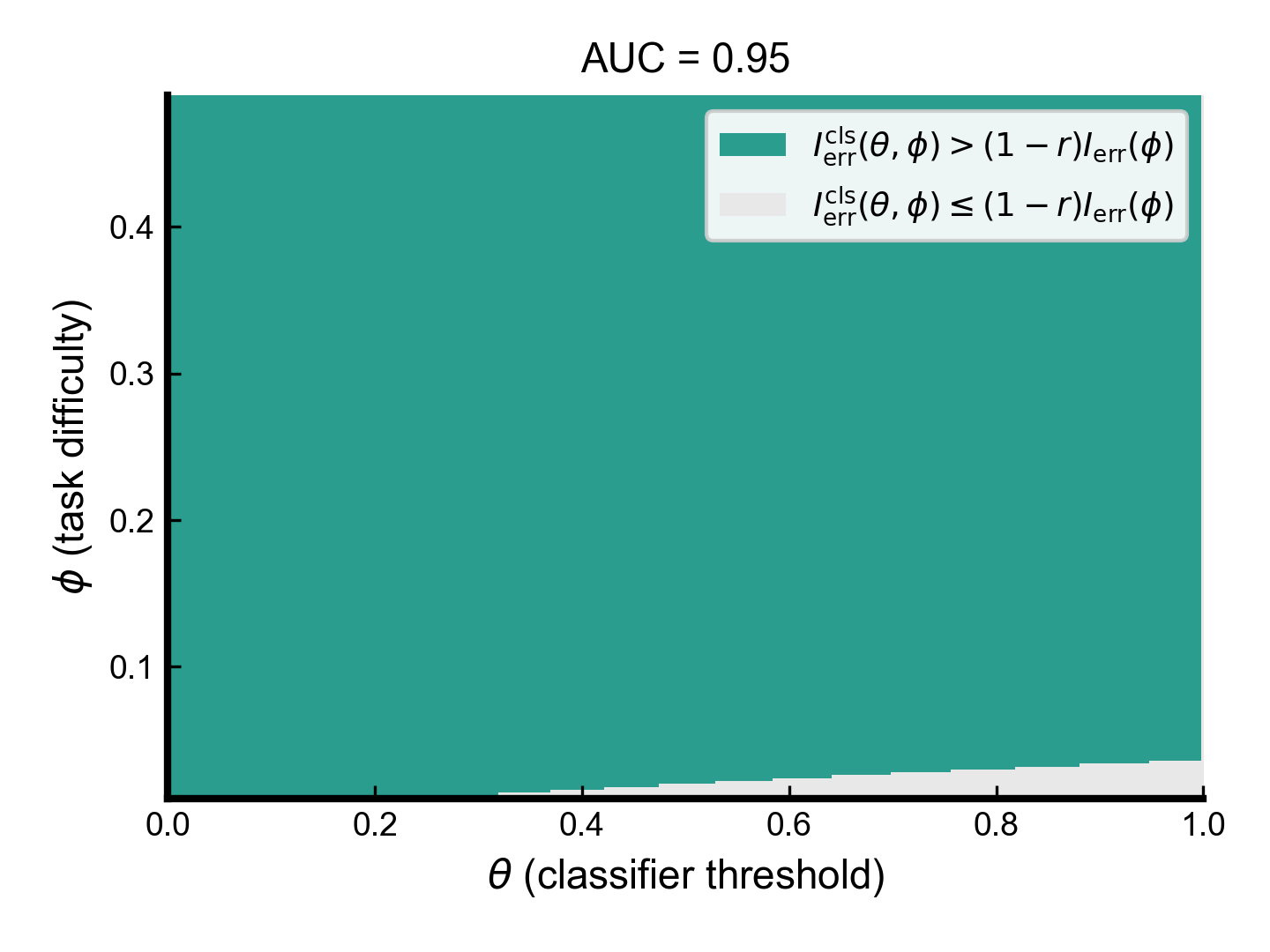}
\caption{Sufficient-condition map for Theorem~\ref{thm:clf_asymptotically_better} over classifier threshold $\theta$ and minority report propensity $\phi$, for binormal ROC models. The teal regime shows where the condition is satisfied and the gray regime shows where it is not satisfied.}
\label{fig:condition_for_RPA}
\end{figure}

\begin{figure}[t!]
\centering
\includegraphics[width=0.32\textwidth]{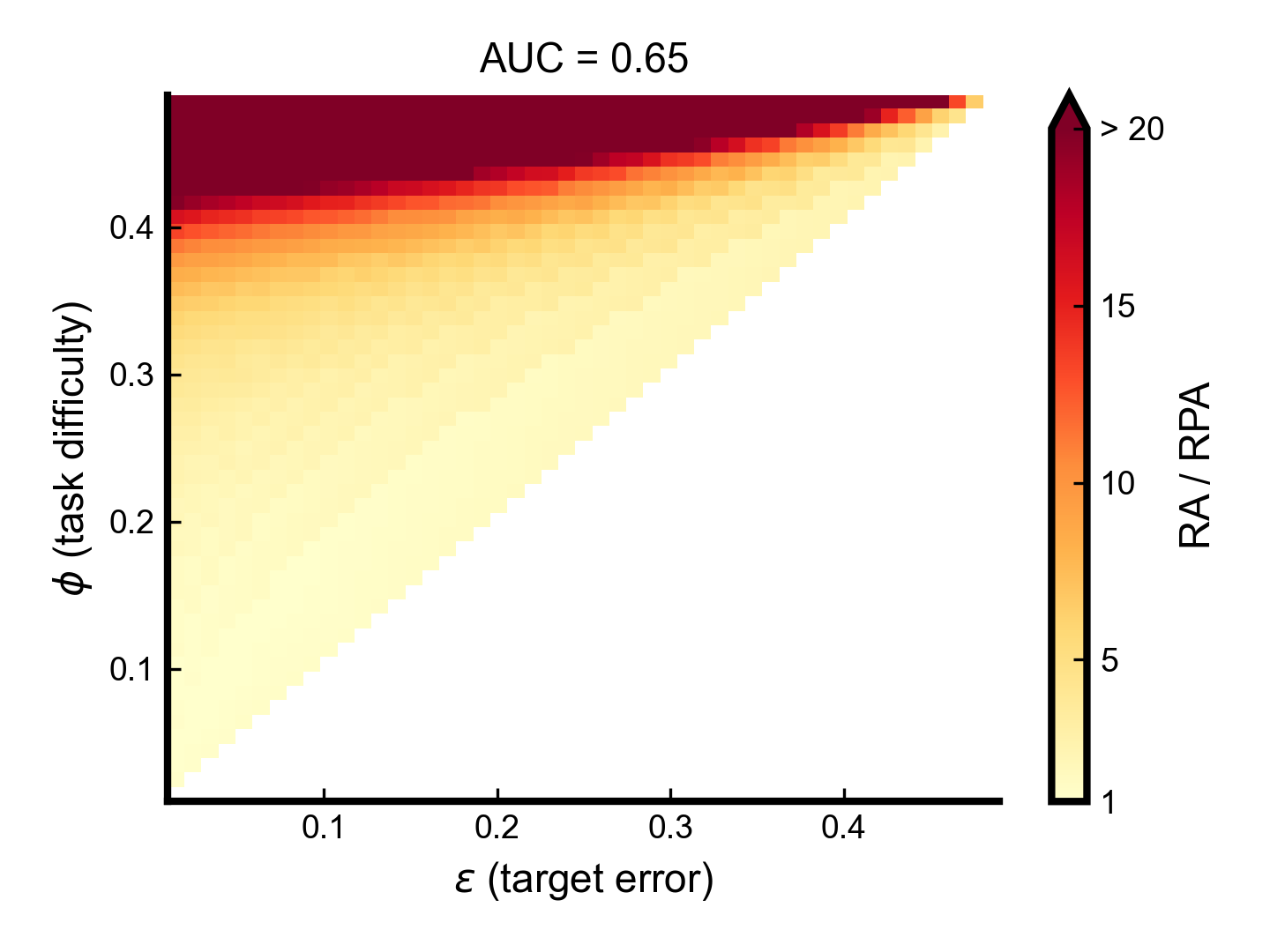}
\includegraphics[width=0.32\textwidth]{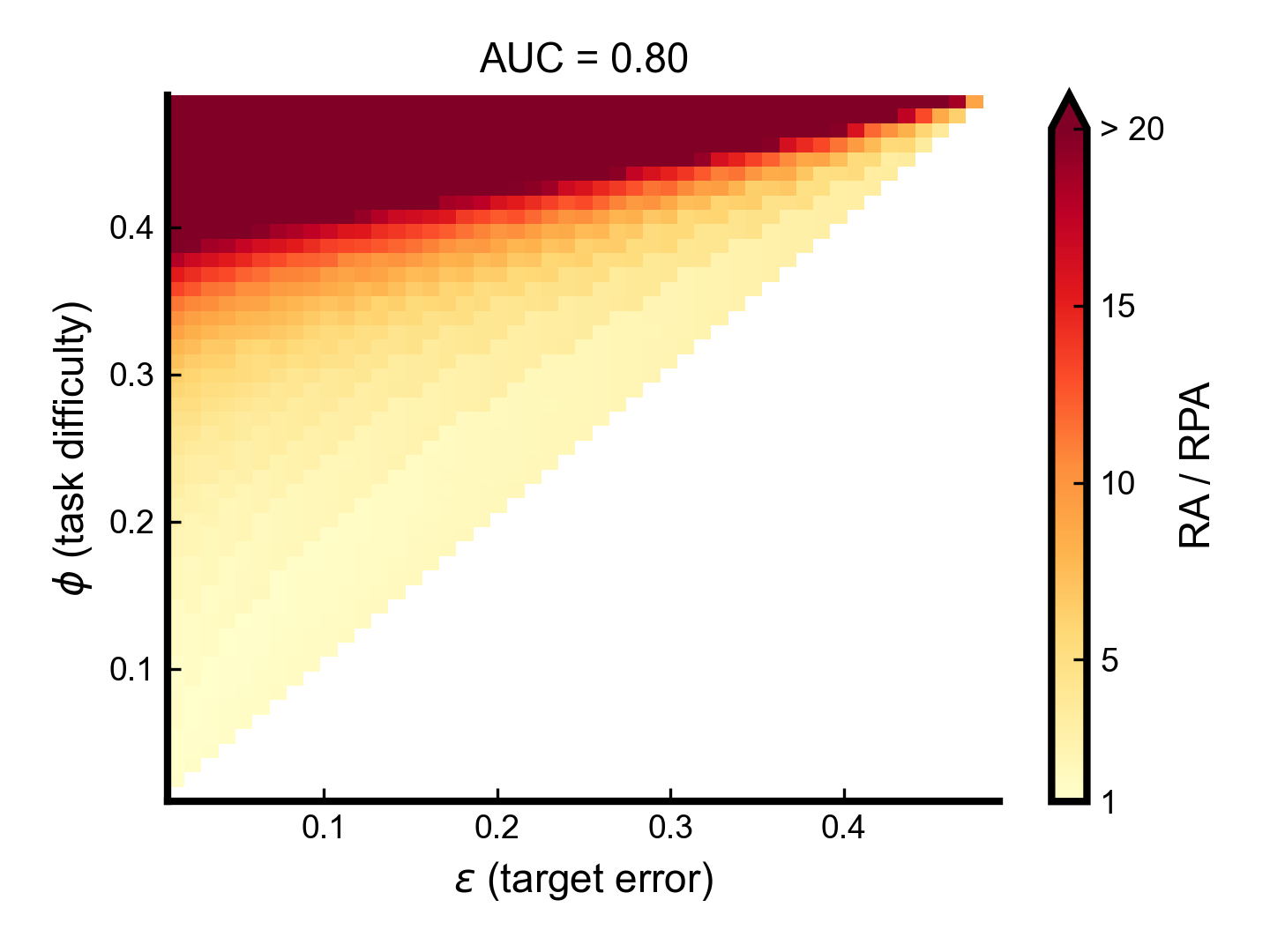}
\includegraphics[width=0.32\textwidth]{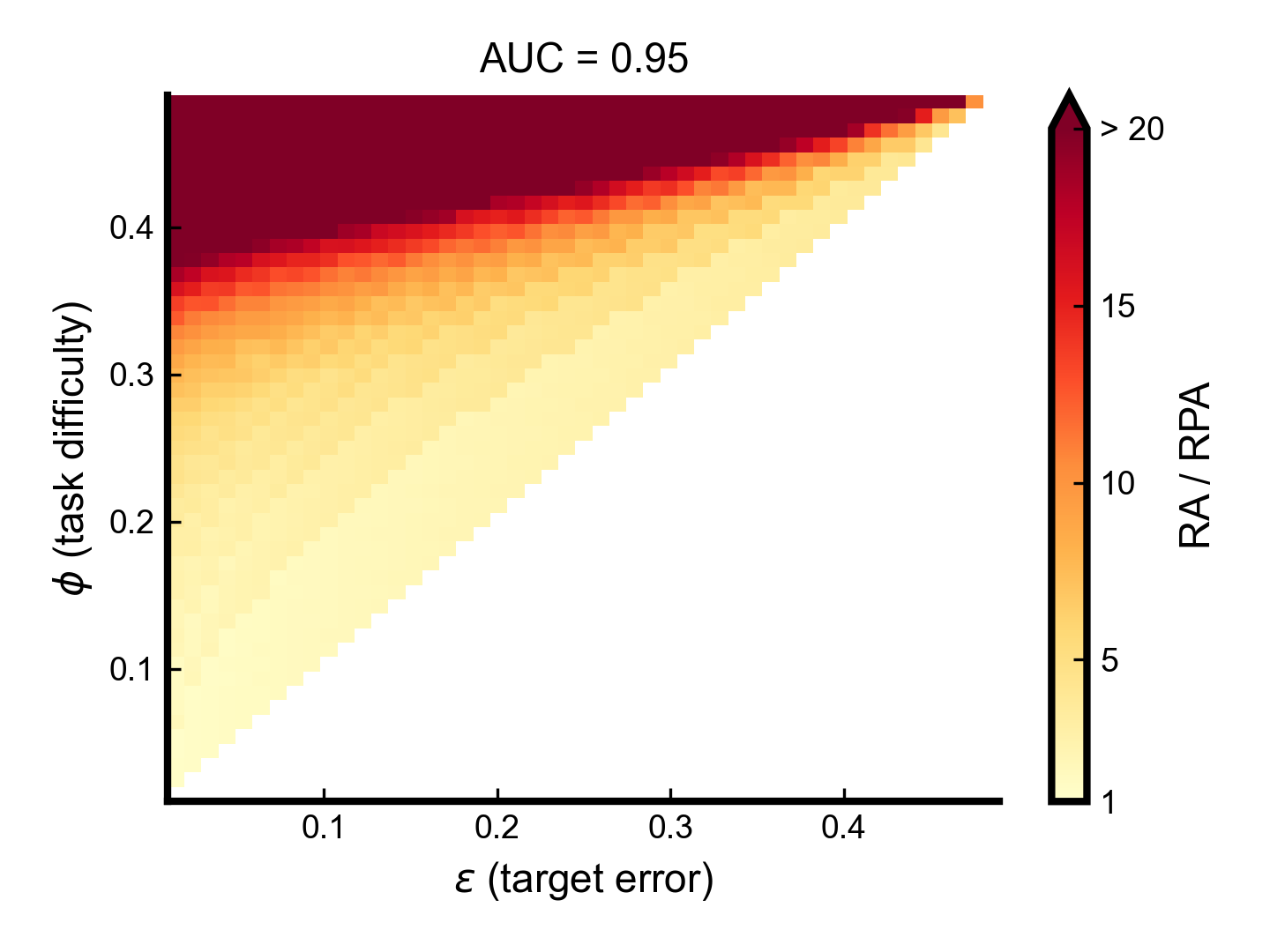}
\caption[Relative optimization cost ratio over $\epsilon$ versus $\phi$ for three binormal ROC models.]{Relative optimization cost ratio over $\epsilon$ versus $\phi$ for three binormal ROC models. White cells mask the region $\epsilon > \phi$}
\label{fig:nom_over_cls}
\end{figure}

\section{Empirical analysis of minority reports}
\label{sec:empirical}

Having established the conditions where pruning can lower annotation cost, in this section, we determine how to estimate the risk of minority reports. We first describe our public source datasets which have been re-annotated. We then evaluate a mixed-effects predictive model of minority report risk.

\subsection{Setting and data}
\label{sec:setting_datasets}

We analyze the effect of annotator performance using two benchmark image recognition datasets, the EuroCity Persons Dataset (ECPD) \citep{braun2019eurocity} and the Zenseact Open Dataset (ZOD) \citep{alibeigi2023zenseact}. ECPD is a public-access dataset containing recorded videos from the perspective of a front camera on-board vehicles driving in urban traffic environments in 12 European countries. ZOD is a public-access dataset that contains simultaneous recorded videos of multiple cameras arranged above vehicles driving around 14 European countries. The video recordings in each dataset are split into images (also referred to as frames). There are 47,000 images in ECPD and over 100,000 images in ZOD. For each image, different objects of interest are marked via crops (also referred to as bounding boxes) that describe the coordinates of a box around each object in the image. ECPD contains crops for humans, separated into pedestrians and riders (e.g., on motorbikes, bicycles, wheelchairs). ZOD includes crops for automobiles, bicycles, wheelchairs, pedestrians, and animals. Figure \ref{fig:ecpd_example} shows an example of an annotated frame from ECPD.

\begin{figure}[t!]
  \centering
  \includegraphics[width=0.75\linewidth]{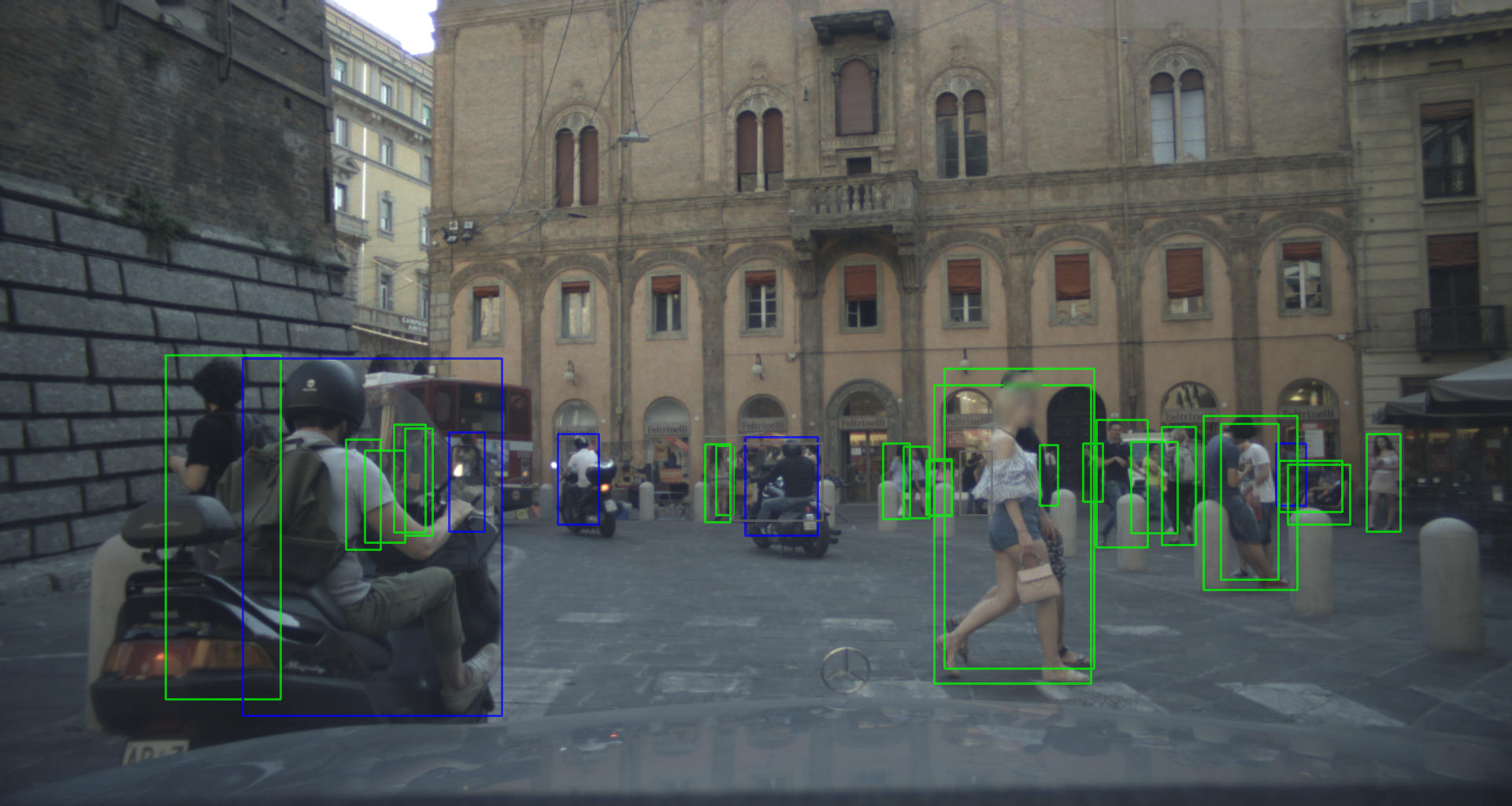}
  \caption{A frame from ECPD with annotated bounding boxes. The green boxes correspond to crops for pedestrians and the blue boxes to crops for motorcycle riders.}
  \label{fig:ecpd_example}
\end{figure}

The original annotator votes for these datasets are not publicly available.
Consequently, for each source dataset, a professional annotation platform constructed new annotation tasks and collected repeated responses.
Table \ref{tab:dataset_statistics} summarizes the key statistics for both datasets. For annotation, we focus on the problem of crops containing vulnerable road users, i.e., pedestrians, cyclists, and motorcyclists, which are an important category of objects that must be detected by autonomous vehicle software. Each crop has associated questions which annotators can answer with `yes', `no', or `can't solve'.
Over a study period of multiple days, the platform observed a randomly chosen cohort of annotators using the professional platform. Annotators chose their own hours of employment, were given detailed annotation guidelines to follow, and completed a series of tasks as determined by the queue.
Every task was assigned to multiple annotators through the observed assignment process.
Finally, at the end of the study period, all instances were aggregated and the reference responses for each task were determined via majority vote.
In our analysis, we label each instance as a correct response if the response matches the majority vote, and a minority report if the response is different from the majority vote, or a `can't solve'.

\begin{table}[t!]
  \centering
  \small
  \begin{tabular}{lcc}
    \toprule
    & ECPD & ZOD \\ \midrule
    Observation period      & Jan. 18-23, 2023 & Jul. 16-17, 2024   \\
    Annotators              & 94        & 18   \\
    Images                  & 21,776          & 8,488   \\
    Crops                   & 32,741    & 16,000   \\
    Tasks per crop          & 6         & 1   \\
    Repeats per task & 5-12 & 10--11   \\
    Total repeats & 1,035,840 & 175,962 \\
    Frequency of responses (\%) & ~ & \\
    $\quad$ \emph{Yes}                  & 18.5 & 94.1 \\
    $\quad$ \emph{No}                   & 80.2 & 1.3 \\
    $\quad$ \emph{Can't solve} & 1.3 & 4.6 \\
    Average time spent on task (sec) & 0.9 & 1.1 \\
    Average minority report rate (\%) & 4.2 & 5.2
    \\ \bottomrule
  \end{tabular}
  \caption{Summary statistics of the annotation datasets before pre-processing.}
  \label{tab:dataset_statistics}
\end{table}

On ECPD, we first select a random subset of 21,776 images from the dataset, containing a total of 32,741 crops.
Every crop is associated with the following six questions:
\begin{enumerate}
\item[Q1)] Is the object in the crop a human being?

\item[Q2)] Is the object in the crop on a bicycle?

\item[Q3)] Is the object in the crop on a poster?

\item[Q4)] Is the object in the crop on wheels?

\item[Q5)] Is the object in the crop a reflection from a window or mirror?

\item[Q6)] Is the object in the crop a statue or mannequin?
\end{enumerate}
Annotation tasks were assigned to a randomly chosen cohort of 94 annotators over a study period from January 18 to January 23, 2023, obtaining a total of 1,035,840 task instances.
Each task was annotated by between five and twelve annotators, according to the internal recommendations of the annotation platform, which uses an algorithm to dynamically assign repeats based on task ambiguity.
We note that six workers had completed less than 100 tasks and we remove the repeats corresponding to these non-representative workers. After removing these repeats, 48 tasks had fewer than five repeats, making it difficult to determine a stable reference majority-vote label. Consequently, we removed these tasks entirely from the dataset. After pre-processing, our dataset contains a total of 1,035,491 responses.

We use ZOD to validate the robustness of analysis to different task assignment protocols, i.e., a common target repeat count across tasks.
Here, we first select a random subset of 8,488 images from the dataset. We then focus on only crops containing pedestrians and bicyclists within these images and ignore all other objects. Every crop is associated with the question: Is the object in the crop a pedestrian?
Consequently, an annotation task on ZOD is defined by only the crop. These tasks were completed by a randomly chosen cohort of 18 annotators over a two-day study period starting July 16, 2024. Every task on ZOD, other than 38 samples, is annotated by a fixed number of 11 annotators; we remove these 38 samples for simplicity.

We note that the empirical minority report rate in these experiments is measured relative to the reference majority vote, whereas Section~\ref{sec:ra_rpa_model} defines minority reports relative to an unobserved ground truth.
The consequence is that the empirical model estimates the risk that a planned repeat will disagree with the platform's reference label rather than an oracle truth.
Nonetheless, this objective is aligned with annotation platforms since the majority vote is the operational label available to an annotation platform for training an ex-ante minority report estimator.

\subsection{Modeling minority report risk}
\label{sec:empirical_model}

We use both datasets to explore the effect of factors affecting the likelihood of a minority report.
Specifically, minority report rates vary across crops and annotators in both datasets, with highly skewed distributions at each level. Visually ambiguous crops and differences in annotator skill can contribute to this heterogeneity \citep{raykar2010learning, karger2014budget}. Hourly minority report rates also vary over the study periods.
We refer to the EC for additional details on these observations.

We model the probability of a minority report as a function of crop ambiguity, worker heterogeneity, and worker activity.
Consider a set of image crops indexed by $i \in \{1, 2, \cdots, I\}$ where $I$ is the total number of image crops to annotate, and a set of annotators indexed by $j \in \{1, 2, \cdots, J\}$ where $J$ is the total number of annotators in the labor pool. Finally, suppose that for each image crop, we must annotate $K$ different questions, indexed by $k$.
Given an annotation task $(i, k)$ completed by worker $j$, let $y_{i,j,k} \in \{0, 1\}$ be a binary response where $y_{i,j,k} = 1$ denotes a minority report where the worker response is either a `can't solve' or the opposite response to the majority vote. Moreover, let $p_{i,j,k}$ denote the fitted logistic probability of the response being a minority report.
Consider the following logistic regression model:
\begin{align}
\label{eq:main_logreg_model}
\logit( p_{i,j,k} ) = u_i + v_j + t_{i,j,k} \beta_{t,1} + t^2_{i,j,k} \beta_{t,2} + \blambda^\tpose \bx_{i,j,k}
\end{align}
where $u_i$ and $v_j$ refer to latent random effect variables corresponding to the relative ambiguity in image crop $i$ and the relative skill level of worker $j$, respectively.
Furthermore, $t_{i,j,k}$ refers to the amount of time (in hours) that the annotator has been continuously active on the annotation platform, while $\beta_{t,1}$ and $\beta_{t,2}$ capture linear and quadratic associations with activity.
Given that the average time to completion of a task on ECPD is 0.9 seconds, we define a continuous period of activity as a period where the time to completion of any task has not exceeded 10 minutes. In the EC, we demonstrate the sensitivity of our model to different definitions of continuous activity. Finally, $\blambda^\tpose \bx_{i,j,k}$ includes the intercept and control variables. All models control for the type of question being asked (i.e., of the six questions in ECPD, some may be more challenging than others), as well as the day-of-experiment.

Our model treats minority-report indicators as conditionally independent given the random effects and controls. However, repeats share an estimated majority-vote label, and a worker's successive responses may remain temporally dependent.
Consequently, we report residual diagnostics in the EC.

\begin{table}[t!]
\centering
\small
\begin{tabular}{lccccc}
\toprule

& Base & A & A + W & A + C & A + W + C \\

\midrule

Continuous activity & --- & $0.399^{***}$ & $0.446^{***}$ & $0.491^{***}$ & $0.542^{***}$ \\
&  & $($0.020$)$ & $($0.022$)$ & $($0.022$)$ & $($0.022$)$ \\

Continuous activity$^2$ & --- & $1.428^{***}$ & $1.306^{***}$ & $1.319^{***}$ & $1.212^{***}$ \\
&  & $($0.007$)$ & $($0.008$)$ & $($0.008$)$ & $($0.008$)$ \\

\midrule

Crop effect $u_i$ (SD)          & --- & --- & ---               & 1.258 & 1.261         \\
Worker effect $v_j$ (SD)        & --- & --- & 1.228  & ---               & 1.204         \\

\midrule

Human being & $13.415^{***}$ & $11.401^{***}$ & $11.318^{***}$ & $13.742^{***}$ & $13.892^{***}$ \\
& $($0.034$)$ & $($0.034$)$ & $($0.037$)$ & $($0.035$)$ & $($0.035$)$ \\

On bike & $27.515^{***}$ & $21.807^{***}$ & $26.420^{***}$ & $12.706^{***}$ & $16.353^{***}$ \\
& $($0.045$)$ & $($0.045$)$ & $($0.063$)$ & $($0.047$)$ & $($0.065$)$ \\

On wheels & $4.966^{***}$ & $4.389^{***}$ & $3.336^{***}$ & $6.057^{***}$ & $4.679^{***}$ \\
& $($0.037$)$ & $($0.037$)$ & $($0.041$)$ & $($0.039$)$ & $($0.040$)$ \\

Poster & $0.638^{***}$ & $0.577^{***}$ & $0.440^{***}$ & $0.590^{***}$ & $0.458^{***}$ \\
& $($0.043$)$ & $($0.043$)$ & $($0.045$)$ & $($0.043$)$ & $($0.041$)$ \\

Statue/mannequin & $2.978^{***}$ & $2.824^{***}$ & $1.598^{***}$ & $2.828^{***}$ & $1.523^{***}$ \\
& $($0.031$)$ & $($0.032$)$ & $($0.037$)$ & $($0.032$)$ & $($0.035$)$ \\

Jan 18 & $0.569^{***}$ & $0.604^{***}$ & $1.198$ & $1.541^{***}$ & $4.933^{***}$ \\
& $($0.065$)$ & $($0.066$)$ & $($0.196$)$ & $($0.070$)$ & $($0.183$)$ \\

Jan 19 & $0.069^{***}$ & $0.091^{***}$ & $0.037^{***}$ & $0.381^{***}$ & $0.152^{***}$ \\
& $($0.052$)$ & $($0.052$)$ & $($0.072$)$ & $($0.057$)$ & $($0.076$)$ \\

Jan 20 & $0.310^{***}$ & $0.334^{***}$ & $0.244^{***}$ & $0.838^{***}$ & $0.628^{***}$ \\
& $($0.035$)$ & $($0.036$)$ & $($0.041$)$ & $($0.039$)$ & $($0.043$)$ \\

Jan 21 & $0.302^{***}$ & $0.344^{***}$ & $0.200^{***}$ & $0.674^{***}$ & $0.394^{***}$ \\
& $($0.027$)$ & $($0.028$)$ & $($0.032$)$ & $($0.031$)$ & $($0.034$)$ \\

Jan 22 & $0.312^{***}$ & $0.347^{***}$ & $0.126^{***}$ & $0.470^{***}$ & $0.239^{***}$ \\
& $($0.058$)$ & $($0.058$)$ & $($0.134$)$ & $($0.061$)$ & $($0.126$)$ \\

\midrule

Observations                            & 1,035,491       & 1,035,491       & 1,035,491             & 1,035,491             & 1,035,491             \\
Pseudo-$R^2$                 & 0.0906          & 0.0968           & 0.1667                & 0.1527               & 0.2203                \\

\bottomrule
\end{tabular}
\caption{Comparison of odds ratios for models on ECPD. Standard errors for the log-odds are shown in parentheses. For random effects, we report the standard deviation (SD) over log odds-ratios. Significance levels: *** $p<0.001$, ** $p<0.01$, * $p<0.05$.}
\label{tab:model_abcd_ecpd}
\end{table}

\begin{table}[t!]
\centering
\small
\begin{tabular}{lccccc}
\toprule

& Base & A & A + W & A + C & A + W + C \\

\midrule

Continuous activity & --- & $1.023$ & $0.963$ & $1.022$ & $0.981$ \\
&  & $($0.047$)$ & $($0.050$)$ & $($0.055$)$ & $($0.056$)$ \\

Continuous activity$^2$ & --- & $1.103^{***}$ & $1.016$ & $1.149^{***}$ & $1.011$ \\
&  & $($0.019$)$ & $($0.019$)$ & $($0.023$)$ & $($0.022$)$ \\

\midrule

Crop effect $u_i$ (SD) & --- & --- & --- & $2.456$ & $2.540$ \\

Worker effect $v_j$ (SD) & --- & --- & $1.493$ & --- & $1.640$ \\

\midrule

Jul 16 & $1.144^{***}$ & $1.074^{**}$ & $1.034$ & $1.085^{**}$ & $1.093^{*}$ \\
& $($0.025$)$ & $($0.026$)$ & $($0.032$)$ & $($0.029$)$ & $($0.036$)$ \\

\midrule

Observations                            & 175,962       & 175,962      & 175,962             & 175,962             & 175,962            \\
Pseudo-$R^2$                 & 0.0004          & 0.0034           & 0.1589                & 0.1375               & 0.3188                \\

\bottomrule
\end{tabular}
\caption{Comparison of odds ratios for models on ZOD. Standard errors for the log-odds are shown in parentheses. For random effects, we report the standard deviation (SD) over log odds-ratios. Significance levels: *** $p<0.001$, ** $p<0.01$, * $p<0.05$.}
\label{tab:model_abcd_zod}
\end{table}

\begin{figure}[h!]
\centering
\begin{subfigure}[b]{0.49\textwidth}
\centering
\includegraphics[width=\textwidth]{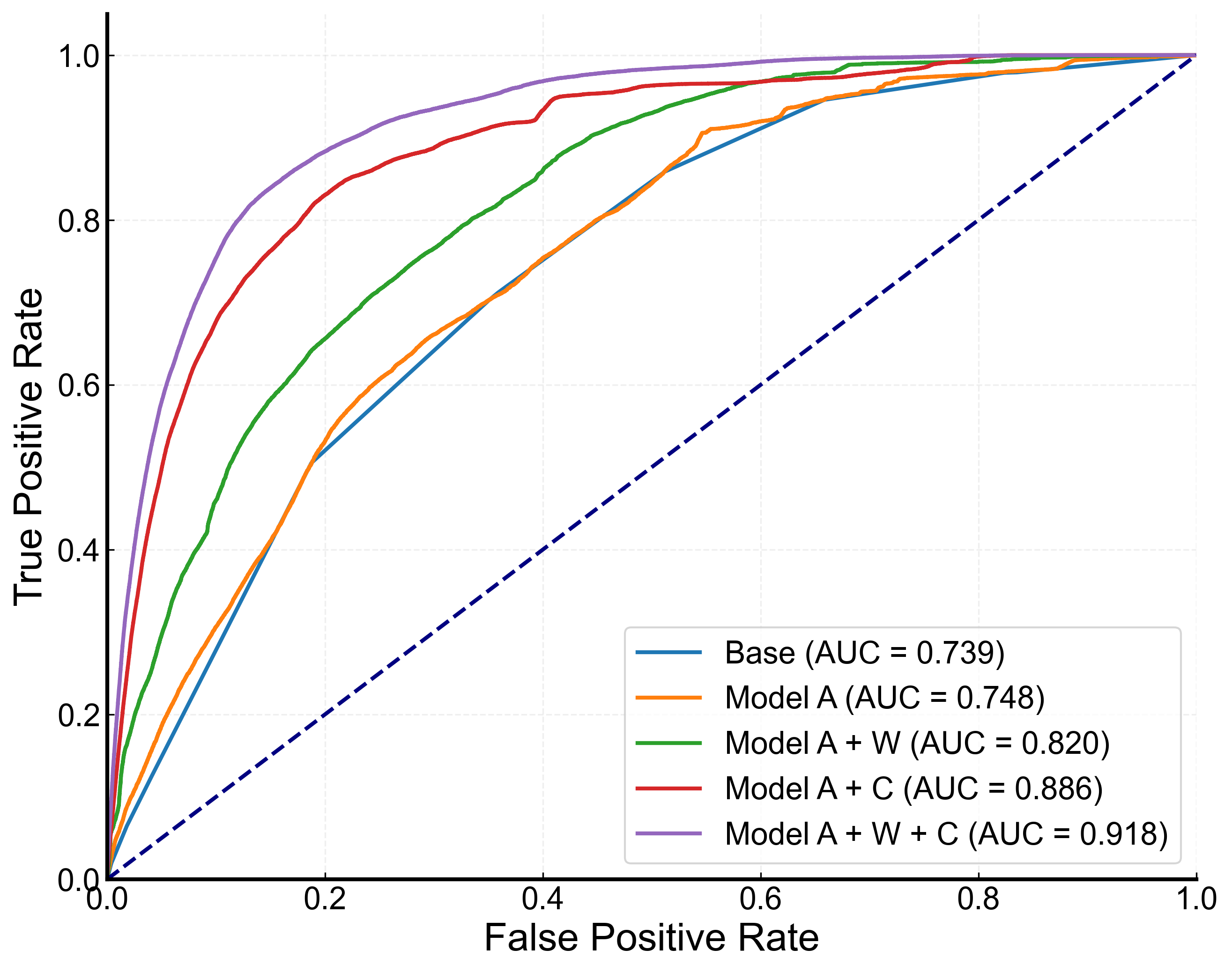}
\caption{ECPD}
\label{fig:roc_curves_comparison_ecpd}
\end{subfigure}
\hfill
\begin{subfigure}[b]{0.49\textwidth}
\centering
\includegraphics[width=\textwidth]{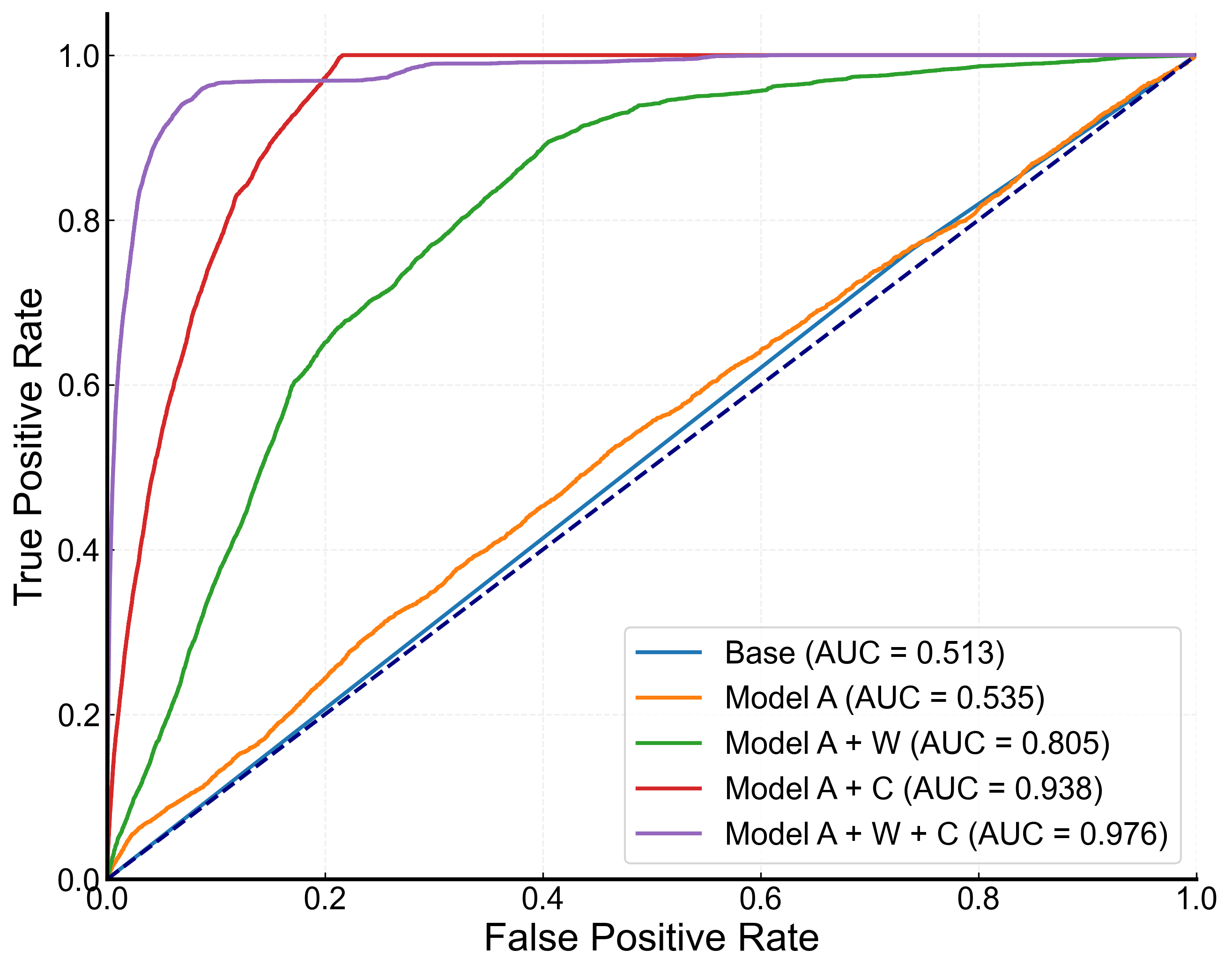}
\caption{ZOD}
\label{fig:roc_curves_comparison_zod}
\end{subfigure}
\caption{
ROC curves of model accuracy over both datasets. We analyze ROC in-sample to ensure all models are evaluated on the same dataset and have access to worker and crop random effects.
}
\label{fig:roc_curves_comparison}
\end{figure}

Table \ref{tab:model_abcd_ecpd} summarizes the parameter estimates and goodness-of-fit of model \eqref{eq:main_logreg_model} on ECPD and Table \ref{tab:model_abcd_zod} summarizes the results on ZOD. For comparison, we develop five variants: (i) \emph{Base} containing only the control variables $\bx_{i,j,k}$; (ii) \emph{Activity (A)} containing the control and the continuous activity variables $t_{i,j,k}$; (iii) \emph{Activity + Worker (A + W)} which includes worker random effects $v_j$; (iv) \emph{Activity + Crop (A + C)} which includes crop random effects $u_i$; and (v) \emph{Activity + Worker + Crop (A + W + C)} which includes both worker and crop random effects. We report McFadden's Pseudo-$R^2$ for all models by comparing against a null model with no fixed or random effects.
Furthermore, Figure \ref{fig:roc_curves_comparison} visualizes ROC curves for each model over the two datasets. In the EC, we include additional results on our final Model A + W + C, including residual diagnostics.

We first discuss the key insights drawn from the ECPD study, before discussing the ZOD study.
For ECPD, the continuous activity variables are always statistically significant \pleooone~when included. The negative linear and positive quadratic coefficients imply a U-shaped association between continuous activity and conditional log-odds of a minority report. In Figure \ref{fig:roc_curves_comparison_ecpd}, Model A has a higher Area-Under-Curve (AUC) above the Base model by $0.009$. Finally, we compare Model A with the Base model by a Likelihood Ratio Test (LRT) to confirm that this model presents a statistically significant improvement over the baseline \pleooone.

Incorporating worker-level (A + W) and crop-level (A + C) random effects improves upon Model A.
The Pseudo-$R^2$ of these two models is larger than the Pseudo-$R^2$ of Model A by $1.7\times$ and $1.6\times$.
LRTs against Model A show that both Model A + W and Model A + C give statistically significant improvements over Model A \pleooone.
Finally, Figure \ref{fig:roc_curves_comparison_ecpd} compares the ROC curves of these models to show that both of these models individually improve the AUC by $0.072$ and $0.138$ respectively. Interestingly, crop-level effects improve the predictability of minority reports by a larger margin than worker-level effects.

Finally, we combine both random effects in Model A + W + C. This model maintains similar fixed effect odds ratios and standard deviations on the random effects, while raising the Pseudo-$R^2$ over Model A + W and Model A + C by $1.3\times$ and $1.4\times$, respectively. LRTs confirm that Model A + W + C improves over both Model A + W and Model A + C \pleooone. Most importantly from Figure \ref{fig:roc_curves_comparison_ecpd}, this complete model achieves an AUC of $0.918$.

We now discuss the results on ZOD. First, incorporating continuous activity variables improves upon the Base model \pleooone. The squared continuous activity term consistently has an odds ratio greater than $1$, similar to before, although not statistically significant for the full model.
Furthermore, incorporating both worker-level random effects (A + W) and crop-level random effects (A + C) improve upon Model A \pleooone, the Pseudo-$R^2$, and AUC (see Figure \ref{fig:roc_curves_comparison_zod}).
Finally, we fit the combined Model A + W + C, which significantly improves upon all previous models \pleooone, Pseudo-$R^2$, and AUC.

Both datasets show substantial crop and worker heterogeneity in minority-report risk.
Furthermore, both worker-level and crop-level variability materially affect the predictability of minority report events.
These two trends suggest that minority reports on certain tasks can be largely predicted by estimating crop ambiguity and worker skill levels.

\section{Evaluation of Repeat-Prune Allocation}
\label{sec:pruning}

In this section, we evaluate a rolling pruning policy motivated by Section~\ref{sec:ra_rpa_model}, using the risk model from Section~\ref{sec:empirical}.
We simulate this policy on the fixed ECPD assignment stream that was collected by the annotation platform. Consequently, we evaluate performance with respect to the realized majority-vote labels under a baseline of no pruning.
Our framework yields a cost-quality frontier where with hourly recalibration, the policy prunes over 35\% of repeats while preserving more than 99.9\% of task-level majority votes. We further show that selective pruning preserves historical labels better than simply requesting fewer annotations non-selectively at comparable costs.

\begin{algorithm}[!htbp]
\small
\caption{Simulating $\CLSD$ on the ECPD dataset}
\label{alg:simulation_pruning}
\begin{algorithmic}[1]
\State \textbf{Input:} Set of assigned tasks $\set{I}_t := \{ (i,j,k) \}_t$ per each time period $t \in [T]$
\State \textbf{Initialize:} Initial observation period $\tau < T$, Pruning parameter $\theta \in (0, 1]$, Minority-fraction threshold $\rho$, Observed annotations $\set{D} = \emptyset$, Classification model $p(\cdot)$
\For{$t \in \{1, \cdots, \tau\}$}
\State Observe all annotations and update $\set{D} \gets \set{D} \cup \set{I}_t$
\EndFor
\State Compute minority reports for all observed annotations in $\set{D}$ and fit model $p(\cdot)$
\For{$t \in \{\tau + 1, \cdots, T\}$}
\For{Each assigned task $(i,j,k) \in \set{I}_t$}
\If{both task instance $(i,k)$ and worker $j$ are observed in the training data for $p(\cdot)$}
\If{$p(i,j,k) \geq 1-\theta$ and $(i,k)$ has at least $3$ strictly earlier retained binary votes with minority fraction below $\rho$}
\State Prune assignment $\set{I}_t \gets \set{I}_t \setminus \{ (i,j,k) \}$
\EndIf
\EndIf
\State If $(i,j,k)$ is not pruned, observe its response and add it to $\set{D}$
\EndFor
\State Update minority reports for all observed annotations in $\set{D}$ and re-fit model $p(\cdot)$
\EndFor
\State \textbf{Output:} Compute majority votes for all observed annotations in $\set{D}$
\end{algorithmic}
\end{algorithm}

\subsection{Simulation protocol}
\label{sec:pruning_sim}

When using crowdsourced workers, an annotation platform observes assignments sequentially, where workers enter the pool over time and the platform must decide whether to execute a planned repeat using only the history available.
Because workers and tasks are transient, the platform cannot obtain latent variables for the full set of workers and tasks at the time of task allocation.
Instead, the platform estimates minority report risk on a rolling basis as new information arrives. As a result, we implement a rolling window variant of our predict-then-prune policy.

Algorithm \ref{alg:simulation_pruning} summarizes our simulation framework.
Let $\Delta$ denote the length of one recalibration window, and let $\set{I}_t$ denote the assignments scheduled in each window $t$. These assignments have been determined a priori based on which workers were available online in that window. Furthermore, let $p(i,j,k)$ denote the predicted probability that assignment $(i,j,k)$ will yield a minority report, and let $\theta \in (0,1]$ govern pruning aggressiveness, with higher $\theta$ meaning more aggressive pruning. Let $\rho$ denote the minority-fraction threshold.
In our simulation, we advance through the assignment stream of the ECPD dataset over each recalibration window. During the first $\tau$ windows, the platform observes all assigned repeats without pruning and accumulates them in $\set{D}$. Afterwards, within each window $t$, the platform processes assignments in timestamp order using the following rule:
\begin{enumerate}
  \item If worker $j$ or if task instance $(i, k)$ has not yet appeared in $\set{D}_t$, retain the assignment.

  \item If task instance $(i, k)$ has fewer than $3$ earlier retained binary votes or an observed minority fraction above $\rho$, retain the assignment.

  \item Otherwise, score the assignment with $p(i,j,k)$ and prune it when $p(i,j,k) \geq 1-\theta$.
\end{enumerate}
We retain assignments under the first condition because the predictive model cannot estimate minority report risk for unseen workers or tasks.
The second condition protects against questions with little retained evidence or substantial observed disagreement. We set $\rho=0.20$ for all experiments.
Within each window, the platform updates $\set{D}$ and retained-vote counts as responses arrive. At the end of window $t$, it recomputes historical minority reports, and refits the pruning model for the next window.
We refer to this selective rolling-window policy as Dynamic Repeat-Prune Allocation ($\CLSD$).
The initial observation period lasts 36 hours for every policy, and we compare $\Delta\in\{1,2,4,8,12,24,\infty\}$ hours, where $\Delta=\infty$ uses one model fitted after warm-up and no subsequent refitting.
Finally, to predict the risk of minority reports, we use a variant of the mixed-effects logistic regression model from Section \ref{sec:empirical} that considers only worker and crop random effects plus question-type controls. We omit worker activity since pruned assignments will no longer generate the future service path needed to define counterfactual activity.

We evaluate policies through their cost-quality tradeoff. To measure cost, we consider the prune rate, i.e., the share of historical repeats removed before service and therefore the reduction in executed workload relative to the fully executed observed assignment stream.
This evaluation differs slightly from the trade-off assessed in Section \ref{sec:theoretical_analysis}, wherein cost was assessed with respect to the number of annotations used to collect the reference majority-vote data.
Here, we freeze the recorded assignment times and retained responses, while applying the pruning policy.
The resulting savings measure avoidable annotation workload conditioned on the response history.
We measure quality with respect to agreement with the reference majority-vote labels from the completed re-annotation data. This is reasonable, as the reference labels are the operational labels that would be used otherwise to train a model or to make decisions.
Because task labels and correctness are imbalanced, we use $F_1$ as the primary quality metric and accuracy as a secondary measure of reference-label agreement.

We benchmark $\CLSD$ against dynamic Repeat Allocation ($\NOMD$), a nonselective pruning analogue of $\NOM$. For a target prune rate $r$ and a question with $N$ original attempts, we randomly round $rN$ and select that many attempts uniformly without replacement. It therefore varies redundancy without selecting repeats by their risk scores.
Because both $\CLSD$ and $\NOMD$\ act on the same assignment stream, we can compare $\NOMD$ and $\CLSD$ simultaneously on the same Pareto frontier between dataset quality and cost.

\subsection{Selective pruning improves the cost-quality tradeoff}
\label{sec:pruning_tradeoffs}

\begin{figure}[t!]
\centering
\centering
\includegraphics[width=0.7\textwidth]{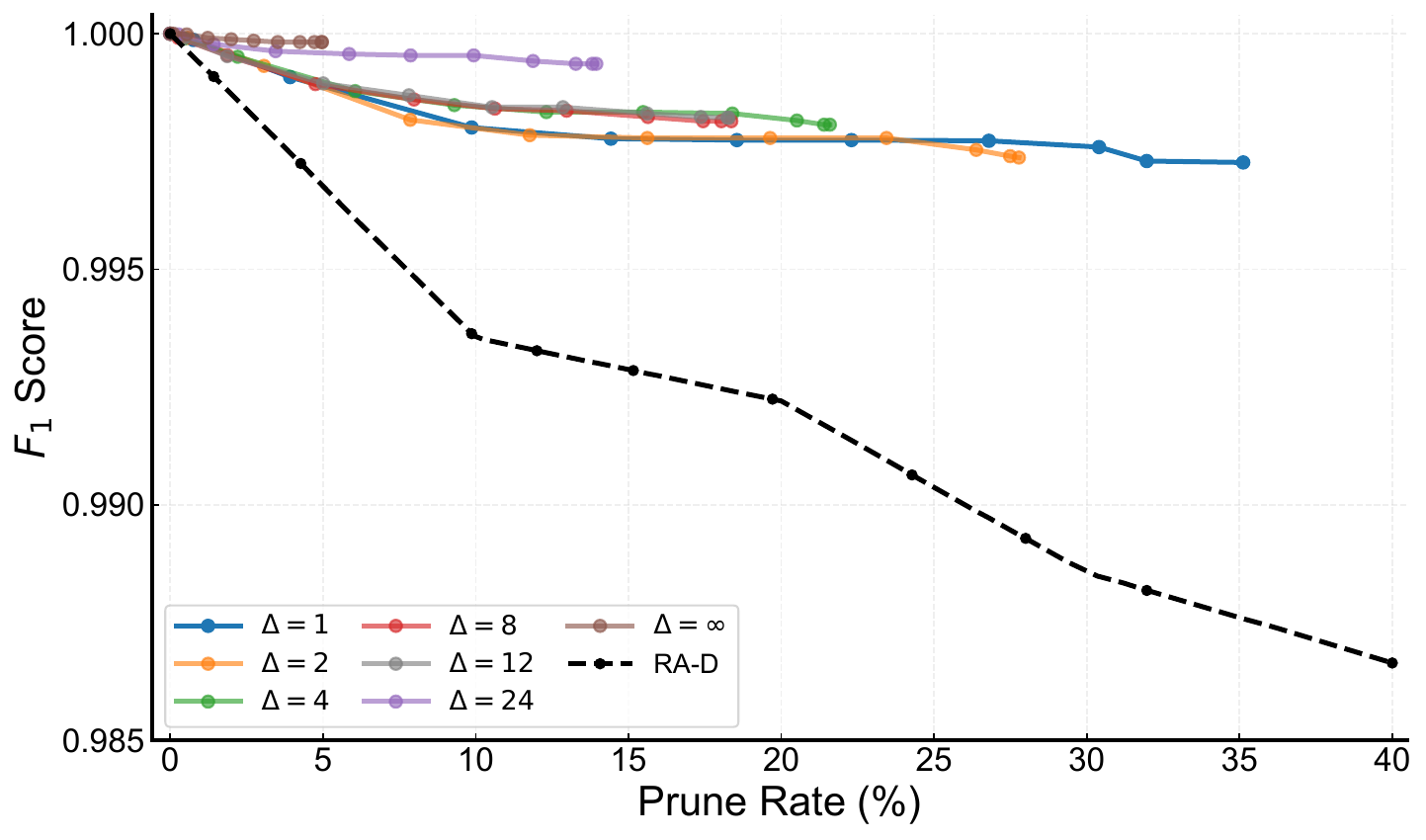}
\caption{Cost-quality frontier for $\CLSD$ under different recalibration periods $\Delta$. The dashed black curve is the mean $\NOMD$ baseline over 100 seeds.}
\label{fig:data_quality_ratio_default_ECPD}
\end{figure}

Figure~\ref{fig:data_quality_ratio_default_ECPD} plots the cost-quality frontier. %
At a 22\% prune rate, the hourly ($\Delta = 1$) $\CLSD$ policy preserves an $F_1$ score of 0.998 and leaves more than 99.9\% of task-level majority votes unchanged. At an average 0.9 seconds per task completion, this saves about 58 hours of effort over the study period.
Moving further along the frontier raises the workload savings while decreasing quality marginally. At its 35\% prune-rate endpoint, hourly $\CLSD$ keeps an $F_1$ score of 0.997 and saves about 92 annotator-hours of effort.
Moreover, the dashed black curve shows the $\NOMD$ benchmark. At the same prune rate, $\CLSD$ consistently preserves historical majority labels better than $\NOMD$. $\CLSD$ always achieves a higher $F_1$ score than the $\NOMD$ curve at the same cost.

Decreasing the time between recalibrations $\Delta$ expands the performance of $\CLSD$. Recalibrating every 24 hours reaches approximately a 13\% prune rate at $F_1=0.999$, compared with a 35\% prune rate reached by hourly recalibration at $F_1=0.997$. The same pattern appears at $\Delta=4$ and $\Delta=8$, which top out at 21\% and 18\% prune rates, respectively.
Frequent recalibration allows the model to incorporate new information about workers and crops, while aggressively pruning more assignments.

\subsection{Labor savings concentrate on a subset of workers}
\label{sec:pruning_worker_incidence}

\begin{figure}[t!]
\centering
\begin{subfigure}[b]{0.49\textwidth}
\centering
\includegraphics[width=\textwidth]{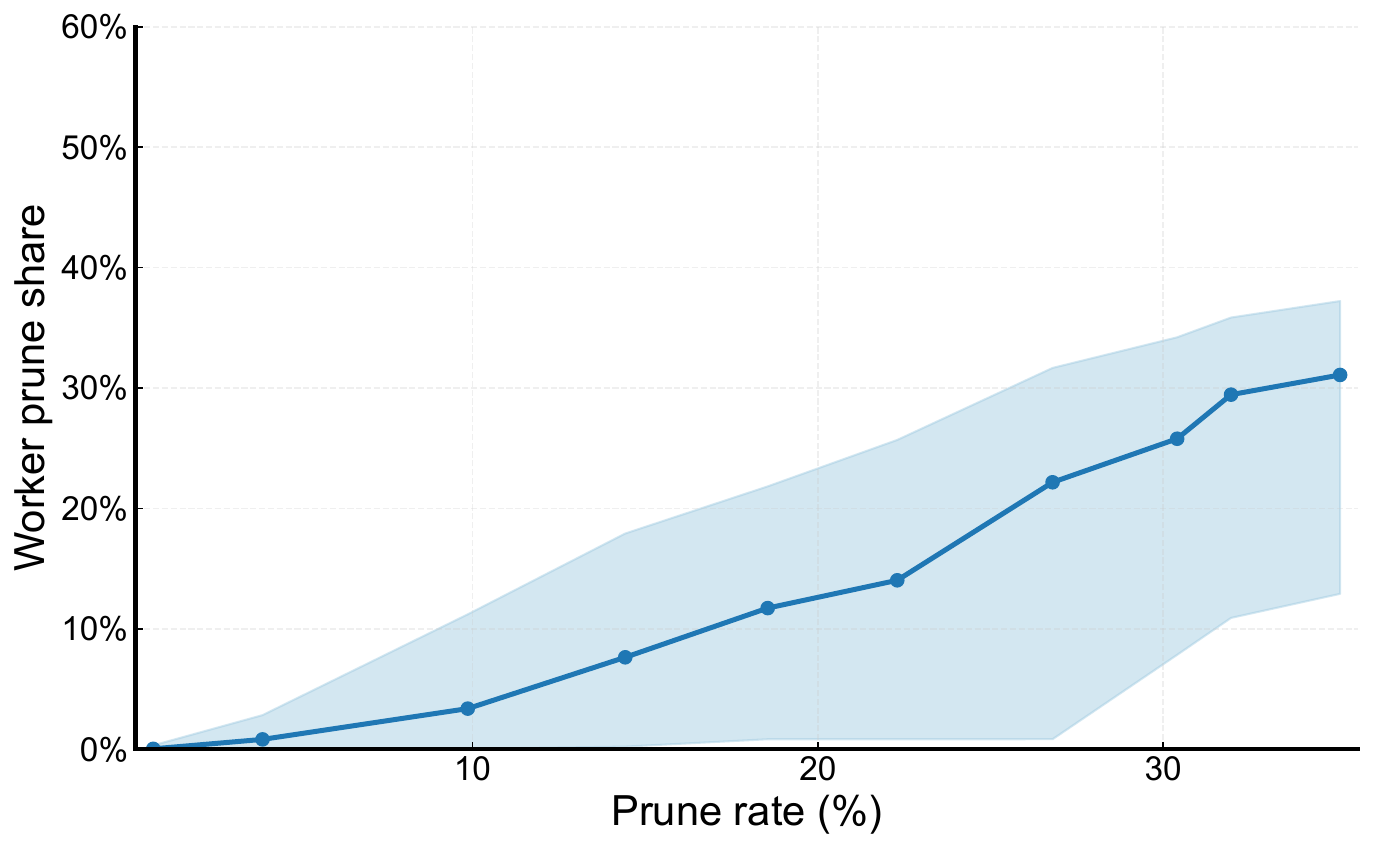}
\caption{Median and interquartile range of worker prune shares by prune rate.}
\label{fig:worker_prune_share_theta}
\end{subfigure}
\hfill
\begin{subfigure}[b]{0.49\textwidth}
\centering
\includegraphics[width=\textwidth]{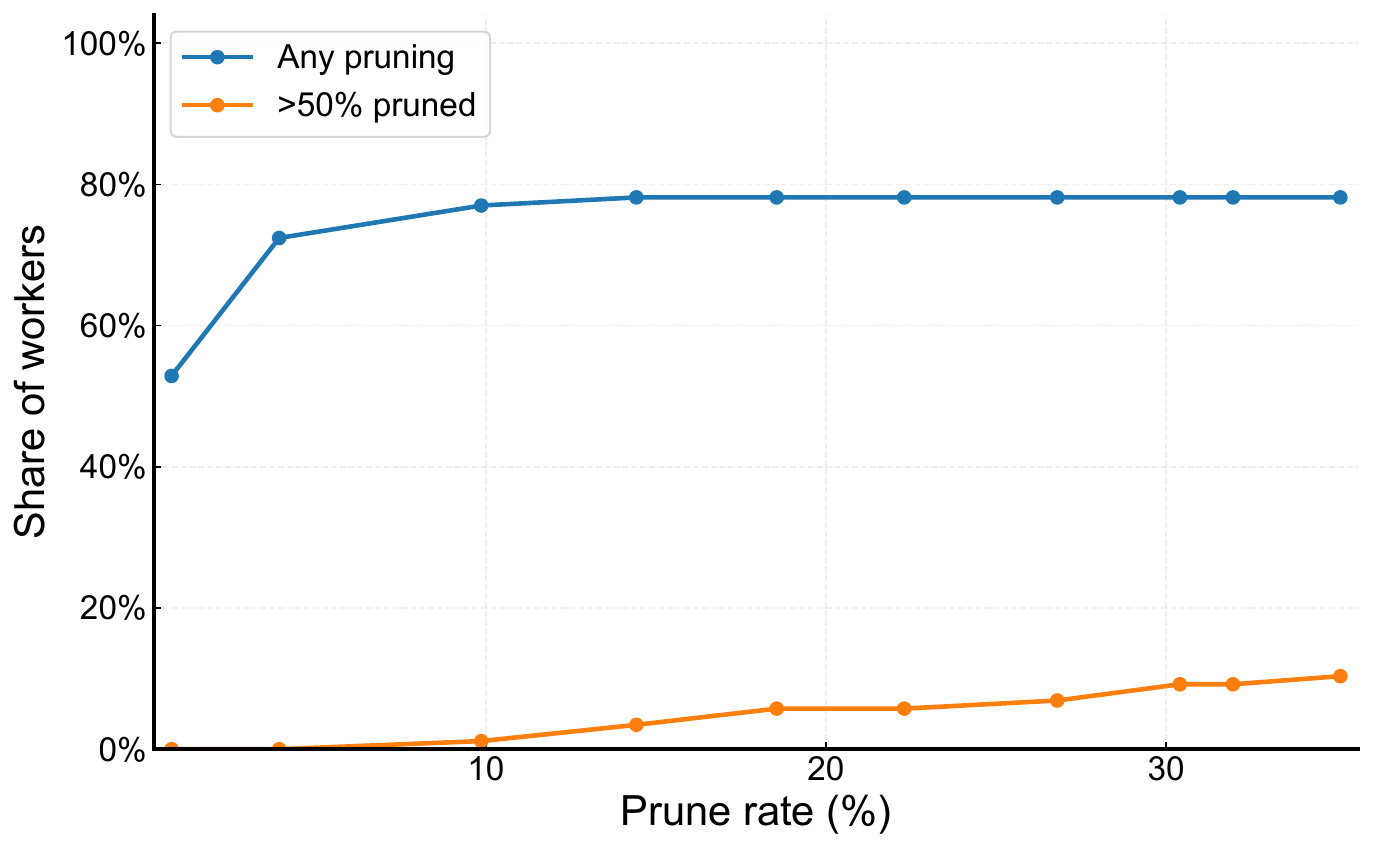}
\caption{Fraction of workers exposed to more than $0\%$ and more than $50\%$ pruning}
\label{fig:worker_exposure_theta}
\end{subfigure}
\caption{Exposure and distribution of prunes to workers for $\CLSD$ with $\Delta = 1$ on the ECPD dataset.}
\label{fig:worker_burden_incidence}
\end{figure}

Platforms using pruning must determine the spread of workload savings across workers, because concentrated cuts can affect access to work, perceived fairness, and the sustainability of the policy. We fix the hourly policy ($\Delta=1$) and trace the worker prune share, defined as the fraction of each worker's post-warmup assignments removed by the policy. Figure~\ref{fig:worker_prune_share_theta} reports the median and interquartile range of worker prune share, while Figure~\ref{fig:worker_exposure_theta} reports the exposed worker share, i.e., the fraction of workers with any pruning, and the fraction of workers losing more than half of their workload. We note that exposure rises quickly and then plateaus. At a 9\% prune rate, 67 of the 87 workers (77\%) are exposed to some pruning; by a 14\% prune rate, this count reaches 68 of 87 workers (78.2\%). The remaining 19 workers are never exposed throughout the sweep.

From 14\% through 35\% pruning, the same 68 workers remain exposed, and higher total prune rates primarily deepen cuts within the same group of workers. The fraction of workers losing more than half of their post-warmup workload rises from 5 of 87 workers (5.7\%) at 22\% pruning to 9 of 87 workers (10.3\%) at 35\%. No worker loses all post-warmup assignments, but realized work is compressed sharply for workers who are already exposed to pruning. Thus, platforms can select a lower prune rate when they place greater weight on limiting concentrated workload reductions, while still obtaining substantial annotation savings.

\subsection{Pruning mostly trims within-session work}
\label{sec:pruning_session_integrity}

Although scattered workload reductions may offer relief from sustained activity, removal of entire work periods may negatively affect crowdsourced worker sentiment. Here, we measure how pruning removes assignments within work sessions.
The post-warmup ECPD dataset contains 373 sessions of continuous activity separated by gaps of at least 10 minutes spread across all workers, with a median duration of 41 minutes.
Figure~\ref{fig:session_loss_theta} reports lost minutes, i.e., the number of removed assignments multiplied by the average time to completion, among affected sessions and the fraction with all post-warm-up assignments pruned.
Across all prune rate settings, the median lost minutes among affected sessions remains below 12 minutes and the fraction of fully pruned sessions remains below 2.2\%.
We conclude that aggressive pruning is feasible through much of the frontier.

\begin{figure}[t!]
\centering
\begin{subfigure}[b]{0.49\textwidth}
\centering
\includegraphics[width=\textwidth]{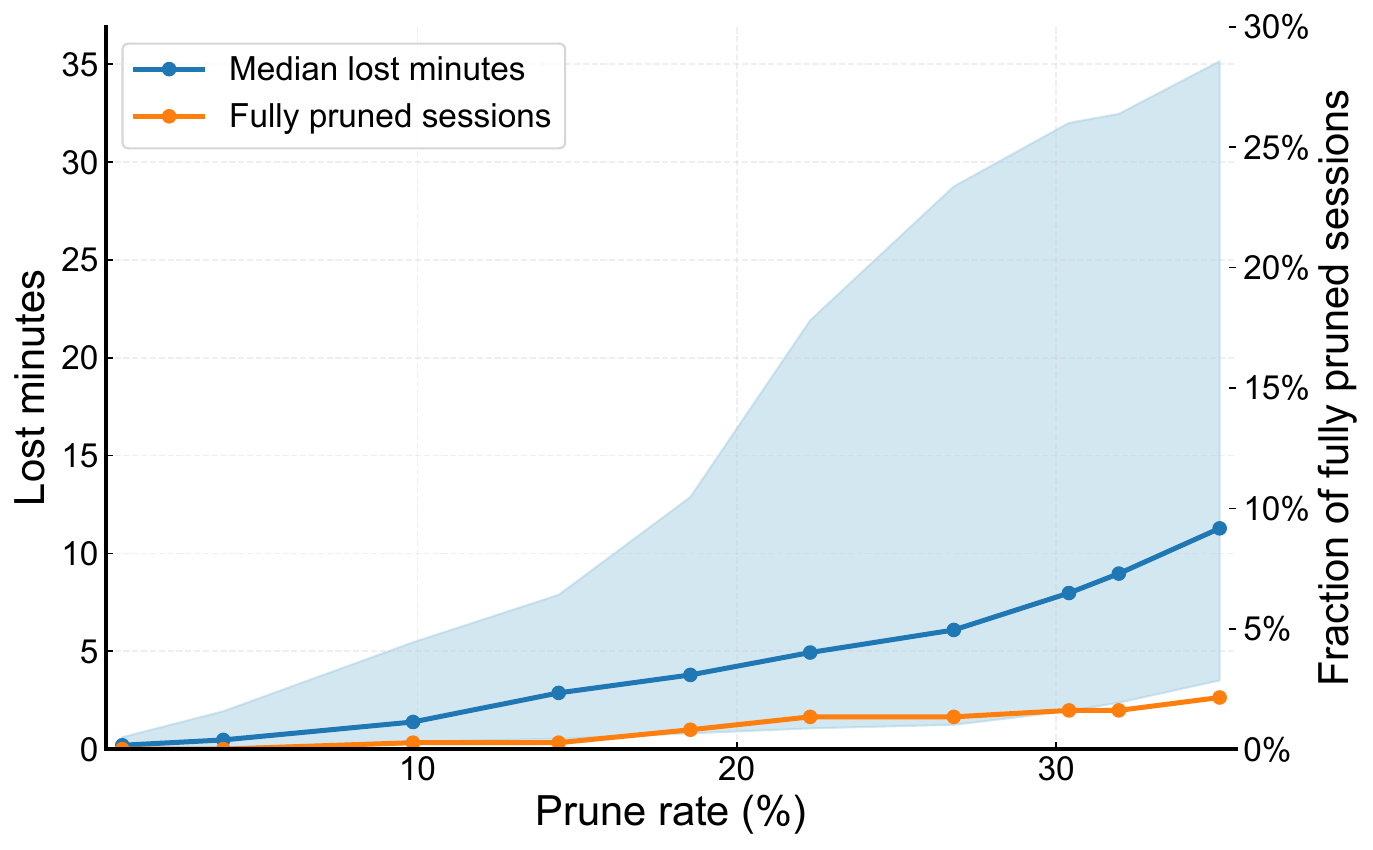}
\caption{Median and interquartile range of lost minutes and fraction of full pruning.}
\label{fig:session_loss_theta}
\end{subfigure}
\hfill
\begin{subfigure}[b]{0.49\textwidth}
\centering
\includegraphics[width=\textwidth]{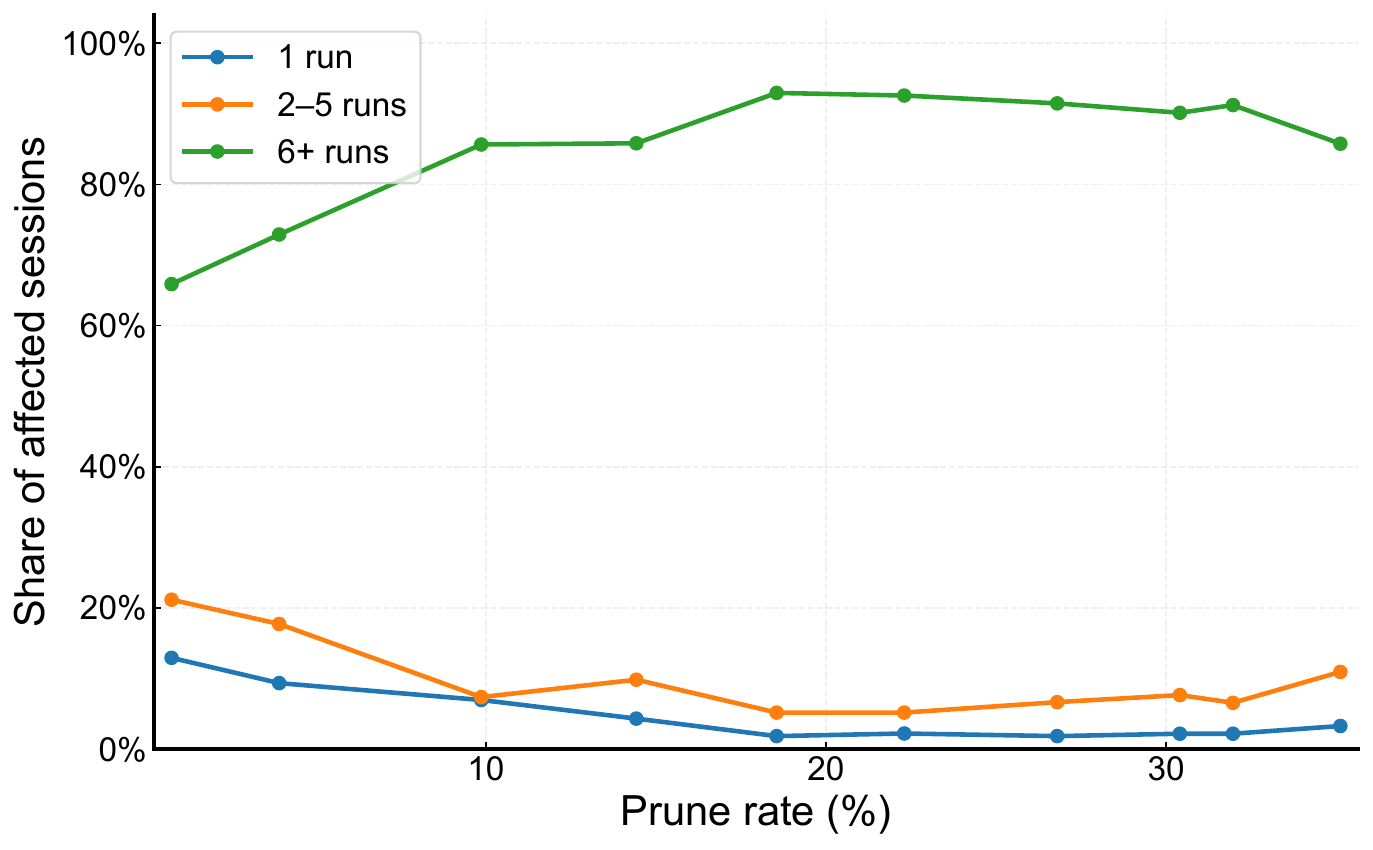}
\caption{Fragmentation within affected sessions into separate runs of pruned assignments.}
\label{fig:session_fragmentation_runs_theta}
\end{subfigure}
\caption{Session-level disruption for $\CLSD$ with $\Delta=1$ on the ECPD dataset.}
\label{fig:session_disruption}
\end{figure}

Finally, we examine how pruned assignments are distributed within baseline sessions. To measure this, we define a run as a maximal contiguous block of pruned assignments within the realized task order of a baseline session. Figure~\ref{fig:session_fragmentation_runs_theta} reports whether removed assignments occur in one run, two to five runs, or six or more runs within the session.
At a 9\% prune rate, only 7\% of affected sessions lose work in a single contiguous pruned run, whereas 85.7\% experience six or more runs. The removal pattern remains  fragmented as prune rates increase. At 22\% pruning, 2.2\% of affected sessions lose work in one run and 92.6\% experience six or more runs. At the 35\% endpoint the corresponding shares are 3.3\% and 85.8\%, respectively. We conclude that protected pruning removes effort mainly through scattered assignments while preserving most baseline sessions.

\section{Conclusion and discussion}
\label{sec:conclusion}

This paper studies repeat allocation and pre-execution task pruning as linked operational decisions inside the data annotation pipeline for machine learning technology. We theoretically show that pruning can lower expected annotation cost at a fixed quality target. Our empirical analysis on re-annotated datasets shows that minority report risk can be estimated from crop ambiguity, worker heterogeneity, and within-shift fatigue. We apply this framework to a rolling selective-pruning policy that removes executed workload while preserving the quality of the downstream dataset.

Our theory and empirical analysis both show that pruning performance depends on classifier quality, baseline task difficulty, and the amount of redundancy already in the system.
Moreover, our rolling window simulations show that substantial workload savings are possible while preserving almost all reference majority-vote labels. Platforms can use the resulting cost-quality frontier to choose pruning levels according to their annotation budgets and label-preservation requirements.
We note that disagreement with the majority does not imply that an annotation has no value. For instance, apparently incorrect responses can still contain useful information, particularly when an aggregation method accounts for systematic worker errors \citep{Wang2017}.
However, we start with the premise that the counterfactual majority-vote label is the operational ground truth, and we focus on preserving this reference while reducing annotation cost.

In our simulations, higher prune rates concentrate workload cuts among the same exposed workers. However, even at 35\% pruning, no worker loses all post-warmup assignments, and only 8 of 373 sessions (2.1\%) are fully removed.
This result diverges from prior work on identifying ``bad actors'' among the crowdsourced annotator pool \citep{karger2013efficient,jagabathula2017identifying}.
Importantly, our framework is rooted in optimism and ensures that workers retain some assignments. This is critical for maintaining worker engagement and future availability.
We stress that annotation platforms should therefore monitor worker exposure and recalibration alongside label quality and labor savings when pruning.

Our empirical results are obtained by two computer vision re-annotation studies conducted through a professional annotation platform. We then evaluate pruning performance by simulating on the historical assignment stream.
Our results demonstrate the value of a simple pruning strategy that combines current risk estimates, retained-vote protections, and periodic recalibration.
This analysis warrants several future directions.
First, estimating minority report risk from fleeting crowdsourced workers requires a rolling window analysis to ensure we can accurately predict for incoming workers and tasks. A natural extension is to explicitly account for the information that an executed assignment provides for subsequent pruning decisions, jointly considering immediate annotation cost and future learning.
A second question is to consider the joint design of repeat allocation and task assignment when information on tasks versus workers arrives asynchronously.
Finally, we may consider the downstream impact of worker response as pruning changes workload and continuity.

{
\normalsize

\normalsize
\bibliographystyle{informs2014} %
\bibliography{references} %

@inproceedings{Kittur2013,
  author    = {Aniket Kittur and Jeffrey V. Nickerson and Michael S. Bernstein and Elizabeth Gerber and Aaron Shaw and John Zimmerman and Matthew Lease and John J. Horton},
  title     = {The Future of Crowd Work},
  booktitle = {Proceedings of the 2013 Conference on Computer Supported Cooperative Work (CSCW)},
  year      = {2013},
  publisher = {ACM},
  address   = {New York, NY, USA}
}

@inproceedings{Yin2018,
  author    = {Ming Yin and Siddharth Suri and Mary L. Gray},
  title     = {Running Out of Time: The Impact and Value of Flexibility in On-Demand Crowdwork},
  booktitle = {Proceedings of the 2018 CHI Conference on Human Factors in Computing Systems (CHI)},
  year      = {2018},
  publisher = {ACM},
  address   = {New York, NY, USA},
}

@inproceedings{ogawa2020gamification,
  title={Gamification platform for collecting task-oriented dialogue data},
  author={Ogawa, Haruna and Nishikawa, Hitoshi and Tokunaga, Takenobu and Yokono, Hikaru},
  booktitle={Proceedings of the Twelfth Language Resources and Evaluation Conference},
  pages={7084--7093},
  year={2020}
}

@article{metz1999proper,
  title={{{{“Proper”}}} binormal {{{ROC}}} curves: theory and maximum-likelihood estimation},
  author={Metz, Charles E and Pan, Xiaochuan},
  journal={{Journal of Mathematical Psychology}},
  volume={43},
  number={1},
  pages={1--33},
  year={1999},
  publisher={Elsevier}
}

@article{Wen2016,
  author    = {Zhong Wen and Lihui Lin},
  title     = {Optimal Fee Structures of Crowdsourcing Platforms},
  journal   = {Decision Sciences},
  year      = {2016},
  volume    = {47},
  number    = {5},
  pages     = {820--850}
}

@article{Manshadi2022,
  author    = {Vahideh H. Manshadi and Scott Rodilitz},
  title     = {Online Policies for Efficient Volunteer Crowdsourcing},
  journal   = {Management Science},
  year      = {2022},
  volume    = {68},
  number    = {9},
  pages     = {6572--6590},
}

@article{Fatehi2022,
  author    = {Soraya Fatehi and Michael R. Wagner},
  title     = {Crowdsourcing Last-Mile Deliveries},
  journal   = {Manufacturing {\&} Service Operations Management},
  year      = {2022},
  volume    = {24},
  number    = {2},
  pages     = {791--809},
}

@article{Wang2017,
  author    = {Jing Wang and Panagiotis G. Ipeirotis and Foster Provost},
  title     = {Cost-Effective Quality Assurance in Crowd Labeling},
  journal   = {Information Systems Research},
  year      = {2017},
  volume    = {28},
  number    = {1},
  pages     = {137--158},
}

@article{Tian2022,
  author    = {Xuhan Tian and Junmin Shi and Xiangtong Qi},
  title     = {Stochastic Sequential Allocations for Creative Crowdsourcing},
  journal   = {Production and Operations Management},
  year      = {2022},
  volume    = {31},
  number    = {2},
  pages     = {697--714},
}

@inproceedings{Rodrigues2018,
  author    = {Filipe Rodrigues and Francisco C. Pereira},
  title     = {Deep Learning from Crowds},
  booktitle = {Proceedings of the 32nd AAAI Conference on Artificial Intelligence (AAAI)},
  year      = {2018}
}

@inproceedings{Tanno2019,
  author    = {Ryutaro Tanno and Ardavan Saeedi and Swami Sankaranarayanan and Daniel C. Alexander and Nathan Silberman},
  title     = {Learning from Noisy Labels by Regularized Estimation of Annotator Confusion},
  booktitle = {Proceedings of the IEEE Conference on Computer Vision and Pattern Recognition (CVPR)},
  year      = {2019},
  pages     = {11236--11245},
}

@article{Goh2023,
  title={{CROWDLAB}: Supervised learning to infer consensus labels and quality scores for data with multiple annotators},
  author={Goh, Hui Wen and Tkachenko, Ulyana and Mueller, Jonas},
  journal={arXiv preprint arXiv:2210.06812},
  year={2022}
}

@article{Horton2011,
  author    = {John J. Horton},
  title     = {The Condition of the Turking Class: Are Online Employers Fair and Honest?},
  journal   = {Economics Letters},
  year      = {2011},
  volume    = {111},
  number    = {1},
  pages     = {10--12}
}

@article{bouajaja2017survey,
  title={A survey on human resource allocation problem and its applications},
  author={Bouajaja, Sana and Dridi, Najoua},
  journal={Operational Research},
  volume={17},
  pages={339--369},
  year={2017},
  publisher={Springer}
}

@article{jin2020cost,
  title={The Cost of Task Switching: Evidence from Emergency Departments},
  author={Jin, Yiwen and Duan, Yige and Ding, Yichuan and Nagarajan, Mahesh and Hunte, Garth},
  journal={Available at SSRN 3756677},
  year={2020}
}

@article{gurvich2020collaboration,
  title={Collaboration, interruptions, and changeover times: workflow model and empirical study of hospitalist charting},
  author={Gurvich, Itai and O’Leary, Kevin J and Wang, Lu and Van Mieghem, Jan A},
  journal={{Manufacturing \& Service Operations Management}},
  volume={22},
  number={4},
  pages={754--774},
  year={2020},
  publisher={INFORMS}
}

@article{narayanan2009matter,
  title={A matter of balance: Specialization, task variety, and individual learning in a software maintenance environment},
  author={Narayanan, Sriram and Balasubramanian, Sridhar and Swaminathan, Jayashankar M},
  journal={Management Science},
  volume={55},
  number={11},
  pages={1861--1876},
  year={2009},
  publisher={INFORMS}
}

@article{staats2012specialization,
  title={Specialization and variety in repetitive tasks: Evidence from a {Japanese} bank},
  author={Staats, Bradley R and Gino, Francesca},
  journal={{Management Science}},
  volume={58},
  number={6},
  pages={1141--1159},
  year={2012},
  publisher={INFORMS}
}

@article{cai2018recover,
  title={Recover overnight? Work interruption and worker productivity},
  author={Cai, Xiqian and Gong, Jie and Lu, Yi and Zhong, Songfa},
  journal={Management Science},
  volume={64},
  number={8},
  pages={3489--3500},
  year={2018},
  publisher={INFORMS}
}

@article{kc2012accumulating,
  title={Accumulating a portfolio of experience: The effect of focal and related experience on surgeon performance},
  author={Kc, Diwas Singh and Staats, Bradley R},
  journal={Manufacturing \& Service Operations Management},
  volume={14},
  number={4},
  pages={618--633},
  year={2012},
  publisher={INFORMS}
}

@techreport{dimensional2019,
     title = {Artificial Intelligence and Machine Learning Projects Are Obstructed by Data Issues: Global Survey of Data Scientists, {AI} Experts and Stakeholders},
     author = {{Dimensional Research}},
     group = {Dimensional Research},
     institution = {Dimensional Research},
     year = {2019},
     month = {05},
}

@article{raykar2010learning,
  author    = {Vikas C. Raykar and Shipeng Yu and Linda H. Zhao and Gerardo Hermosillo Valadez and Charles Florin and Luca Bogoni and Mark Mozer},
  title     = {Learning From Crowds},
  journal   = {Journal of Machine Learning Research},
  year      = {2010},
  volume    = {11},
  pages     = {1297--1322}
}

@article{jagabathula2017identifying,
  title={Identifying unreliable and adversarial workers in crowdsourced labeling tasks},
  author={Jagabathula, Srikanth and Subramanian, Lakshminarayanan and Venkataraman, Ashwin},
  journal={Journal of Machine Learning Research},
  volume={18},
  number={93},
  pages={1--67},
  year={2017}
}

@article{stratman2024decision,
  title={Decision-Aware Predictive Model Selection for Workforce Allocation},
  author={Stratman, Eric G and Boutilier, Justin J and Albert, Laura A},
  journal={arXiv preprint arXiv:2410.07932},
  year={2024}
}

@article{vila2014branch,
  title={A branch-and-bound algorithm for assembly line worker assignment and balancing problems},
  author={Vila, Mariona and Pereira, Jordi},
  journal={Computers \& Operations Research},
  volume={44},
  pages={105--114},
  year={2014},
  publisher={Elsevier}
}

@article{lanzarone2014robust,
  title={Robust nurse-to-patient assignment in home care services to minimize overtimes under continuity of care},
  author={Lanzarone, Ettore and Matta, Andrea},
  journal={Operations Research for Health Care},
  volume={3},
  number={2},
  pages={48--58},
  year={2014},
  publisher={Elsevier}
}

@article{dawid1979maximum,
  author    = {A. Philip Dawid and Allan M. Skene},
  title     = {Maximum Likelihood Estimation of Observer Error-Rates Using the {EM} Algorithm},
  journal   = {Journal of the Royal Statistical Society: Series C (Applied Statistics)},
  year      = {1979},
  volume    = {28},
  number    = {1},
  pages     = {20--28}
}

@article{khetan2016achieving,
  title={Achieving budget-optimality with adaptive schemes in crowdsourcing},
  author={Khetan, Ashish and Oh, Sewoong},
  journal={Advances in Neural Information Processing Systems},
  volume={29},
  year={2016}
}

@inproceedings{karger2013efficient,
  title={Efficient crowdsourcing for multi-class labeling},
  author={Karger, David R and Oh, Sewoong and Shah, Devavrat},
  booktitle={Proceedings of the ACM SIGMETRICS/international conference on Measurement and modeling of computer systems},
  pages={81--92},
  year={2013}
}

@article{karger2014budget,
  title={Budget-optimal task allocation for reliable crowdsourcing systems},
  author={Karger, David R and Oh, Sewoong and Shah, Devavrat},
  journal={Operations Research},
  volume={62},
  number={1},
  pages={1--24},
  year={2014},
  publisher={INFORMS}
}

@article{
  braun2019eurocity,
  author={Braun, Markus and Krebs, Sebastian and Flohr, Fabian B. and Gavrila, Dariu M.},
  journal={IEEE Transactions on Pattern Analysis and Machine Intelligence},
  title={{EuroCity Persons}: A Novel Benchmark for Person Detection in Traffic Scenes},
  year={2019},
  volume={41},
  number={8},
  pages={1844--1861},
  ISSN={0162-8828},
  month={}
}

@article{northcutt2021pervasive,
  title={Pervasive label errors in test sets destabilize machine learning benchmarks},
  author={Northcutt, Curtis G and Athalye, Anish and Mueller, Jonas},
  journal={arXiv preprint arXiv:2103.14749},
  year={2021}
}

@article{kaynar2023estimating,
  title={Estimating effects of incentive contracts in online labor platforms},
  author={Kaynar, Nur and Siddiq, Auyon},
  journal={Management Science},
  volume={69},
  number={4},
  pages={2106--2126},
  year={2023},
  publisher={INFORMS}
}

@inproceedings{horton2010labor,
  title={The labor economics of paid crowdsourcing},
  author={Horton, John J. and Chilton, Lydia B},
  booktitle={Proceedings of the 11th ACM conference on Electronic Commerce},
  pages={209--218},
  year={2010}
}

@InProceedings{Liao_2021_CVPR,
    author    = {Liao, Yuan-Hong and Kar, Amlan and Fidler, Sanja},
    title     = {Towards Good Practices for Efficiently Annotating Large-Scale Image Classification Datasets},
    booktitle = {Proceedings of the IEEE/CVF Conference on Computer Vision and Pattern Recognition (CVPR)},
    month     = {June},
    year      = {2021},
    pages     = {4350-4359}
}

@article{song2022learning,
  title={Learning from noisy labels with deep neural networks: A survey},
  author={Song, Hwanjun and Kim, Minseok and Park, Dongmin and Shin, Yooju and Lee, Jae-Gil},
  journal={{IEEE Transactions on Neural Networks and Learning Systems}},
  volume={34},
  number={11},
  pages={8135--8153},
  year={2023},
  publisher={IEEE}
}

@book{o2011practical,
  title={Practical reliability engineering},
  author={O'Connor, Patrick and Kleyner, Andre},
  year={2011},
  publisher={{John Wiley \& Sons}}
}

@article{palley2023boosting,
  title={Boosting the wisdom of crowds within a single judgment problem: Weighted averaging based on peer predictions},
  author={Palley, Asa B and Satop{\"a}{\"a}, Ville A},
  journal={Management Science},
  volume={69},
  number={9},
  pages={5128--5146},
  year={2023},
  publisher={INFORMS}
}

@article{yin2021learning,
  title={Learning from crowdsourced multi-labeling: A variational {Bayesian} approach},
  author={Yin, Junming and Luo, Jerry and Brown, Susan A},
  journal={Information Systems Research},
  volume={32},
  number={3},
  pages={752--773},
  year={2021},
  publisher={INFORMS}
}

@article{gonzalez2024fairness,
  title={Fairness in crowdwork: Making the human {AI} supply chain more humane},
  author={Gonzalez-Cabello, Martin and Siddiq, Auyon and Corbett, Charles J and Hu, Catherine},
  journal={Business Horizons},
  year={2024},
  publisher={Elsevier}
}

@book{berg2018digital,
  title={Digital labour platforms and the future of work: Towards decent work in the online world},
  author={Berg, Janine and Furrer, Marianne and Harmon, Ellie and Rani, Uma and Silberman, M Six},
  year={2018},
  publisher={{International Labour Organization}}
}

@inproceedings{alibeigi2023zenseact,
  title={{Zenseact Open Dataset}: A large-scale and diverse multimodal dataset for autonomous driving},
  author={Alibeigi, Mina and Ljungbergh, William and Tonderski, Adam and Hess, Georg and Lilja, Adam and Lindstr{\"o}m, Carl and Motorniuk, Daria and Fu, Junsheng and Widahl, Jenny and Petersson, Christoffer},
  booktitle={Proceedings of the IEEE/CVF International Conference on Computer Vision},
  pages={20178--20188},
  year={2023}
}

@inproceedings{deng2009imagenet,
  title={{ImageNet}: A large-scale hierarchical image database},
  author={Deng, Jia and Dong, Wei and Socher, Richard and Li, Li-Jia and Li, Kai and Fei-Fei, Li},
  booktitle={2009 IEEE Conference on Computer Vision and Pattern Recognition},
  pages={248--255},
  year={2009},
  organization={{IEEE}}
}

@article{adelman2024dynamic,
  title={Dynamic Optimization of Workforce Talent},
  author={Adelman, Daniel and Mersereau, Adam and Pakiman, Parshan},
  journal={Available at SSRN 4964890},
  year={2024}
}

@article{hartig2016dharma,
  title={ {DHARMa}: residual diagnostics for hierarchical (multi-level/mixed) regression models},
  author={Hartig, Florian},
  journal={CRAN: Contributed Packages},
  year={2016},
  publisher={The R Foundation}
}

@inproceedings{durasov2026uncertainty,
  title={Uncertainty-Guided Data Curation for {3D} Object Detection},
  author={Durasov, Nikita and Mahmood, Rafid and Choi, Jiwoong and Law, Marc T and Lucas, James and Fua, Pascal and Alvarez, Jose M},
  booktitle={Proceedings of the IEEE/CVF Conference on Computer Vision and Pattern Recognition},
  pages={4332--4342},
  year={2026}
}

@inproceedings{
    liao2024transferring,
    title={Transferring Labels to Solve Annotation Mismatches Across Object Detection Datasets},
    author={Yuan-Hong Liao and David Acuna and Rafid Mahmood and James Lucas and Viraj Uday Prabhu and Sanja Fidler},
    booktitle={The Twelfth International Conference on Learning Representations},
    year={2024},
}

@article{farajollahzadeh2026raising,
  author = {Farajollahzadeh, Setareh and Glew, Rob},
  title = {Raising the Bar: Motivating Contributors in {AI}-Assisted Crowdsourcing},
  year = {2026},
  journal = {Available at SSRN 5627813},
}

}

\ECSwitch
\renewcommand{\theHequation}{EC.\arabic{equation}}
\renewcommand{\theHfigure}{EC.\arabic{figure}}
\renewcommand{\theHtable}{EC.\arabic{table}}
\renewcommand{\theHsection}{EC.\arabic{section}}
\renewcommand{\theHsubsection}{\theHsection.\arabic{subsection}}
\renewcommand{\theHsubsubsection}{\theHsubsection.\arabic{subsubsection}}
\ECTYPE{Electronic companion to\enspace}
\ECHead{Electronic Companion to Minority Reports: Balancing Cost and Accuracy in Ground Truth Data Annotation}

\section{Proofs}
\label{sec:app_proofs}

This section is organized in three parts.
First, Section~\ref{sec:app_supporting_rates} collects the supporting rate and asymptotic results used in Section~\ref{sec:theoretical_analysis}.
Second, Section~\ref{sec:app_proofs_monotonicity} provides proofs for the monotonicity results that govern how repeat allocations affect error.
Third, Section~\ref{sec:app_proofs_asymptotic} provides proofs for the asymptotic-rate results and the strict-improvement theorem.

\subsection{Supporting Results for Theorem \ref{thm:clf_asymptotically_better}}
\label{sec:app_supporting_rates}

In this section, we develop a series of asymptotic supporting results supplementary to Section \ref{sec:theoretical_analysis}.
These results characterize error exponents and repeat-complexity scaling for $\NOM$ and $\CLS$, and support Theorem~\ref{thm:clf_asymptotically_better}.
We first establish the large-deviation rate for the error probability of repeat-only aggregation and pruning-based aggregation.
We then map these error exponents to the repeat-complexity and executed-cost objects defined in Section~\ref{sec:ra_rpa_model}, which are used to compare the cost-quality tradeoffs of $\NOM$ and $\CLS$.

We first restate the information rate that determines how fast the error in problem $\NOM$ shrinks with the repeat count. Let
$I_\mathrm{err}(\phi) := D_{KL}(1/2\|\phi) = -\log(4\phi\bar{\phi})/2$ be the rate function associated with the random variable $Y \sim \bernoullidist(\phi)$. Then, $I_\mathrm{err}(\phi)$ governs the decay of the majority-vote error probability in $n$.

\begin{lem}\label{lem:mv_cramer}
  Under Assumption \ref{ass:phi_half},
  \begin{align*}
    \lim_{n\to \infty} \frac{1}{n} \log \Perr(n, \phi) = - \Ierr(\phi)
  \end{align*}
\end{lem}

Lemma~\ref{lem:mv_cramer} shows that majority-vote error shrinks exponentially in the number of repeats. As $n$ grows, an incorrect majority becomes rapidly less likely at the rate $\Ierr(\phi)$.

We next derive a pruning analogue of Lemma~\ref{lem:mv_cramer} for $\Perrcls(n,\theta,\phi)$.
\begin{lem}\label{lem:perrcls_chernoff}
  Under Assumptions \ref{ass:phi_half} and \ref{ass:clf}, $\Perrcls(n,\theta,\phi) \leq \exp(-n \Ierrcls(\theta,\phi))$ where
  $\Ierrcls(\theta,\phi) := - \log(r(\theta,\phi) + \bar{r}(\theta,\phi)\exp(-\Ierr(\pi(\theta,\phi))))$.
\end{lem}

As with Lemma~\ref{lem:mv_cramer}, the post-pruning error probability $\Perrcls(n,\theta,\phi)$ is also exponentially small in the number of allocated repeats $n$, but with a new exponent $\Ierrcls(\theta,\phi)$ that depends on both the pruning rate $r(\theta,\phi)$ and the retained difficulty $\pi(\theta,\phi)$. Note that $\Ierrcls(\theta,\phi)$ has a mixture interpretation. That is, with probability $r$ the repeat is pruned and gives no vote, while with probability $1-r$ the repeat is kept but drawn from a population with lower rate $\pi$.
Consequently, pruning can improve reliability by reducing the effective error rate from $\phi$ to $\pi$, but it can also weaken the exponential rate by shrinking the number of completed votes through a larger $r$.
Thus, $\theta$ must be tuned so that the gains from lowering $\pi$ outweigh the losses from removing too many repeats.

We first note the following.

\begin{cor}\label{cor:ierr_properties}
    $\Ierrcls(\theta,\phi) \leq \bar{r}(\theta,\phi) \Ierr(\pi(\theta,\phi))$.
\end{cor}

Corollary~\ref{cor:ierr_properties} upper bounds the pruning exponent from Lemma~\ref{lem:perrcls_chernoff}. %
We now map the exponent results to the decision objects in Section~\ref{sec:ra_rpa_model}. Theorem~\ref{thm:asymptotic_nominal} gives a scaling of the optimal solution to $\NOM$.

\begin{thm}\label{thm:asymptotic_nominal}
  For any $\phi \in (0, 1/2)$, the optimal solution to $\NOM$ satisfies
  \begin{align*}
    \lim_{\epsilon \to 0} \frac{n^\dagger(\epsilon, \phi)}{\log(1/\epsilon)} = \frac{1}{\Ierr(\phi)}
  \end{align*}
\end{thm}

Theorem~\ref{thm:asymptotic_nominal} shows that achieving a smaller target error $\epsilon$ requires repeats that scale only logarithmically in $1/\epsilon$, with leading constant $1/\Ierr(\phi)$. Thus, easier tasks (smaller $\phi$) need fewer repeats to meet the same quality target.

We compare pruning policies against this nominal scaling and ask when problem $\CLS$ can achieve the same target $\epsilon$ with lower expected executed repeats.

\begin{thm}\label{thm:asymptotic_clf}
  For any $\phi \in (0, 1/2)$, the optimal cost of problem $\CLS$ satisfies
  \begin{align*}
    \limsup_{\epsilon\to 0} \min_{\theta} \frac{n^*(\epsilon,\theta,\phi)(1 - r(\theta, \phi))}{\log(1/\epsilon)} \leq \inf_\theta \frac{1 - r(\theta, \phi)}{\Ierrcls(\theta, \phi)}
  \end{align*}
\end{thm}

Theorem~\ref{thm:asymptotic_clf} characterizes the best achievable asymptotic cost under pruning.
Comparing Theorem~\ref{thm:asymptotic_clf} with Theorem~\ref{thm:asymptotic_nominal} yields the sufficient condition used in Theorem~\ref{thm:clf_asymptotically_better} in the main text.

\subsection{Proofs for Monotonicity Results}
\label{sec:app_proofs_monotonicity}
This subsection proves the finite-repeat monotonicity statements used in Section~\ref{sec:theoretical_analysis}.

\begin{proof}[Proof of Lemma \ref{lem:perr_decreasing}]
Let $S_n := \sum_{j=1}^n Y_j$ and let $B := Y_{n+1} + Y_{n+2} \sim \mathrm{Binomial}(2, \phi)$. The case $\phi=0$ is immediate. For $0<\phi<1$, we want to show that, if and only if $\phi \leq (\lfloor n/2 \rfloor + 1)/(n+1)$, we have
\begin{align*}
\Perr(n+2, \phi) \leq \Perr(n, \phi) ~\Longleftrightarrow~ \Pr\left( S_n + B \geq \left\lceil \frac{n}{2} \right\rceil + 1\right) \leq \Pr\left( S_n \geq \left\lceil \frac{n}{2} \right\rceil \right)
\end{align*}

Note that if $S_n \geq \lceil n/2 \rceil + 1$, then the data is incorrectly labeled regardless of whether there are $n$ or $n + 2$ repeats. Similarly, if $S_n \leq \lceil n/2 \rceil - 2$, then the data is correctly labeled regardless of whether there are $n$ or $n+2$ repeats. On the other hand, if $S_n = \lceil n/2 \rceil - 1$, then the additional repeats incur an error if and only if $B = 2$. Finally, if $S_n = \lceil n/2 \rceil$, then the additional repeats can resolve the error if and only if $B = 0$.
Combining these cases, we obtain
\begin{align}
\Perr(n+2, \phi) - \Perr(n, \phi) &= \Pr\left(S_n = \left\lceil \frac{n}{2} \right\rceil - 1 \right) \Pr(B=2) - \Pr\left(S_n = \left\lceil \frac{n}{2} \right\rceil \right) \Pr(B=0) \label{eq:proof_lem_perr_decreasing0} \\
&= {n \choose \lceil n/2 \rceil - 1} \phi^{\lceil n/2 \rceil + 1} \bar{\phi}^{n-\lceil n/2 \rceil + 1} - {n \choose \lceil n/2 \rceil} \phi^{\lceil n/2 \rceil} \bar{\phi}^{n-\lceil n/2 \rceil + 2}  \label{eq:proof_lem_perr_decreasing1} \\
&= {n \choose \lceil n/2 \rceil}  \phi^{\lceil n/2 \rceil} \bar{\phi}^{n - \lceil n/2 \rceil + 1}
\left(
\frac{\lceil n/2 \rceil}{n - \lceil n/2 \rceil + 1} \phi - \bar{\phi}
\right) \label{eq:proof_lem_perr_decreasing2}
\\
&\leq 0
\end{align}
Above, \eqref{eq:proof_lem_perr_decreasing0} is obtained by equating the difference in probabilities to the probabilities of the non-intersecting events and \eqref{eq:proof_lem_perr_decreasing1} follows from the Binomial distribution definitions of each term.
Then, \eqref{eq:proof_lem_perr_decreasing2} factors the terms, noting that ${n \choose \lceil n/2 \rceil -1} = {n \choose \lceil n/2 \rceil} \lceil n/2 \rceil/(n - \lceil n/2 \rceil + 1)$. Finally, we observe that the last inequality holds if and only if the terms inside parentheses are non-positive, i.e.,
\begin{align*}
\frac{\lceil n/2 \rceil}{n - \lceil n/2 \rceil + 1} \phi - \bar{\phi} \leq 0 ~~~\Longleftrightarrow~~~ \phi \leq \frac{n - \lceil n/2 \rceil + 1}{n +1} = \frac{\lfloor n/2 \rfloor + 1}{n +1}
\end{align*}
\end{proof}

\begin{proof}[Proof of Lemma \ref{lem:pi_le_phi}]
Because $q_T(\theta) \geq q_F(\theta)$, we have $\bar{q}_F(\theta) \geq \bar{q}_T(\theta)$. Furthermore, $\phi \bar{q}_T(\theta) + \bar{\phi}\bar{q}_F(\theta) \geq \bar{q}_T(\theta)$, because the left-hand-side is a convex combination of $\bar{q}_T(\theta)$ and $\bar{q}_F(\theta)$. Then,
\begin{align*}
\pi(\theta,\phi) = \phi \frac{\bar{q}_T(\theta)}{\phi \bar{q}_T(\theta) + \bar{\phi}\bar{q}_F(\theta)} \leq \phi
\end{align*}
\end{proof}

To prove Theorem \ref{thm:monotone_n}, we next state an auxiliary lemma used only in that proof.

\begin{lem}\label{lem:perr_decreasing_n1}
Suppose that $\phi < 1/2$. Then,
\begin{align*}
\left|\Perr(n+1, \phi) - \Perr(n, \phi) \right| \leq - \frac{\Perr(n+2, \phi) - \Perr(n, \phi)}{1 - 2 \phi}
\end{align*}
\end{lem}
\begin{proof}[Proof of Lemma \ref{lem:perr_decreasing_n1}]
We break the proof into two cases for when $n$ is odd and when $n$ is even.

\textbf{If $n = 2m+1$ for $m \in \field{N}$:} Recall from \eqref{eq:proof_lem_perr_decreasing2} in the proof of Lemma \ref{lem:perr_decreasing} that
\begin{align}
\Perr(n+2, \phi) - \Perr(n, \phi) &= {n \choose m}  \phi^{m+1} \bar{\phi}^{m+1} \left( 2\phi - 1 \right) \label{eq:proof_lem_perr_decreasing_n1_1}\\
&= \Pr(S_n = m) (2\phi - 1) \phi. \label{eq:proof_lem_perr_decreasing_n1_2}
\end{align}
Furthermore, let $Y_{n+1} \sim \bernoullidist(\phi)$ and note that
\begin{align}
\Perr(n+1, \phi) - \Perr(n, \phi) &= \Pr( S_n + Y_{n+1} \geq m+1 ) - \Pr( S_n \geq m+1 )  \label{eq:proof_lem_perr_decreasing_n1_3} \\
&= \Pr( S_n = m ) \phi \label{eq:proof_lem_perr_decreasing_n1_4}
\end{align}
Combining \eqref{eq:proof_lem_perr_decreasing_n1_2} and \eqref{eq:proof_lem_perr_decreasing_n1_4} yields
\begin{align*}
-\frac{\Perr(n+2, \phi) - \Perr(n, \phi) }{1 - 2 \phi} = \Perr(n+1, \phi) - \Perr(n, \phi),
\end{align*}
which satisfies the bound.

\textbf{If $n=2m$ for $m \in \field{N}$:} Let $Y_{n+1}, Y_{n+2} \sim \bernoullidist(\phi)$ and let $B = Y_{n+1} + Y_{n+2}$. Then, note that
\begin{align}
\Perr(n+1, \phi) - \Perr(n, \phi) &= \Pr( S_n + Y_{n+1} \geq m+1 ) - \Pr( S_n \geq m )  \label{eq:proof_lem_perr_decreasing_n1_5} \\
&= - \Pr(S_n = m)\bar{\phi} \leq 0 \label{eq:proof_lem_perr_decreasing_n1_6}
\end{align}
Furthermore from \eqref{eq:proof_lem_perr_decreasing0} in the proof of Lemma \ref{lem:perr_decreasing}, we have
\begin{align}
\Perr(n+2, \phi) - \Perr(n, \phi) &= \Pr(S_n = m-1) \Pr(B=2) - \Pr(S_n = m) \Pr( B = 0) \label{eq:proof_lem_perr_decreasing_n1_7} \\
&= \Pr(S_n = m-1) \phi^2 - \Pr(S_n = m) \bar{\phi}^2 \label{eq:proof_lem_perr_decreasing_n1_8}
\end{align}
We add $\Pr(S_n=m)(1-2\phi)\bar{\phi}$ to both sides of this equality to yield
\begin{align}
\mathrm{LHS} \eqref{eq:proof_lem_perr_decreasing_n1_7} + \Pr(S_n=m)(1-2\phi)\bar{\phi} &= \mathrm{RHS}  \eqref{eq:proof_lem_perr_decreasing_n1_8} + \Pr(S_n=m)(1-2\phi)\bar{\phi}  \label{eq:proof_lem_perr_decreasing_n1_9} \\
&= \Pr(S_n = m-1) \phi^2  - \Pr(S_n=m)\bar{\phi} \phi \label{eq:proof_lem_perr_decreasing_n1_10} \\
&= \phi \left( \Pr(S_n = m-1) \phi - \Pr(S_n=m)\bar{\phi} \right) \label{eq:proof_lem_perr_decreasing_n1_11}
\end{align}
We now use the identity
\begin{align*}
\frac{\Pr(S_n = m)}{\Pr(S_n = m-1)} = \frac{  {2m \choose m} \phi^m \bar{\phi}^{n-m} }{ {2m \choose m-1} \phi^{m-1} \bar{\phi}^{n-m+1}  } = \frac{(m + 1)\phi}{m\bar{\phi}} > \frac{\phi}{\bar{\phi}}
\end{align*}
to obtain $\mathrm{RHS} \eqref{eq:proof_lem_perr_decreasing_n1_11} \leq 0$. Re-arranging $\mathrm{LHS} \eqref{eq:proof_lem_perr_decreasing_n1_9}$ yields
\begin{align*}
\Perr(n+2, \phi) - \Perr(n, \phi) \leq -\Pr(S_n = m)(1-2\phi)\bar{\phi}=(1-2\phi)\left( \Perr(n+1, \phi) - \Perr(n, \phi) \right)
\end{align*}
which completes the bound.
\end{proof}

We now prove Theorem \ref{thm:monotone_n}.
\begin{proof}[Proof of Theorem \ref{thm:monotone_n}]
We first handle the boundary case $\pi=1/2$. By Lemma \ref{lem:pi_le_phi} and Assumption \ref{ass:phi_half}, $\pi=1/2$ implies $\phi=1/2$. Substituting $\phi=\pi=1/2$ into the definition of $\pi$ yields $\bar q_T=\bar q_F$, and hence $q_T=q_F$. Letting $q:=q_T=q_F$, we have $r=q$, while the sufficient condition in Theorem \ref{thm:monotone_n} reduces to $q\leq q/3$. Therefore, $q=r=0$, so no repeats are pruned and $\Perrcls(n,\theta,\phi)=\Perr(n,1/2)$. Lemma \ref{lem:perr_decreasing} then gives $\Perrcls(n+2,\theta,\phi)\leq\Perrcls(n,\theta,\phi)$.

We now consider the general case of $\pi < 1/2$.
Let $Y_{n+1}, Y_{n+2} \sim \bernoullidist(1-r)$ indicate whether two additional repeats survive pruning. Let $B := Y_{n+1} + Y_{n+2}$ be their corresponding Binomial variable. Then,
\begin{align}
&\Perrcls(n + 2, \theta, \phi)   \label{eq:proof_thm_monotone_n1} \\
&\; =\EX_{Z_{n}} \Big[ \EX_{B} \left[ \Perr(Z_n + B, \pi) \;|\; Z_{n} \right] \Big] \label{eq:proof_thm_monotone_n2} \\
&\; =\EX_{Z_n} \Big[\Perr(Z_n, \pi) r^2 + 2 \Perr(Z_n+1, \pi)r \bar{r} + \Perr(Z_n + 2, \pi) \bar{r}^2 \Big] \label{eq:proof_thm_monotone_n3} \\
&\; =\EX_{Z_n} \Big[ \Perr(Z_n, \pi)  + 2 r \bar{r} ( \Perr(Z_n+1, \pi) - \Perr(Z_n, \pi) ) + \bar{r}^2 ( \Perr(Z_n+2, \pi) - \Perr(Z_n, \pi) ) \Big] \label{eq:proof_thm_monotone_n4} \\
&\; \leq \Perrcls(n, \theta, \phi) + \EX_{Z_n} \left[ (\Perr(Z_n+2, \pi) - \Perr(Z_n, \pi)) \left( \bar{r}^2 - \frac{2 r \bar{r}}{1 - 2\pi} \right)  \right] \label{eq:proof_thm_monotone_n5}
\end{align}
Above, \eqref{eq:proof_thm_monotone_n2} follows from the law of conditional expectation and \eqref{eq:proof_thm_monotone_n3} follows from expanding out the expectation over $B$. We then obtain \eqref{eq:proof_thm_monotone_n4} by adding and subtracting $\Perr(Z_n, \pi) (2r \bar{r} + \bar{r}^2)$. Finally, we obtain \eqref{eq:proof_thm_monotone_n5} by applying Lemma \ref{lem:perr_decreasing_n1} to substitute $\Perr(Z_n+1,\pi) - \Perr(Z_n,\pi)$, followed by grouping the remaining terms.

Finally, we note that a sufficient condition for $\Perrcls(n + 2, \theta, \phi) \leq \Perrcls(n, \theta, \phi)$ to hold is if the expectation term is negative. First, from Lemma \ref{lem:perr_decreasing}, $\Perr(Z_n+2,\pi) - \Perr(Z_n,\pi) \leq 0$ is guaranteed because $\pi \leq \phi \leq 1/2$. Consequently, the sufficient condition is equivalent to
\begin{align}
\bar{r}^2 - \frac{2 r \bar{r}}{1 - 2\pi} \geq 0 & \quad \Longleftrightarrow \quad \bar{r} (1 - 2 \pi) \geq 2r \label{eq:proof_thm_monotone_n6}  \\
& \quad \Longleftrightarrow \bar{r} - 2 \phi \bar{q}_T \geq 2 r  \label{eq:proof_thm_monotone_n7} \\
& \quad \Longleftrightarrow r \leq \frac{1 - 2 \phi \bar{q}_T}{3} \label{eq:proof_thm_monotone_n8}
\end{align}
Above, \eqref{eq:proof_thm_monotone_n6} is obtained by dividing $\bar{r}$ from both sides and moving the denominator. Then, \eqref{eq:proof_thm_monotone_n7} is obtained by noting $1 - 2\pi = (\bar{r} - 2\phi \bar{q}_T)/\bar{r}$. Finally, \eqref{eq:proof_thm_monotone_n8} follows from rearranging the terms.
\end{proof}

\subsection{Proofs for Asymptotic-Rate Results}
\label{sec:app_proofs_asymptotic}
This subsection proves the supporting rate lemmas and asymptotic theorems used to establish high-accuracy cost comparisons.

\begin{proof}[Proof of Lemma \ref{lem:mv_cramer}]
The lemma is an application of Cram\'er's Theorem, as follows:
\begin{align}
\lim_{n\to \infty} \frac{1}{n} \log \Perr(n, \phi) &= \lim_{n\to\infty} \frac{1}{n} \log  \Pr\left( S_n \geq \left\lceil \frac{n}{2} \right\rceil \right) \label{eq:proof_mv_cramer1} \\
&= \lim_{n\to \infty} \frac{1}{n} \log \Pr\left(S_n \geq \frac{n}{2}\right) \label{eq:proof_mv_cramer2} \\
&= - \sup_{\lambda \in \field{R}} \left\{ \frac{\lambda}{2} - \log \EX\left[ e^{\lambda Y} \right] \right\}  \label{eq:proof_mv_cramer3} \\
&= - \left( \frac{1}{2} \log\frac{\bar{\phi}}{\phi} - \log2\bar{\phi} \right) \label{eq:proof_mv_cramer4} \\
&= - \frac{1}{2} \log\frac{1}{4 \phi \bar{\phi}}  \label{eq:proof_mv_cramer5} \\
&= - \Ierr(\phi) \label{eq:proof_mv_cramer6}
\end{align}
Above, \eqref{eq:proof_mv_cramer1} follows by definition of $\Perr(n, \phi)$ and \eqref{eq:proof_mv_cramer2} follows by noting the remainder disappears as $n$ increases. We apply Cram\'er's Theorem in \eqref{eq:proof_mv_cramer3}. We then note that the supremum is attained at $\lambda^* = \log(\bar{\phi}/\phi)$ to obtain \eqref{eq:proof_mv_cramer4}. Finally, \eqref{eq:proof_mv_cramer5} and \eqref{eq:proof_mv_cramer6} follow by simplifying the terms.
\end{proof}

\begin{proof}[Proof of Lemma \ref{lem:perrcls_chernoff}]
Recall that $Z_n \sim \binomialdist(n, \bar{r})$ represents the number of unpruned repeats. Without loss of generality, assume that the last $n - Z_n$ repeats are pruned and consider the sequence of random variables $\Ind(Y_j=1|\hat{Y}_j=0) \sim \bernoullidist(\pi(\theta,\phi))$ for $j\in\{1,\cdots,Z_n\}$. We now bound $\Perrcls(n,\theta,\phi)$ in the following steps:
\begin{align}
\Perrcls(n,\theta,\phi) &= \EX_{Z_n} \left[ \Perr(Z_n, \pi(\theta,\phi)) \right] \label{eq:proof_lem_perrcls_chernoff1} \\
& \leq \EX_{Z_n} \left[ \Pr\left( \sum_{j=1}^{Z_n}\Ind(Y_j=1|\hat{Y}_j=0)  \geq \frac{Z_n}{2} \right) \right] \label{eq:proof_lem_perrcls_chernoff2} \\
& \leq \EX_{Z_n} \left[ e^{ -Z_n \Ierr(\pi(\theta,\phi)) } \right] \label{eq:proof_lem_perrcls_chernoff3} \\
& = \left( r(\theta,\phi) + \bar{r}(\theta,\phi) e^{-\Ierr(\pi)} \right)^n \label{eq:proof_lem_perrcls_chernoff4} \\
& = e^{ -n \Ierrcls(\theta,\phi) } \label{eq:proof_lem_perrcls_chernoff5}
\end{align}
Above, \eqref{eq:proof_lem_perrcls_chernoff1} follows by definition and \eqref{eq:proof_lem_perrcls_chernoff2} follows from noting $Z_n/2 \leq \lceil Z_n/2 \rceil$. We obtain \eqref{eq:proof_lem_perrcls_chernoff3} via the Chernoff bound for a $\binomialdist(Z_n,\pi)$ and we obtain \eqref{eq:proof_lem_perrcls_chernoff4} by expanding the Moment Generating Function for this binomial distribution. Finally, \eqref{eq:proof_lem_perrcls_chernoff5} follows from substituting $\Ierrcls(\theta,\phi)$.
\end{proof}

\begin{proof}[Proof of Corollary \ref{cor:ierr_properties}]
Note that $x \mapsto e^{-x}$ is a convex function in $x$. Then, for any $r \in [0, 1]$, we have
\begin{align*}
r e^0 + \bar{r} e^{-\Ierr(\pi)} \geq e^{-\bar{r} \Ierr(\pi)} ~~ \Longleftrightarrow ~~ -\log\left( r + \bar{r} e^{-\Ierr(\pi)}\right) \leq -\log e^{-\bar{r}\Ierr(\pi)} = \bar{r} \Ierr(\pi).
\end{align*}

\end{proof}

\begin{proof}[Proof of Theorem \ref{thm:asymptotic_nominal}]
We break the proof into an upper and a lower limit. We first derive the upper limit from the Chernoff bound
\begin{align*}
\Perr(n, \phi) &\leq \Pr\left( S_n \geq \frac{n}{2} \right) \leq \exp(-n\Ierr(\phi))
\end{align*}
Any odd $n \geq \log(1/\epsilon) / \Ierr(\phi)$ will satisfy $\Perr(n, \phi) \leq \epsilon$, and is consequently a feasible solution. Rounding upward to an odd integer adds less than two. Therefore,
\begin{align}
n^\dagger(\epsilon, \phi) \leq \frac{\log(1/\epsilon)}{\Ierr(\phi)} + 2. \label{eq:proof_thm_asymptotic_nominal0}
\end{align}
We re-arrange this inequality to yield
\begin{align*}
\limsup_{\epsilon \to 0} \frac{n^\dagger(\epsilon, \phi)}{\log(1/\epsilon)} \leq \frac{1}{\Ierr(\phi)}
\end{align*}

We now derive the lower limit by breaking into two cases. First, if $n = 2m$ is an even number, then
\begin{align}
\Perr(n, \phi) = \Pr(S_n \geq m) \geq \Pr(S_n = m) = {2m \choose m} \phi^m \bar{\phi}^m \label{eq:proof_thm_asymptotic_nominal1}
\end{align}
We can bound the factorial using Stirling's inequalities $\sqrt{2\pi n} (n/e)^n \leq n! \leq \sqrt{2 \pi n} (n/e)^n e^{1/(12n)}$ to yield
\begin{align*}
{2m \choose m} = \frac{(2m)!}{(m!)^2} \geq \frac{\sqrt{4 \pi m} (2m/e)^{2m}}{ 2 \pi m (m/e)^{2m} e^{1/(6m)} } = \frac{4^m e^{-\frac{1}{6m}}}{\sqrt{\pi m}}
\end{align*}
Using this bound, we obtain
\begin{align}
\mathrm{RHS \eqref{eq:proof_thm_asymptotic_nominal1}} &\geq \frac{4^m \phi^m \bar{\phi}^m e^{-\frac{1}{6m}}}{\sqrt{\pi m}} \label{eq:proof_thm_asymptotic_nominal2} \\
&= \frac{e^{-\frac{1}{6m}}}{ \sqrt{\pi m}}  e^{-n \Ierr(\phi)} \label{eq:proof_thm_asymptotic_nominal3} \\
&= \frac{\sqrt{2 / \pi}e^{-\frac{1}{3n}} }{ \sqrt n} e^{-n \Ierr(\phi)}  \label{eq:proof_thm_asymptotic_nominal4} \\
&\geq \frac{\sqrt{2 / \pi}e^{-\frac{1}{3}} }{ \sqrt n} e^{-n \Ierr(\phi)} \label{eq:proof_thm_asymptotic_nominal5}
\end{align}
Above, \eqref{eq:proof_thm_asymptotic_nominal2} follows from Stirling's lower bound on the factorial and \eqref{eq:proof_thm_asymptotic_nominal3} follows from substituting the definition of $\Ierr(\phi)$ to show that $e^{-n\Ierr(\phi)} = (4\phi \bar{\phi})^{n/2}$. Finally, \eqref{eq:proof_thm_asymptotic_nominal4} substitutes $m =n/2$ and \eqref{eq:proof_thm_asymptotic_nominal5} uses $e^{-1/(3n)} \geq e^{-1/3}$ for all $n\geq 1$.

Alternatively if $n=2m+1$ is an odd number, then
\begin{align}
\Perr(n,\phi) &\geq \Pr(S_n = m+1) \label{eq:proof_thm_asymptotic_nominal6} \\
&= {2m+1 \choose m+1} \phi^{m+1} \bar{\phi}^m  \label{eq:proof_thm_asymptotic_nominal7} \\
&= {2m \choose m} \frac{2m+1}{m+1} \phi^{m+1} \bar{\phi}^m  \label{eq:proof_thm_asymptotic_nominal8} \\
&\geq \frac{4^m e^{-\frac{1}{6m}}}{\sqrt{\pi m}} \left( \frac{2m+1}{m+1} \right) \phi^{m+1} \bar{\phi}^m  \label{eq:proof_thm_asymptotic_nominal9} \\
&= \frac{e^{-\frac{1}{6m}}}{\sqrt{\pi m}} \left( \frac{2m+1}{m+1} \right) \left( \frac{\phi}{\sqrt{4 \phi \bar{\phi}}} \right) e^{-n \Ierr(\phi)}  \label{eq:proof_thm_asymptotic_nominal10} \\
&\geq \frac{e^{-\frac{1}{6}}}{\sqrt{\pi m}} \left( \frac{\phi}{\sqrt{4 \phi \bar{\phi}}} \right) e^{-n \Ierr(\phi)}  \label{eq:proof_thm_asymptotic_nominal11} \\
&\geq \frac{\sqrt{2/\pi} e^{-\frac{1}{6}} }{ \sqrt{n}} \left( \frac{\phi}{\sqrt{4 \phi \bar{\phi}}} \right) e^{-n \Ierr(\phi)}  \label{eq:proof_thm_asymptotic_nominal12}
\end{align}
Above, \eqref{eq:proof_thm_asymptotic_nominal6} and \eqref{eq:proof_thm_asymptotic_nominal7} follow from definition of the probabilities. Then, \eqref{eq:proof_thm_asymptotic_nominal8} follows from re-writing the factorial and \eqref{eq:proof_thm_asymptotic_nominal9} from applying the Stirling's inequalities following the same steps as in \eqref{eq:proof_thm_asymptotic_nominal2}. We then obtain \eqref{eq:proof_thm_asymptotic_nominal10} by substituting the definition of $\Ierr(\phi)$ to show that $e^{-n\Ierr(\phi)} = (4\phi \bar{\phi})^m \sqrt{4\phi\bar{\phi}}$. Next, \eqref{eq:proof_thm_asymptotic_nominal11} is obtained by noting $e^{-1/(6m)} \geq e^{-1/6}$ for all $m \geq 1$ and noting $(2m+1)/(m+1)\geq 1$. Finally, we obtain \eqref{eq:proof_thm_asymptotic_nominal12} due to $1/\sqrt{\pi m} = \sqrt{2/\pi}/\sqrt{2m} \geq \sqrt{2/\pi}/\sqrt{n}$.

Then for $n^\dagger(\epsilon,\phi)$, we can combine \eqref{eq:proof_thm_asymptotic_nominal5} and \eqref{eq:proof_thm_asymptotic_nominal12} to yield
\begin{align}
\epsilon\geq\Perr(n^\dagger(\epsilon, \phi),\phi) &\geq \underbrace{\min\left( \sqrt{\frac{2}{\pi}} e^{-\frac{1}{3}} \;,\; \sqrt{\frac{2}{\pi}} \frac{e^{-\frac{1}{6}}\phi}{\sqrt{4 \phi \bar{\phi}}}\right)}_{=:c(\phi)} \frac{e^{-n^\dagger(\epsilon, \phi) \Ierr(\phi)}}{\sqrt{n^\dagger(\epsilon, \phi)}}.
\end{align}
Let $c(\phi)$ be defined as the minimum of the two terms. Taking the logarithm on both sides and re-arranging the inequality yields
\begin{align*}
n^\dagger(\epsilon, \phi) \Ierr(\phi) \geq \log\frac{1}{\epsilon} + \log c(\phi) + \log \frac{1}{\sqrt{n^\dagger(\epsilon, \phi)}}
\end{align*}
Finally, we divide both sides by $\log(1/\epsilon)$ and $\Ierr(\phi)$ to obtain
\begin{align*}
\frac{n^\dagger(\epsilon, \phi)}{\log(1/\epsilon)} \geq \frac{1}{\Ierr(\phi)} \left( 1 + \frac{\log c(\phi)}{\log(1/\epsilon)} + \frac{\log(1/\sqrt{n^\dagger(\epsilon, \phi)})}{\log(1/\epsilon)} \right)
\end{align*}
However, recall from \eqref{eq:proof_thm_asymptotic_nominal0} that $n^\dagger(\epsilon,\phi) = O(\log(1/\epsilon))$. Consequently, the correction term $\log (c(\phi)/\sqrt{n^\dagger(\epsilon,\phi)})/\log(1/\epsilon)$ approaches $0$ as $\epsilon$ goes to $0$. Therefore,
\begin{align*}
\liminf_{\epsilon \to 0} \frac{n^\dagger(\epsilon, \phi)}{\log(1/\epsilon)} \geq \frac{1}{\Ierr(\phi)}.
\end{align*}
Finally, because both the limit inferior and limit superior are equal, we prove the limit.
\end{proof}

\begin{proof}[Proof of Theorem \ref{thm:asymptotic_clf}]
Consider fixed $\theta$. From Lemma \ref{lem:perrcls_chernoff}, note that any odd $n \geq \log(1/\epsilon)/\Ierrcls(\theta,\phi)$ satisfies $\Perrcls(n,\theta,\phi) \leq \epsilon$. Consequently, $n^*(\epsilon, \theta, \phi) \leq \log(1/\epsilon)/\Ierrcls(\theta,\phi) + 2$ and
\begin{align}
&n^*(\epsilon, \theta, \phi)(1-r(\theta,\phi)) \leq (1-r(\theta, \phi)) \left( \frac{\log(1/\epsilon)}{\Ierrcls(\theta, \phi)} + 2 \right) \label{eq:proof_thm_asymptotic_clf1} \\
\Rightarrow~~& \frac{n^*(\epsilon, \theta, \phi)(1-r(\theta,\phi))}{\log(1/\epsilon)} \leq (1-r(\theta, \phi))  \left( \frac{1}{\Ierrcls(\theta,\phi)} + \frac{2}{\log(1/\epsilon)} \right) \label{eq:proof_thm_asymptotic_clf2} \\
\Rightarrow~~& \limsup_{\epsilon \to 0} \frac{n^*(\epsilon, \theta, \phi)(1-r(\theta,\phi))}{\log(1/\epsilon)} \leq  \frac{1-r(\theta, \phi)}{\Ierrcls(\theta,\phi)} \label{eq:proof_thm_asymptotic_clf3}
\end{align}
Above, \eqref{eq:proof_thm_asymptotic_clf2} divides both sides by $\log(1/\epsilon)$ and \eqref{eq:proof_thm_asymptotic_clf3} takes the limit supremum. Finally, we extend this inequality below
\begin{align*}
\inf_\theta \frac{1-r(\theta, \phi)}{\Ierrcls(\theta,\phi)} &\geq \inf_\theta \limsup_{\epsilon \to 0} \frac{n^*(\epsilon, \theta, \phi)(1-r(\theta,\phi))}{\log(1/\epsilon)} \\
&\geq \limsup_{\epsilon\to 0} \min_\theta\frac{n^*(\epsilon, \theta, \phi)(1-r(\theta,\phi))}{\log(1/\epsilon)}
\end{align*}
\end{proof}

\begin{proof}[Proof of Theorem \ref{thm:clf_asymptotically_better}]
We apply the following sequence:
\begin{align*}
\limsup_{\epsilon\to 0} \min_{\theta} \frac{n^*(\epsilon, \theta, \phi)(1-r(\theta,\phi))}{\log(1/\epsilon)} \leq \inf_\theta \frac{1-r(\theta, \phi)}{\Ierrcls(\theta,\phi)} < \frac{1}{\Ierr(\phi)} = \lim_{\epsilon\to 0} \frac{n^\dagger(\epsilon, \phi)}{\log(1/\epsilon)}
\end{align*}
Above, the first inequality follows from Theorem \ref{thm:asymptotic_clf}. The second inequality follows from the infimum over $\theta$, which includes some $\theta$ that satisfies $(1-r(\theta,\phi)) \Ierr(\phi)< \Ierrcls(\theta, \phi)$ as required in the condition to this theorem. The equality follows from Theorem \ref{thm:asymptotic_nominal}. Taking the limits completes the proof.
\end{proof}

\section{Additional empirical analysis}
\label{sec:ec_descriptive}

\subsection{Annotation workload}

Figure \ref{fig:tasks_sorted_by_worker_ecpd} plots the number of annotators exceeding each task-completion threshold over the study period.
Figure \ref{fig:tasks_sorted_by_worker_zod} plots the number of annotators exceeding each task-completion threshold.

\begin{figure}[t!]
\centering
\begin{subfigure}[b]{0.49\textwidth}
\centering
\includegraphics[width=\textwidth]{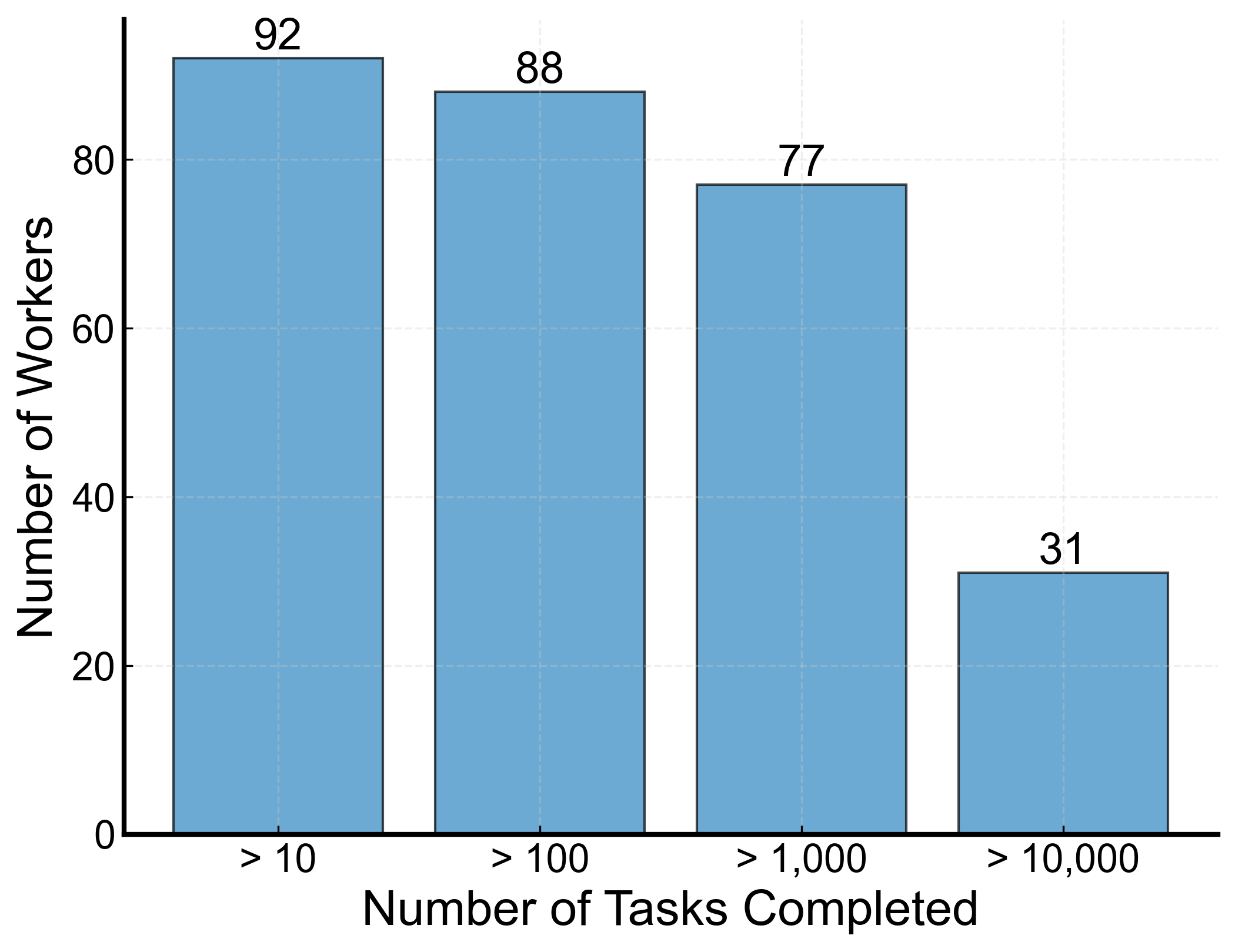}
\caption{ECPD}
\label{fig:tasks_sorted_by_worker_ecpd}
\end{subfigure}
\hfill
\begin{subfigure}[b]{0.49\textwidth}
\centering
\includegraphics[width=\textwidth]{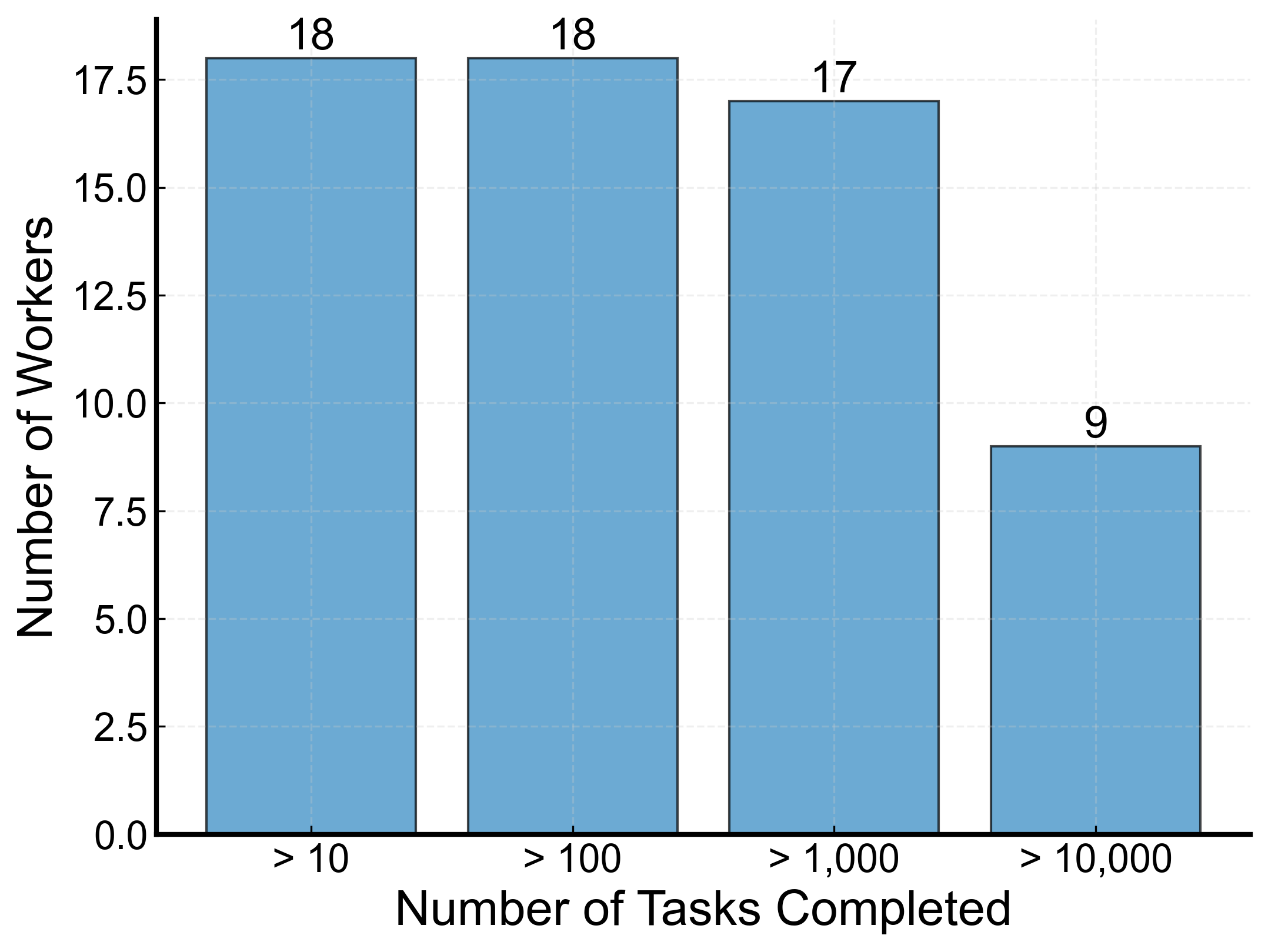}
\caption{ZOD}
\label{fig:tasks_sorted_by_worker_zod}
\end{subfigure}
\caption{
The number of annotators who completed more than 10, 100, 1,000, or 10,000 tasks for the two datasets.
}
\label{fig:tasks_sorted_by_worker}
\end{figure}

\subsection{Minority report rates}
First, some objects that are in the distance, partially occluded, or blurry, can be visually hard to distinguish with respect to questions (see Figure \ref{fig:ecpd_example} for an example).
Figure \ref{fig:preliminary_observations_workers_crops} visualizes the histogram of the crop minority report rate, i.e., proportion of minority reports over all repeats in each crop. Both datasets show highly skewed crop-level minority report distributions, with most crops generating few minority reports and a small subset generating many.
Second, annotators have variable skill levels for data annotation \citep{raykar2010learning, karger2014budget}. Figure \ref{fig:preliminary_observations_workers_crops} visualizes the histogram of the worker minority report rate, i.e., proportion of minority reports produced by each worker. As with crops, worker-level minority report rates are highly skewed.
Finally, workers are subject to behavioral effects where their abilities to correctly answer questions vary as a function of how long they have been working without any rest. Figure~\ref{fig:preliminary_observations_exhaustion} reports minority-report rates aggregated by calendar hour across workers. Some ECPD days exhibit U-shaped profiles resembling bathtub curves in reliability analysis \citep{o2011practical}.

\begin{figure}[t!]
\centering
\begin{subfigure}[b]{0.49\textwidth}
\centering
\includegraphics[width=0.49\textwidth]{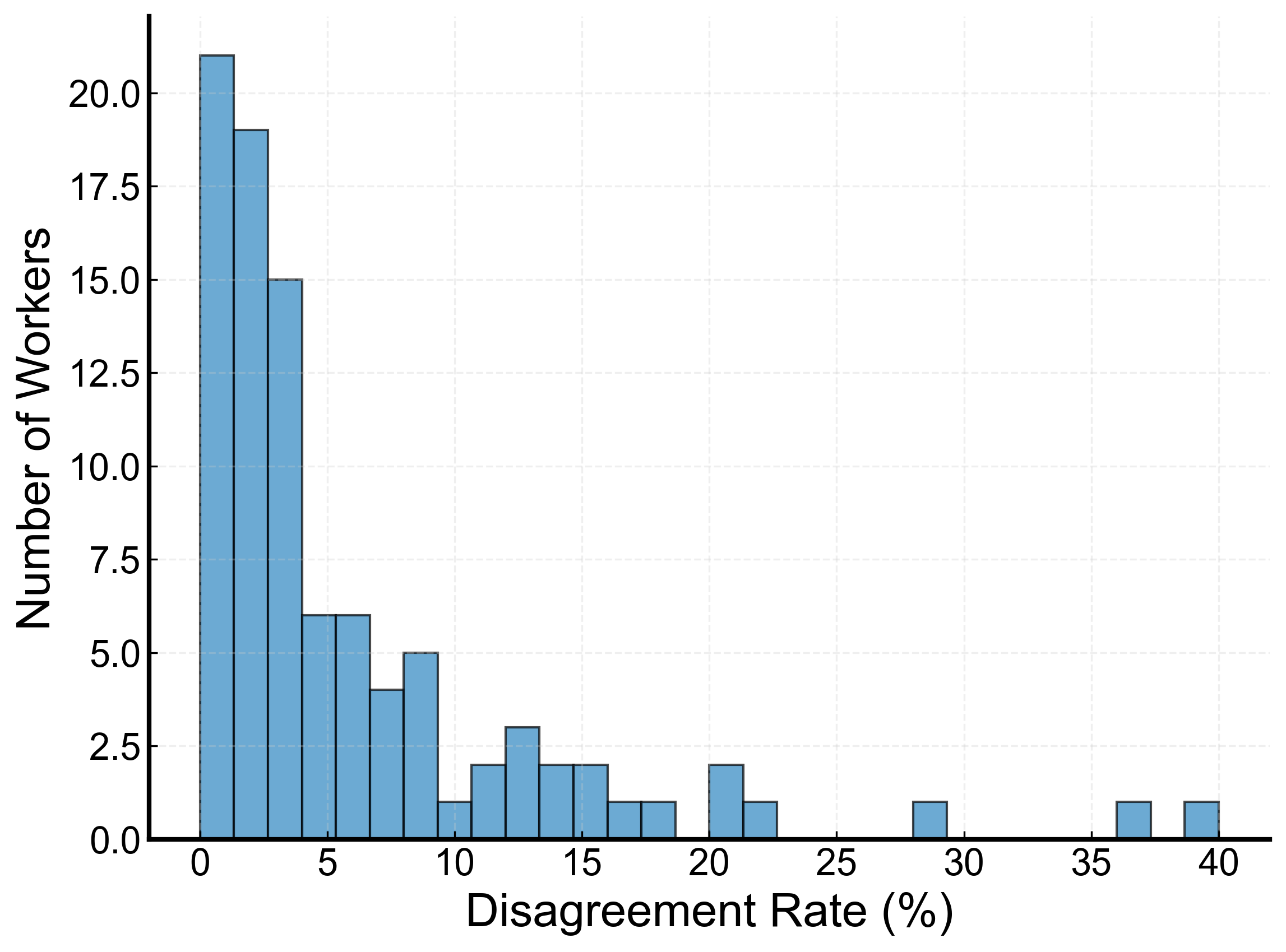}
\includegraphics[width=0.49\textwidth]{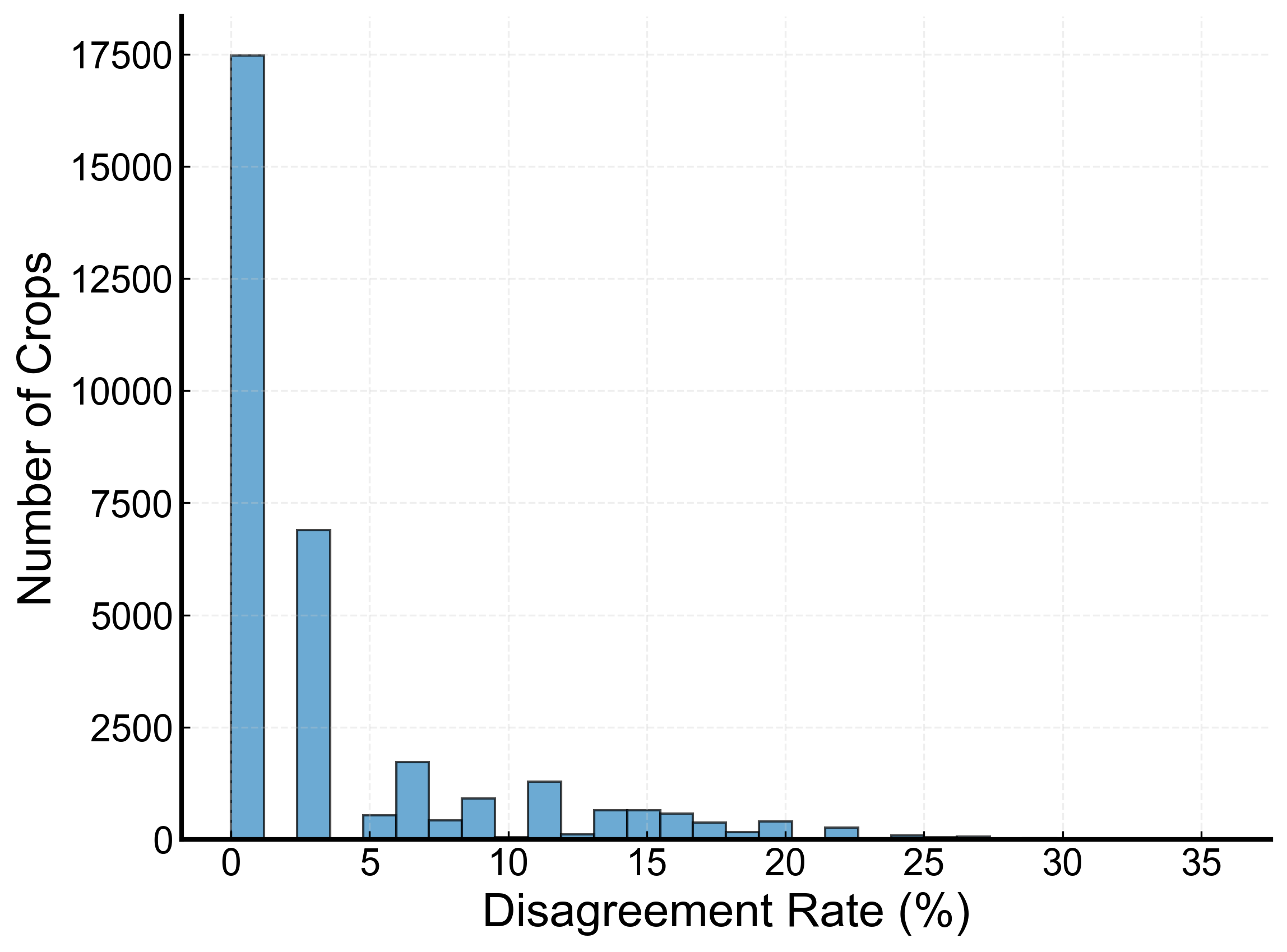}
\caption{ECPD}
\label{fig:preliminary_observations_workers_crops_ecpd}
\end{subfigure}
\hfill
\begin{subfigure}[b]{0.49\textwidth}
\centering
\includegraphics[width=0.49\textwidth]{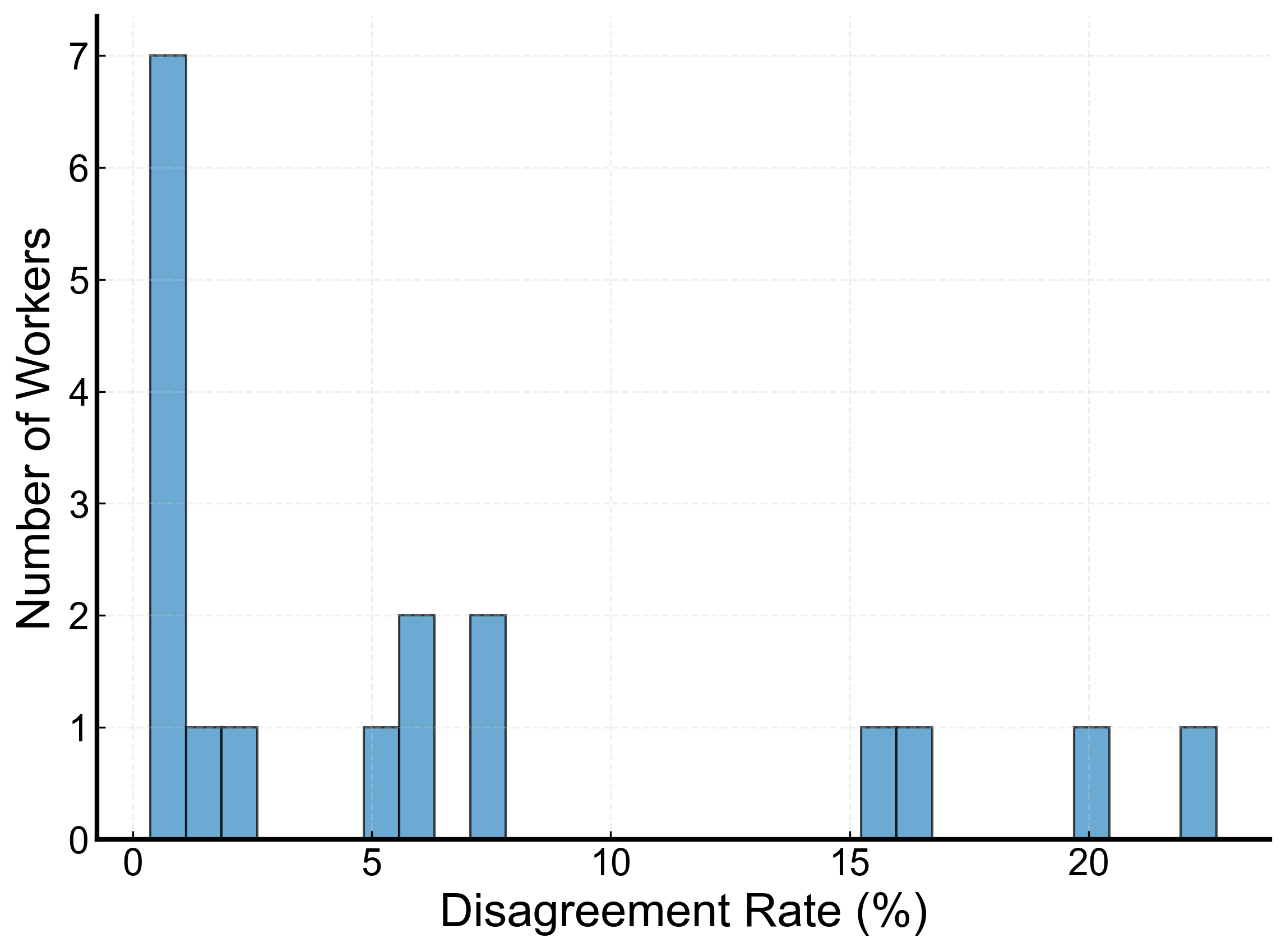}
\includegraphics[width=0.49\textwidth]{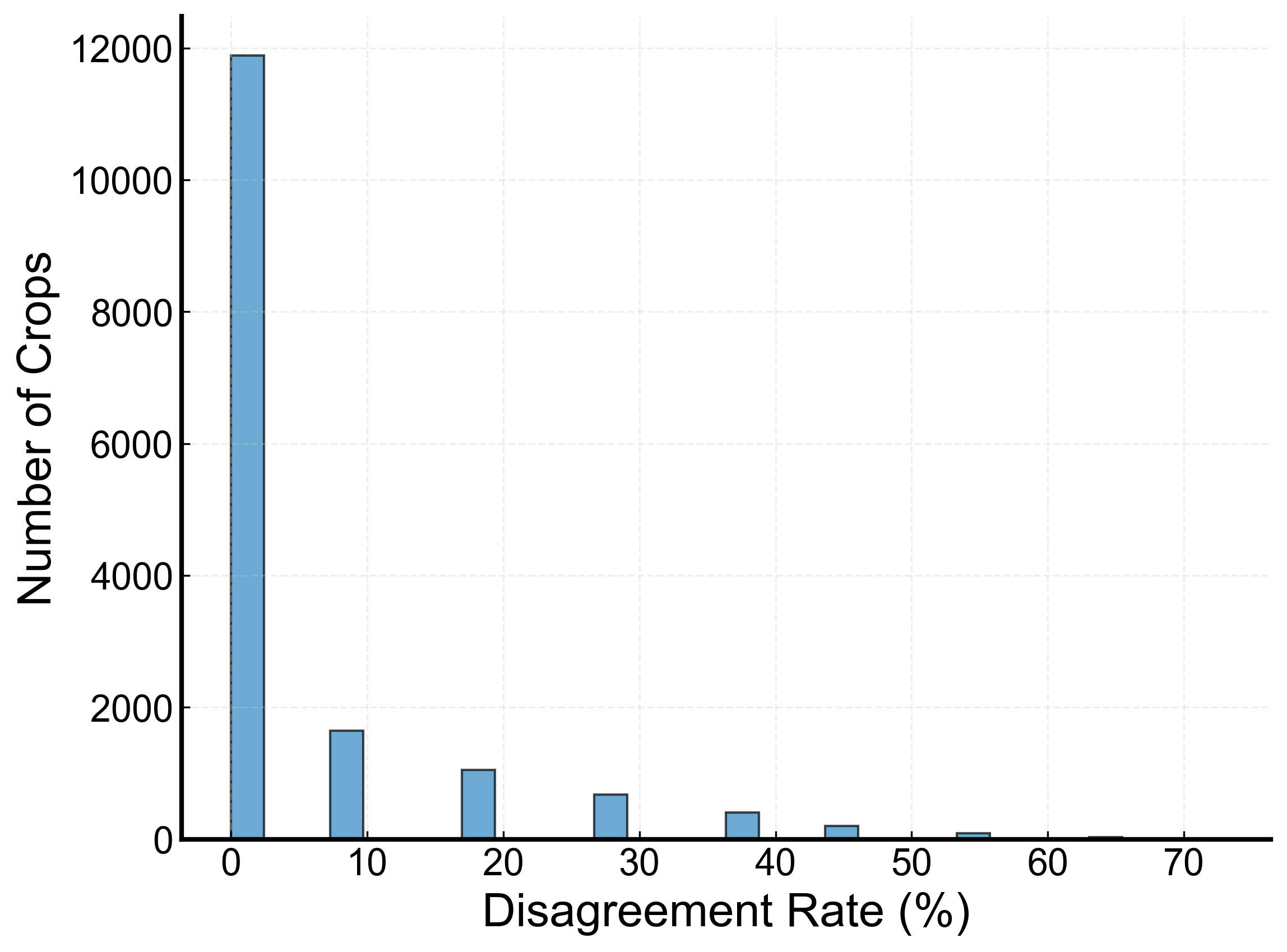}
\caption{ZOD}
\label{fig:preliminary_observations_workers_crops_zod}
\end{subfigure}
\caption{
For both annotation datasets:
\emph{(Left)} the histogram of minority report rates per worker;
\emph{(Right)} the histogram of minority report rates per crop.
}
\label{fig:preliminary_observations_workers_crops}
\end{figure}

\begin{figure}[t!]
\centering
\begin{subfigure}[b]{0.74\textwidth}
\centering
\includegraphics[width=\textwidth]{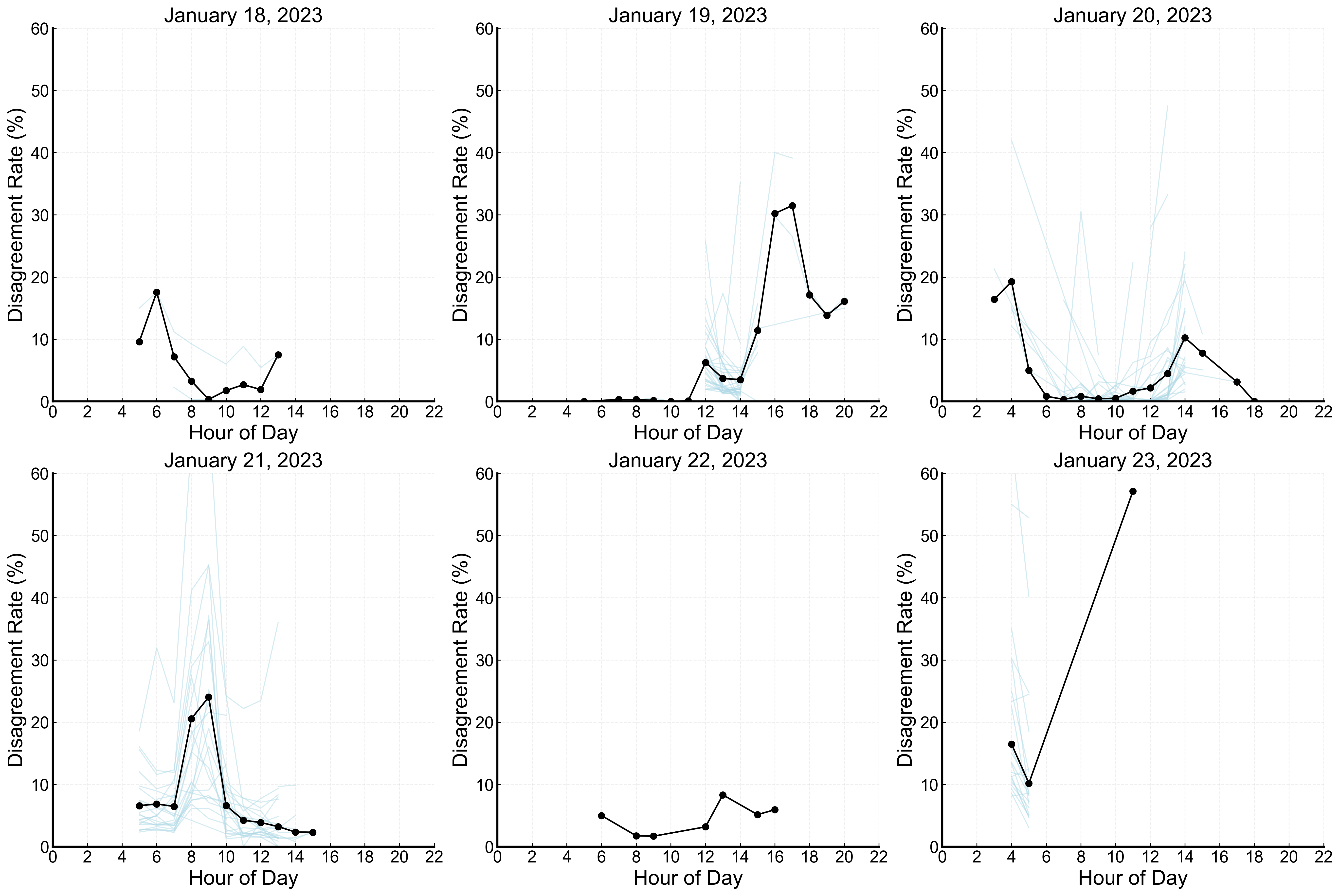}
\caption{ECPD}
\label{fig:preliminary_observations_exhaustion_ecpd}
\end{subfigure}
\hfill
\begin{subfigure}[b]{0.24\textwidth}
\centering
\includegraphics[width=\textwidth]{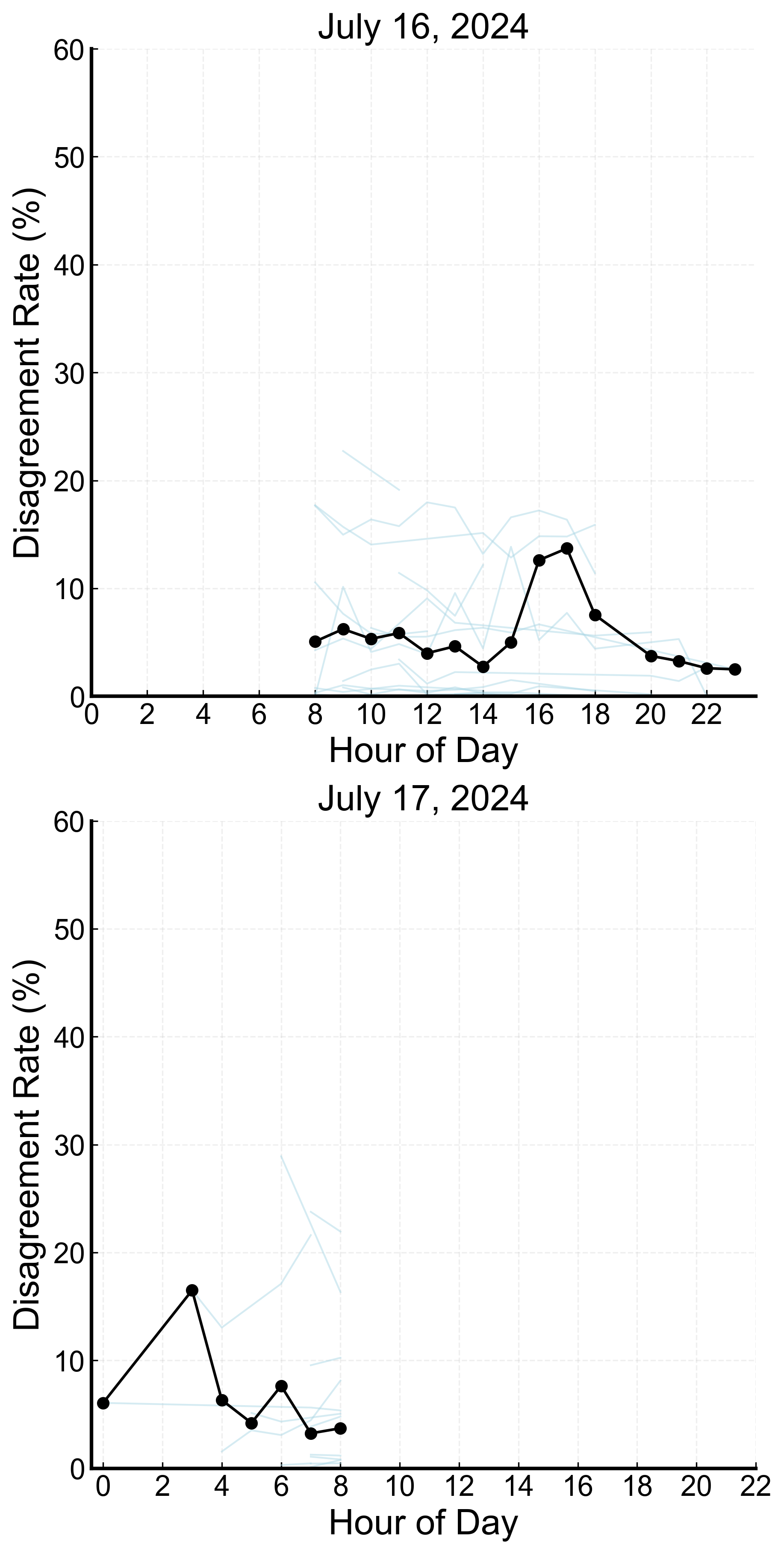}
\caption{ZOD}
\label{fig:preliminary_observations_exhaustion_zod}
\end{subfigure}
\caption{
The hourly minority report rate on each day of the study. The solid black lines reflect the average minority report rate over all active workers during that hour and the light blue lines reflect the minority report rates for individual workers.
}
\label{fig:preliminary_observations_exhaustion}
\end{figure}

\subsection{Residual diagnostics}
\label{sec:app_residuals}

\begin{figure}[t!]
\centering
\begin{subfigure}[b]{0.49\textwidth}
\centering
\includegraphics[width=0.99\textwidth]{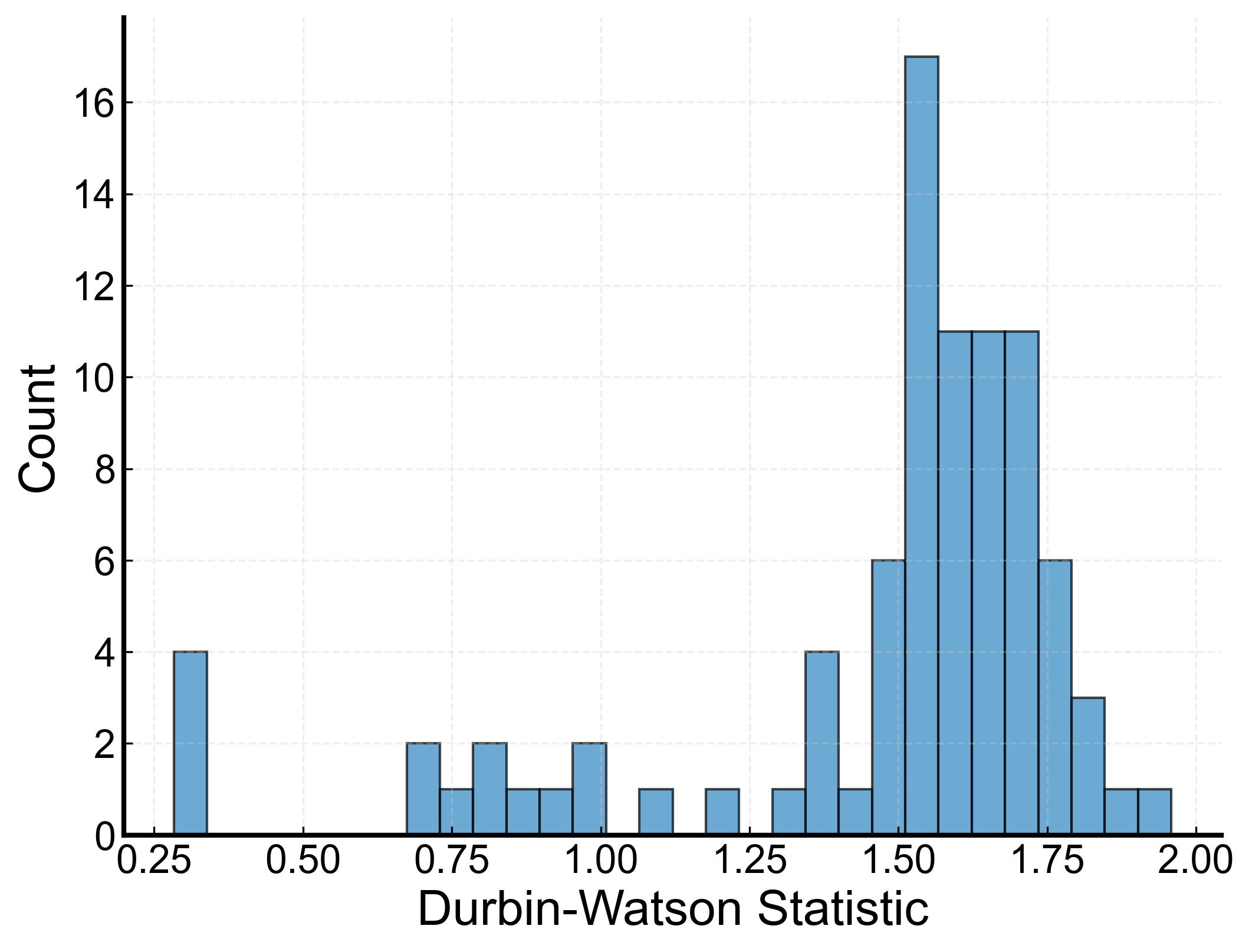}
\caption{ECPD}
\label{fig:residuals_ecpd}
\end{subfigure}
\hfill
\begin{subfigure}[b]{0.49\textwidth}
\centering
\includegraphics[width=0.99\textwidth]{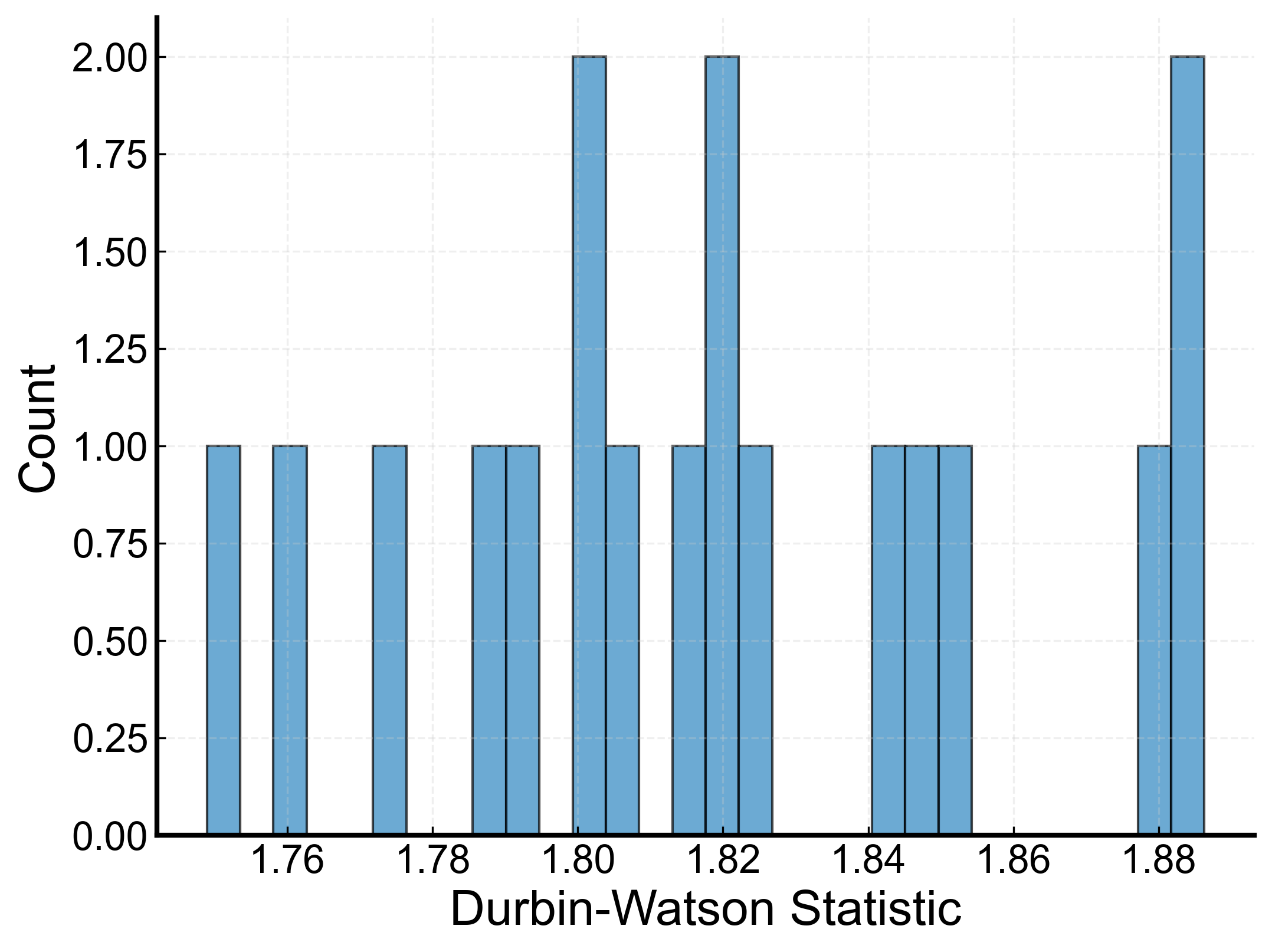}
\caption{ZOD}
\label{fig:residuals_zod}
\end{subfigure}
\caption{
For both annotation datasets,
the histogram of Durbin-Watson statistics for each annotator over the corresponding study period.
}
\label{fig:residuals}
\end{figure}

We examine the residual diagnostics associated with the final logistic regression model, i.e., Model A + W + C, developed in Section \ref{sec:empirical_model}.
To investigate residual variance, we conduct dispersion tests for both datasets \citep{hartig2016dharma}. Results indicate no significant evidence of either over-dispersion or under-dispersion in the residuals ($p > 0.05$).

To examine temporal dependence, we compute Durbin-Watson statistics from each annotator's ordered deviance residuals. Figure \ref{fig:residuals} shows that most values exceed 1.5.
This and the lower tail indicate some potential minor positive autocorrelation for annotators.

\subsection{Sensitivity to the definition of continuous activity}
\label{sec:app_robustness_to_ct}

\begin{table}[t!]
\centering
\small
\begin{tabular}{lcccc}
\toprule

Parameter & Default & 1 minute & 30 minutes & 60 minutes \\

\midrule

Continuous activity$^2$ & $1.428^{***}$ & $1.345^{***}$ & $1.071^{***}$ & $1.016^{***}$ \\
& $($0.007$)$ & $($0.024$)$ & $($0.002$)$ & $($0.001$)$ \\

Continuous activity & $0.399^{***}$ & $0.519^{***}$ & $0.798^{***}$ & $0.930^{***}$ \\
& $($0.020$)$ & $($0.037$)$ & $($0.011$)$ & $($0.008$)$ \\

\midrule

Observations                            & 1,035,491       & 1,035,491       & 1,035,491             & 1,035,491 \\
Pseudo-$R^2$                 & 0.0968          & 0.0918           & 0.0947                & 0.0915 \\

\bottomrule
\end{tabular}
\caption{Comparison of odds ratios for Model A with different definitions of `Continuous activity' on ECPD. The default definition is 10 minutes. Standard errors for the log-odds are shown in parentheses. Significance levels: *** $p<0.001$, ** $p<0.01$, * $p<0.05$.}
\label{tab:model_ct_robust_ecpd}
\end{table}

In Section \ref{sec:empirical}, Model A consisted of a logistic regression model to predict the odds of a minority report given the amount of continuous time that the annotator had spent working. Here, continuous activity was defined as the uninterrupted period of task completion, with breaks of at least 10 minutes marking the start of a new activity period. In this section, we explore the robustness of Model A to alternative definitions of continuous activity.

To assess robustness, we re-estimate Model A using alternative break thresholds of one minute, 30 minutes, and 60 minutes. Table~\ref{tab:model_ct_robust_ecpd} summarizes these results, as well as the default 10 minute break version (Model A from Table~\ref{tab:model_abcd_ecpd}) for reference. Across all model specifications, the coefficients for Continuous Activity and Continuous Activity$^2$ remain statistically significant ($p < 0.001$), indicating robustness to the choice of threshold. The negative linear and positive quadratic coefficients preserve the U-shaped activity association across these thresholds. Finally, we remark that the Pseudo-$R^2$ values remain consistent.

\end{document}